\documentclass{article}

\RequirePackage{etex}

\usepackage[final,nonatbib]{neurips_2020}

\usepackage[utf8]{inputenc} 
\usepackage[T1]{fontenc}    
\usepackage{hyperref}       
\usepackage{url}            
\usepackage{booktabs}       
\usepackage{amsfonts}       
\usepackage{nicefrac}       
\usepackage{microtype}      
\usepackage{xspace}
\usepackage{graphicx}
\usepackage{caption}
\usepackage{subcaption}
\usepackage{boldline}
\usepackage{multirow}

\usepackage{dirtytalk}
\usepackage[resetlabels]{multibib}
\newcites{Ref}{References}

\usepackage{amsmath}
\usepackage{amsfonts}
\usepackage{amssymb}
\usepackage{amsthm}
\usepackage{bm}
\usepackage{bbm}
\usepackage{mathtools}
\usepackage{enumitem}
\usepackage{thmtools,thm-restate}
\usepackage{algorithm}
\usepackage{algorithmic}
\usepackage{color}
\usepackage[dvipsnames]{xcolor}
\usepackage{graphicx}
\usepackage{comment}

\def \etl {\em et al.\rm}

\theoremstyle{plain}
\newtheorem{lemma}{Lemma}
\newtheorem{theorem}{Theorem}
\newtheorem{proposition}{Proposition}
\newtheorem{corollary}{Corollary}

\theoremstyle{definition}
\newtheorem{definition}{Definition}

\theoremstyle{remark}

\def\DD{\mathcal{D}}

\def\YY{\mathcal{Y}}\def\ZZ{\mathcal{Z}}

\def\Sb{\mathbf{S}}

\def\bb{\mathbf{b}}
\def\db{\mathbf{d}}\def\eb{\mathbf{e}}\def\fb{\mathbf{f}}
\def\gb{\mathbf{g}}

\def\mb{\mathbf{m}}
\def\rb{\mathbf{r}}
\def\tb{\mathbf{t}}\def\ub{\mathbf{u}}
\def\vb{\mathbf{v}}\def\wb{\mathbf{w}}\def\xb{\mathbf{x}}
\def\yb{\mathbf{y}}\def\zb{\mathbf{z}}

\def\Abb{\mathbb{A}}\def\Bbb{\mathbb{B}}\def\Cbb{\mathbb{C}}
\def\Dbb{\mathbb{D}}

\def\Nbb{\mathbb{N}}
\def\Pbb{\mathbb{P}}\def\Rbb{\mathbb{R}}

\def\Xbb{\mathbb{X}}

\def\R{\Rbb}

\def\t{\top}
\def\*{\star}

\newcommand{\norm}[1]{ \| #1 \|  }

\DeclareMathOperator*{\argmax}{arg\,max}

\usepackage{etex,etoolbox}
\usepackage{environ}

\makeatletter
\providecommand{\@fourthoffour}[4]{#4}
\newcommand\fixstatement[2][\proofname\space of]{%
	\ifcsname thmt@original@#2\endcsname
	\AtEndEnvironment{#2}{%
		\xdef\pat@label{\expandafter\expandafter\expandafter
			\@fourthoffour\csname thmt@original@#2\endcsname\space\@currentlabel}%
		\xdef\pat@proofof{\@nameuse{pat@proofof@#2}}%
	}%
	\else
	\AtEndEnvironment{#2}{%
		\xdef\pat@label{\expandafter\expandafter\expandafter
			\@fourthoffour\csname #1\endcsname\space\@currentlabel}%
		\xdef\pat@proofof{\@nameuse{pat@proofof@#2}}%
	}%
	\fi
	\@namedef{pat@proofof@#2}{#1}%
}

\globtoksblk\prooftoks{1000}
\newcounter{proofcount}

\NewEnviron{proofatend}{%
	\edef\next{%
		\noexpand\begin{proof}[\pat@proofof\space\pat@label]%
			\unexpanded\expandafter{\BODY}}%
		\global\toks\numexpr\prooftoks+\value{proofcount}\relax=\expandafter{\next\end{proof}}
	\stepcounter{proofcount}}
\def\printproofs{%
	\count@=\z@
	\loop
	\the\toks\numexpr\prooftoks+\count@\relax
	\ifnum\count@<\value{proofcount}%
	\advance\count@\@ne
	\repeat}
\makeatother

\fixstatement{lemma}
\fixstatement{theorem}
\fixstatement{proposition}
\fixstatement{corollary}
\graphicspath{ {figures/} }
\usepackage[capitalise]{cleveref}

\newcommand{\since}[1]{\qquad \text{($\because$ #1)}}

\def\smb{\mathbf{sm}}

\title{Intra Order-Preserving Functions for\\ Calibration of Multi-Class Neural Networks}

\author{
  Amir Rahimi\thanks{Equal Contribution.} \\
  ANU, ACRV \\
  \href{mailto:amir.rahimi@anu.edu.au}{amir.rahimi@anu.edu.au}
  \And
  Amirreza Shaban$^*$ \\
  Georgia Tech \\
  \href{mailto:ashaban@uw.edu}{ashaban@uw.edu}
  \And
  Ching-An Cheng$^*$ \\
  Microsoft Research \\
  \href{mailto:chinganc@microsoft.com}{chinganc@microsoft.com}
  \AND
  Richard Hartley \\
  Google Research, ANU, ACRV \\
  \href{mailto:richard.hartley@anu.edu.au}{richard.hartley@anu.edu.au}
  \And
  Byron Boots \\
  University of Washington \\
  \href{mailto:bboots@cs.washington.edu}{bboots@cs.washington.edu} \\ \\
}
\begin{document}
\maketitle
\begin{abstract}
Predicting calibrated confidence scores for multi-class deep networks is important for avoiding rare but costly mistakes. A common approach is to learn a post-hoc calibration function that transforms the output of the original network into calibrated confidence scores while maintaining the network's accuracy. However, previous post-hoc calibration techniques work only with simple calibration functions, potentially lacking sufficient representation to calibrate the complex function landscape of deep networks. 
In this work, we aim to learn general post-hoc calibration functions that can preserve the top-$k$ predictions of any deep network. We call this family of functions \emph{intra order-preserving} functions. We propose a new neural network architecture that represents a class of intra order-preserving functions by combining common neural network components. Additionally, we introduce \emph{order-invariant} and \emph{diagonal} sub-families, which can act as regularization for better generalization when the training data size is small.
We show the effectiveness of the proposed method across a wide range of datasets and classifiers. Our method outperforms state-of-the-art post-hoc calibration methods, namely temperature scaling and Dirichlet calibration, in several evaluation metrics for the task.
\end{abstract}
\vspace{-0.5cm}
\section{Introduction}
 
 
Deep neural networks have demonstrated impressive accuracy in classification tasks, such as image recognition~\cite{he2016deep, ren2015faster} and medical research~\cite{jiang2012calibrating, caruana2015intelligible}.
These exciting results have recently motivated engineers to adopt deep networks as default components in building decision systems; for example, a multi-class neural network can be treated as a probabilistic predictor and its softmax output can provide the confidence scores of different actions for the downstream decision making pipeline~\cite{girshick2015fast,cao2017realtime,mozafari2019unsupervised}.
While this is an intuitive idea, recent research has found that deep networks, despite being accurate, can be overconfident in their predictions, exhibiting high calibration error
~\cite{maddox2019simple, guo2017calibration, kendall2017uncertainties}.
In other words, trusting the network's output naively as confidence scores in system design could cause undesired consequences: a serious issue for applications where mistakes are costly, such as medical diagnosis and autonomous driving. 


A promising approach to address the miscalibration is to augment a given network with a parameterized calibration function, such as extra learnable layers. 
This additional component is tuned post-hoc using a held-out calibration dataset, so that the effective full network becomes calibrated~\cite{guo2017calibration,kull2019beyond,kull2017beyond,kull2017beta,platt1999probabilistic,zadrozny2001obtaining}. 
In contrast to usual deep learning, 
the calibration dataset here is typically \emph{small}. Therefore, learning an overly general calibration function can easily overfit and actually reduce the accuracy of the given network~\cite{guo2017calibration,kull2019beyond}.
Careful design regularization and parameterization of calibration functions is imperative. 

%
%
A classical non-parametric technique is isotonic regression~\cite{zadrozny2002transforming}, which learns a monotonic staircase calibration function with minimal change in the accuracy. But the complexity of non-parametric learning can be too expensive to provide the needed generalization~\cite{kull2017beyond,kull2017beta}.
By contrast, Guo \etl~\cite{guo2017calibration} proposed to learn a scalar parameter to rescale the original output logits, at the cost of being suboptimal in calibration
~\cite{maddox2019simple}; see also \cref{sec:experiments}.
Recently, Kull \etl~\cite{kull2019beyond} proposed 
to learn linear transformations of the output logits. While this scheme is more expressive than the temperature scaling above, it does not explore non-linear calibration functions.



In general, a preferable hypothesis space needs to be expressive and, at the same time, provably preserve the accuracy of any given network it calibrates.
Limiting the expressivity of calibration functions can be an issue, especially when calibrating deep networks with complicated landscapes.

The main contribution of this paper is introducing a learnable space of functions, called \emph {intra order-preserving} family. 
Informally speaking, an intra order-preserving function $\fb:\R^n \to \R^n$ is a vector-valued function whose output values always share the same ordering as the input values across the $n$ dimensions. For example, if $\xb\in\R^n$ is increasing from coordinate $1$ to $n$, then so is $\fb(\xb)$. In addition, we introduce {order-invariant} and {diagonal} structures, which utilize the shared characteristics between different input dimensions to improve generalization. For illustration, we depict instances of 3-dimensional intra order-preserving and order-invariant functions defined on the unit simplex and compare them to an unconstrained function in \cref{fig:simplex_summary}. We use arrows to show how inputs on the simplex are mapped by each function. Each colored subset in the simplex denotes a region with the same input order; for example, we have $\xb_3>\xb_2>\xb_1$ inside the red region where the subscript~$_i$ denotes the $i$th element of a vector. For the intra order-preserving function shown in \cref{fig:op_map} arrows stay within the same colored region as the inputs, but the vector fields in two different colored region are independent to each other. Order-invariant function in \cref{fig:oi_map} further keeps the function permutation invariant, enforcing the vector fields to be the same among all the $6$ colored regions (as reflected in the symmetry in \cref{fig:oi_map}). This property of order-preserving functions significantly reduce the hypothesis space in learning, from the functions on whole simplex to functions on one colored region, for better generalization.

\begin{figure}[t]
\centering
\begin{subfigure}{0.25\textwidth}
 \includegraphics[width=\linewidth]{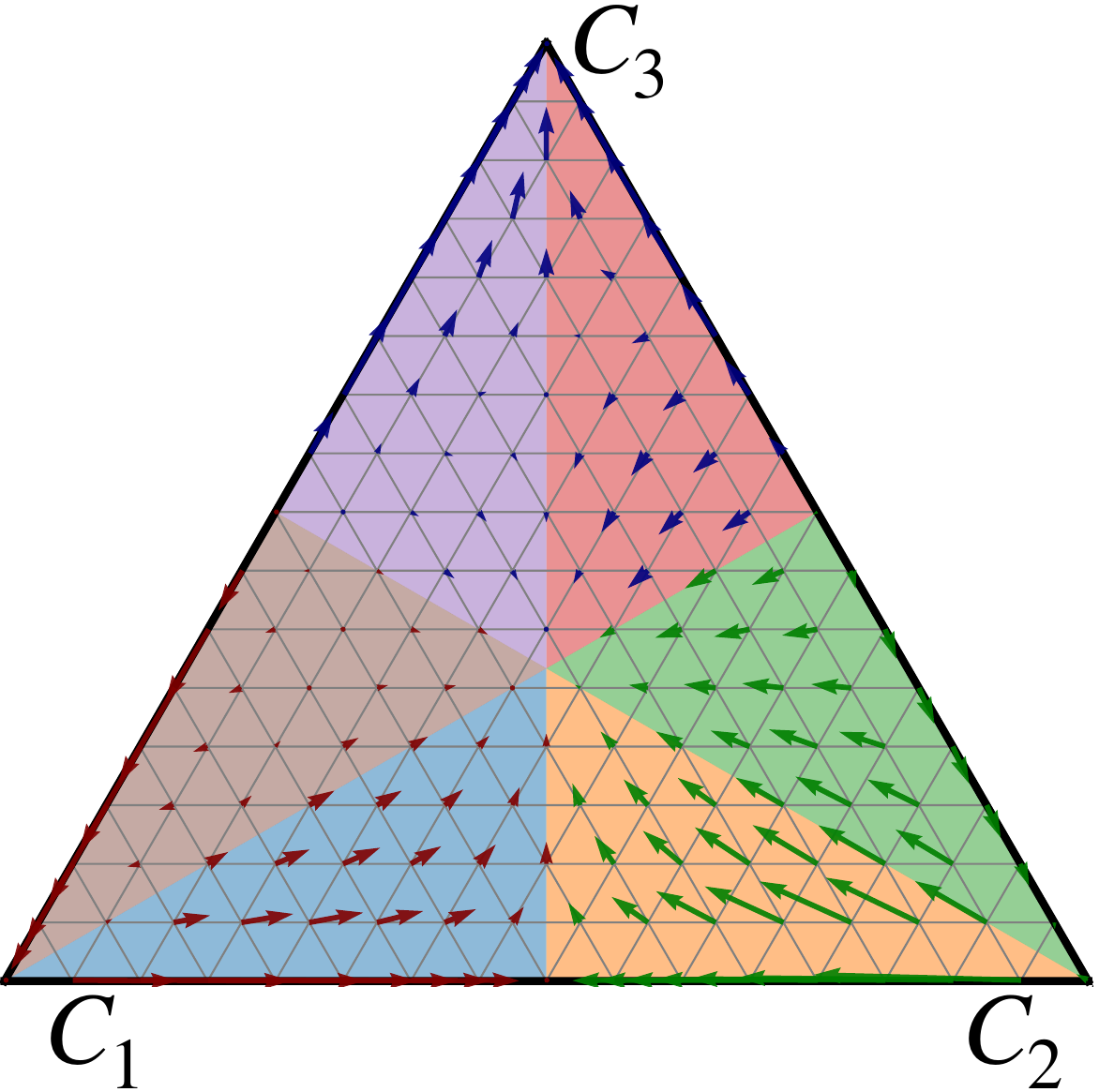}
 \caption{Intra Order-preserving} \label{fig:op_map}
\end{subfigure}~~~~%
\begin{subfigure}{0.25\textwidth}
 \includegraphics[width=\linewidth]{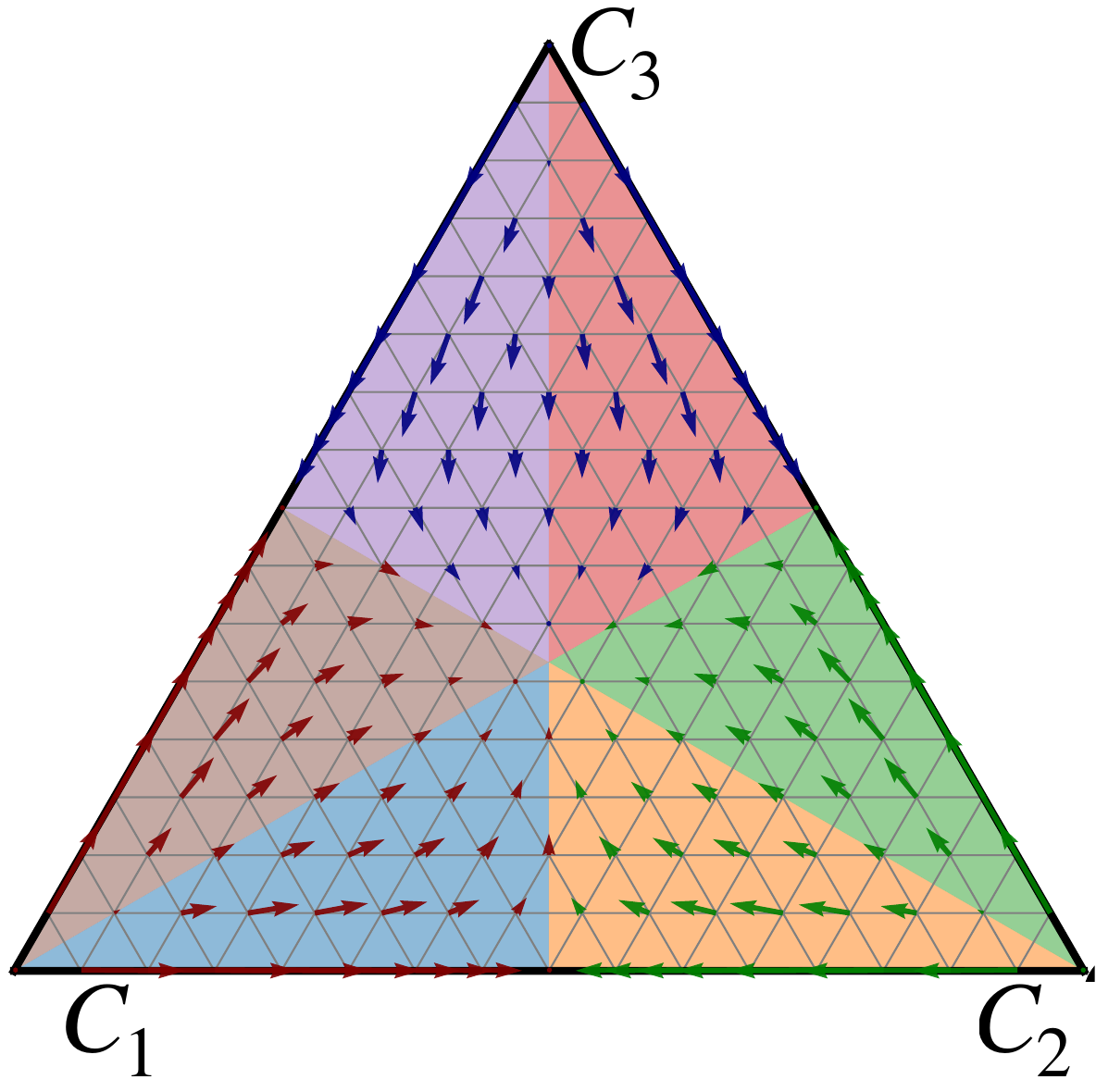}
 \caption{Order-invariant} \label{fig:oi_map}
\end{subfigure}~~~~%
\begin{subfigure}{0.25\textwidth}
 \includegraphics[width=\linewidth]{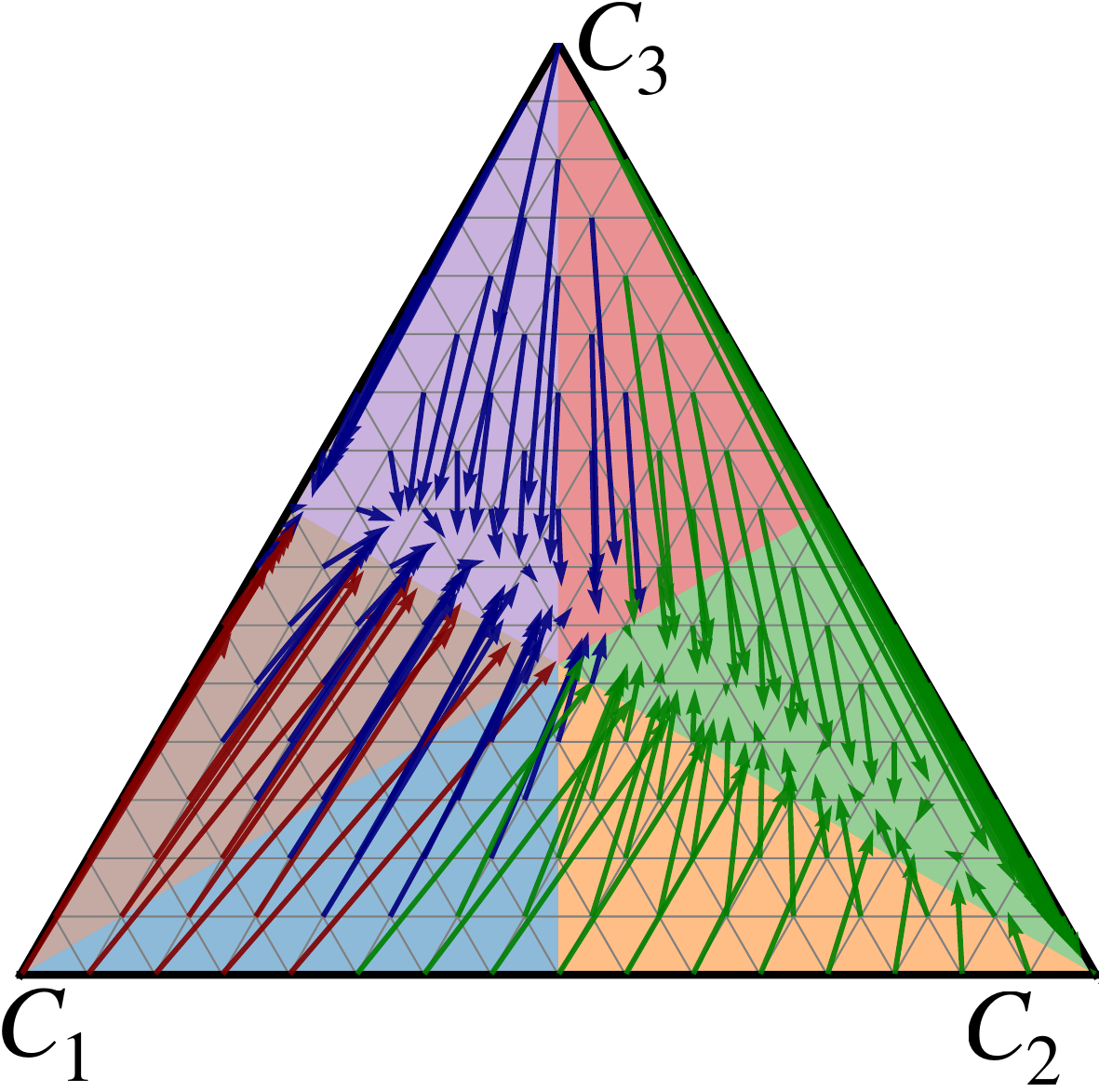}
 \caption{Unconstrained} \label{fig:un_map}
\end{subfigure}
\caption{\label{fig:simplex_summary} Comparing instances of intra order-preserving and order-invariant family defined on the 2-dimensional unit simplex. Points $C_1 = [1, 0, 0]^\top$, $C_2 = [0, 1, 0]^\top$, $C_3 = [0, 0, 1]^\top$ are the simplex corners. Arrows depict how an input is mapped by each function. Unconstrained function freely maps the input probabilities, intra order-preserving function enforces the outputs to stay within the same colored region as the inputs, and order-invariant function further enforces the vector fields to be the same among all the $6$ colored regions as reflected in the symmetry in the visualization.} 
\end{figure}

We identify necessary and sufficient conditions for describing intra order-preserving functions, study their differentiability, and propose a novel neural network architecture that can represent complex intra order-preserving function through common neural network components. 
From practical point of view, we devise a new post-hoc network confidence calibration technique using different intra order-invariant sub-families. 
Because a post-hoc calibration function keeps the top-$k$ class prediction \emph{if and only if} it is an intra order-preserving function,
%
learning the post-hoc calibration function within the intra order-preserving family presents a solution to the dilemma between accuracy and flexibility faced in the previous approaches.
We conduct several experiments to validate the benefits of learning with these new functions for post-hoc network calibration.
The results demonstrate improvement over various calibration performance metrics, compared with the original network, temperature scaling~\cite{guo2017calibration}, and Dirichlet calibration~\cite{kull2019beyond}.
\vspace{-1mm}
\section{Problem Setup}
\vspace{-1mm}

\def\phib{\bm\phi}
\def\psib{\bm\psi}

We address the problem of calibrating neural networks for $n$-class classification. Let define $[n] \coloneqq \{1,\dots,n\}$,
$\ZZ\subseteq \R^d$ be the domain, $\YY  = [n]$ be the label space, and let $\Delta_n$ denote the $n-1$ dimensional unit simplex.
Suppose we are given a trained probabilistic predictor  $\phib_o:\R^d \to \Delta_n$ and a small calibration dataset $\DD_c$ of i.i.d. samples drawn from an unknown distribution $\pi$ on $\ZZ\times\YY$.
For simplicity of exposition, we assume that $\phib_o$ can be expressed as the composition $\phib_o \eqqcolon \smb \circ \gb$, with $\gb: \R^d \to \R^n$ being a non-probabilistic $n$-way classifier and $\smb: \R^n \to \Delta_n$ being the softmax operator\footnotemark, i.e.
$ \mathbf{sm}_i(\xb) = \frac{\exp(\xb_i)}{\sum_{j=1}^n \exp(\xb_j)}$, for $i \in \YY$, 
where the subscript~$_i$ denotes the $i$th element of a vector.
When queried at $\zb\in\ZZ$, the probabilistic predictor $\phib_o$ returns $\argmax_i \phib_{o,i}(\zb)$ as the predicted label and $\max_i \phib_{o,i} (\zb)$ as the associated confidence score. (The top-$k$ prediction is defined similarly.)  We say $\gb(\zb)$ is the \emph{logits} of $\zb$ with respect to $\phib_o$. 

\footnotetext{The softmax requirement is not an assumption but for making the notation consistent with the literature. The proposed algorithm can also be applied to the output of general probabilistic predictors.}

Given $\phib_o$ and $\DD_c$, 
our goal is to learn a post-hoc calibration function $\fb: \R^n\to\R^n$ such that the new probabilistic predictor $\phib \coloneqq \smb \circ \fb \circ \gb$ is better calibrated \emph{and} keeps the accuracy (or similar performance concepts like top-$k$ accuracy) of the original network $\phib_o$. That is, we want to learn new logits $\fb(\gb(\zb))$ of $\zb$.
As we will discuss, this task is non-trivial, because while learning $\fb$ might improve calibration, doing so could also risk over-fitting to the small dataset $\DD_c$ and damaging accuracy. 
To make this statement more precise, below we first review the definition of perfect calibration~\cite{guo2017calibration} and common calibration metrics and then discuss challenges in learning $\fb$ with $\DD_c$.

\begin{definition}\label{df:perfect calibration}
For a distribution $\pi$ on $\ZZ\times\YY$ and a probabilistic predictor $\psib:\R^d \to \Delta_n$, let random variables $\zb \in\ZZ$, $y\in\YY$ be distributed according to $\pi$, and define random variables $\hat{y}  \coloneqq \argmax_i \psib_i(\zb)$ and $\hat{p} \coloneqq \psib_{\hat{y}}(\zb)$.
We say $\psib$ is \emph{perfectly calibrated} with respect to $\pi$, if for any $p\in [0,1]$,
it satisfies $\textrm{Prob}(\hat{y}=y| \hat{p} =p ) = p$. 
\end{definition}
Note that $\zb$, $y$, $\hat{y}$ and $\hat{p}$ are correlated random variables. Therefore, \cref{df:perfect calibration} essentially means that, if $\psib$ is perfectly calibrated, then for any $p\in[0,1]$, the true label $y$ and the predicted label $\hat{y}$ match, with a probability exactly $p$ in the events where $\zb$ satisfies $\max_i \psib_{i}(\zb) = p$. 


In practice, learning a perfectly calibrated predictor is unrealistic, so we need a way to measure the calibration error. A common calibration metric is called Expected Calibration Error (ECE)~\cite{naeini2015obtaining}:
$	\mathrm{ECE} = \sum_{m=1}^M \frac{|B_m|}{N} |\mathrm{acc}(B_m) - \mathrm{conf}(B_m)|
$.
This equation is calculated in two steps: First the confidence scores of samples in $\DD_c$ are partitioned into $M$ equally spaced bins $\{B_m\}_{m=1}^M$. Second the weighted average of the differences between the average confidence $\mathrm{conf}(B_m) = \frac{1}{|B_m|}\sum_{i \in B_m} \hat{p}^i$ and the accuracy $\mathrm{acc}(B_m) = \frac{1}{|B_m|}\sum_{i \in B_m} \mathbbm{1}( y^i = \hat{y}^i) $ in each bin is computed as the ECE metric, 
where $|B_m|$ denotes the size of bin $B_m$, $\mathbbm{1}$ is the indicator function, and the superscript $^i$ indexes the sampled random variable.
%
In addition to ECE, other calibration metrics have also been proposed~\cite{guo2017calibration,nixon2019measuring,brier1950verification,kumar2019calibration}; e.g., Classwise-ECE~\cite{kull2019beyond} and Brier score~\cite{brier1950verification} are proposed as measures of classwise-calibration. All the metrics for measuring calibration have their own pros and cons. Here, we consider the most commonly used metrics for measuring calibration and leave their analysis for future work.

While the calibration metrics above measure the deviation from perfect calibration in \cref{df:perfect calibration}, they are usually not suitable loss functions for optimizing neural networks, e.g., due to the lack of continuity or non-trivial computation time.
Instead, the calibration function $\fb$ in $\phib = \smb \circ \fb \circ \gb$ is often optimized indirectly through a surrogate loss function (e.g. the negative log-likelihood) defined on the held-out calibration dataset $\DD_c$~\cite{guo2017calibration}.

\vspace{-1mm}
\subsection{Importance of Inductive Bias}
Unlike regular deep learning scenarios, here the calibration dataset $\DD_c$ is relatively small. 
Therefore, controlling the capacity of the hypothesis space of $\fb$ becomes a crucial topic~\cite{guo2017calibration, kull2017beta, kull2019beyond}. There is typically a trade-off between preserving accuracy and improving calibration:
Learning $\fb$ could improve the calibration performance, but it could also change the decision boundary of $\phib$ from $\phib_o$ decreasing the accuracy.
%
%
While using simple calibration functions may be applicable when $\phi_o$ has a simple function landscape or is already close to being well calibrated, such a function class might not be sufficient to calibrate modern deep networks with complex decision boundaries as we will show in the experiments in \cref{sec:experiments}.

The observation above motivates us to investigate the possibility of learning calibration functions within a hypothesis space that can provably guarantee preserving the accuracy of the original network $\phib_o$.  
The identification of such functions would address the previous dilemma and give precisely the needed structure to ensure generalization of calibration when the calibration datatset $\DD_c$ is small.

\vspace{-1mm}
\section{Intra Order-Preserving Functions}
\vspace{-1mm}

In this section, we formally describe this desirable class of functions for post-hoc network calibration. We name them  \emph{intra order-preserving functions}. 
Learning within this family is both necessary and sufficient to keep the top-$k$ accuracy of the original network unchanged. 
We also study additional function structures on this family (e.g. limiting how different dimensions can interact), which can be used as regularization in learning calibration functions.
Last, we discuss a new neural network architecture for representing these functions.

\vspace{-1mm}
\subsection{Setup: Sorting and Ranking}
We begin by defining {sorting functions} and {ranking} in preparation for the formal definition of intra order-preserving functions.
Let $\Pbb^n \subset \{0,1\}^{n\times n}$ denote the set of $n\times n$ permutation matrices. Sorting can be viewed as a permutation matrix; Given a vector $\xb\in\R^n$, we say $S:\R^n \to \Pbb^n$ is a \emph{sorting function} if $\yb = S(\xb)\xb$ satisfies $\yb_1\geq \yb_2 \geq \dots \geq \yb_n$. In case there are ties in the input vector $\xb$, the sorting matrix can not be uniquely defined. To resolve this, we use a pre-defined \emph{tie breaker} vector which is used as a tie breaking protocol. We say a vector $\tb \in \R^n$ is a tie breaker if $\tb = P \rb$, for some $P \in \Pbb^n$, where $\rb = [1,\dots, n]^\t \in \R^n$. Tie breaker pre-assigns priorities to indices of the input vector and is used to resolve ties. For instance, $\Sb_1=\bigl[ \begin{smallmatrix}1 & 0\\ 0 & 1\end{smallmatrix}\bigr]$ and $\Sb_2=\bigl[ \begin{smallmatrix}0 & 1\\ 1 & 0\end{smallmatrix}\bigr]$ are the unique sorting matrices of input $\xb=[0, 0]^\top$ with respect to tie breaker $\tb_1 = [1, 2]^\top$ and $\tb_2 = [2, 1]^\top$, respectively. 
We say two vectors $\ub,\vb\in\R^n$ \emph{share the same ranking} if $S(\ub) = S(\vb)$ for \emph{any} tie breaker $\tb$.

\vspace{-1mm}
\subsection{Intra Order-Preserving Functions}

We define the \emph{intra} order-preserving property with respect to different coordinates of a vector input. 

\begin{definition}\label{def:intra_order_preserving}
We say a function $\fb: \R^n \to \R^n$ is \emph{intra order-preserving}, if, for any $\xb\in\R^n$, both $\xb$ and $\fb(\xb)$ share the same ranking.
\end{definition}
The output of an intra order-preserving function $\fb(\xb)$ maintains \emph{all} ties and strict inequalities between elements of the input vector $\xb$. Namely, for all $i,j\in[n]$, we have $\xb_i>\xb_j$ (or $\xb_i=\xb_j$) if and only if $\fb_i(\xb)>\fb_j(\xb)$ (or $\fb_i(\xb)=\fb_j(\xb)$). 
For example, a simple intra order-preserving function is the temperature scaling  $\fb(\xb) = \xb/t$ for some $t>0$. Another common instance is the softmax operator. 

Clearly, applying an intra order-preserving function as the calibration function in $\phib = \smb\circ\fb\circ\gb$ does not change top-$k$ predictions between $\phib$ and $\phib_o = \smb\circ\gb$. 
Next, we provide a necessary and sufficient condition for constructing continuous, intra order-invariant functions. This theorem will be later used to design neural network architectures for learning calibration functions. 
Note that for a vector $\vb\in\R^n$ and an upper-triangular matrix of ones $U$, $U\vb$ is the \emph{reverse} cumulative sum of $\vb$ (i.e. $(U\vb)_i = \sum_{j=i}^n \vb_i$).

\begin{restatable}{theorem}{stOrderPreserv} \label{st:order_preserv}
A continuous function $\fb:\R^n\to\R^n$ is intra order-preserving, if and only if $\fb(\xb) = S(\xb)^{-1}U \wb(\xb)$ with $U$ being an upper-triangular matrix of ones and $\wb:\R^n\to\R^n$ being a continuous function such that
\begin{itemize}\vspace{-2mm}
    \item $\wb_i(\xb) = 0$, if $\yb_i = \yb_{i+1}$ and $i<n$, \vspace{-1mm}
    \item $\wb_i(\xb) > 0$,  if $\yb_i > \yb_{i+1}$ and $i<n$,\vspace{-1mm}
    \item $\wb_n(\xb)$ is arbitrary,\vspace{-1mm}
\end{itemize}\vspace{-2mm}
where $\yb = S(\xb) \xb$ is the sorted version of $\xb$.
\end{restatable}

The proof is deferred to Appendix. Here we provide as sketch as to why \cref{st:order_preserv} is true. Since $\wb_i(\xb)\geq0$ for $i<n$, applying the matrix $U$ on $\wb(\xb)$ results in a sorted vector $U \wb(\xb)$. Thus, applying $S(\xb)^{-1}$ further on $U \wb(\xb)$ makes sure that $\fb(\xb)$ has the same ordering as the input vector $\xb$. The reverse direction can be proved similarly. For the continuity, observe that the sorting function $S(\xb)$ is piece-wise constant with discontinuities only when there is a tie in the input $\xb$. This means that if the corresponding elements in $U \wb(\xb)$ are also equally valued when a tie happens, the discontinuity of the sorting function $S$ does not affect the continuity of $\fb$ inherited from $\wb$.

\vspace{-1mm}
\subsection{Order-invariant and Diagonal Sub-families}

Different classes in a classification task typically have shared characteristics. Therefore, calibration functions sharing properties across different classes can work as a suitable inductive bias in learning.
Here we use this idea to define two additional structures interesting to intra order-preserving functions: 
\emph{order-invariant} and \emph{diagonal} properties. 
%
%
%
Similar to the purpose of the previous section, we will study necessary and sufficient conditions for functions with these properties.
 
First, we study the concept of order-invariant functions.  
\begin{definition}
\label{def:order_invariant}
We say a function $\fb: \R^n \to \R^n$ is \emph{order-invariant}, if $\fb(P\xb) = P\fb(\xb)$ for all $\xb \in \R^n$ and permutation matrices $P\in\Pbb^n$.
\end{definition}

For an order-invariant function $\fb$, 
when two elements $\xb_i$ and $\xb_j$ in the input $\xb$ are swapped, the corresponding elements $\fb_i(\xb)$ and $\fb_j(\xb)$ in the output $\fb(\xb)$ are also swapped. In this way, the mapping learned for the $i$th class can also be used for the $j$th class. 
Thus, the order-invariant family shares the calibration function between different classes while allowing the output of each class be a function of all other class predictions. 

We characterize in the theorem below the properties of functions that are both intra order-preserving and order-invariant (an instance is the softmax operator). 
It shows that, to make an intra order-preserving function also order-invariant, we just need to feed the function $\wb$ in \cref{st:order_preserv} with the sorted input $\yb=S(\xb)\xb$ instead of $\xb$.
This scheme makes the learning of $\wb$ easier since it always sees sorted vectors (which are a subset of $\R^n$).

\begin{restatable}{theorem}{stIntraAndOrderInvariant}\label{st:intra_and_order_invariant}
A continuous, intra order-preserving function $\fb:\R^n\to\R^n$ is order-invariant, if and only if $\fb(\xb) = S(\xb)^{-1}U \wb(\yb)$, where $U$, $\wb$, and $\yb$ are in \cref{st:order_preserv}.
\end{restatable}

Another structure of interest here is the diagonal property. 
\begin{definition}
We say a function $\fb: \R^n \to \R^n$ is \emph{diagonal}, if $\fb(\xb) = [f_1(\xb_1), \dots, f_n(\xb_n)]$ for  $f_i:\R\to\R$ with $i\in[n]$.
\end{definition}
In the context of calibration, a diagonal calibration function means that different class predictions do not interact with each other in $\fb$. 
Defining diagonal family is mostly motivated by the success of temperature scaling method~\cite{guo2017calibration}, which is a linear diagonal intra order-preserving function. 
%
Therefore, although diagonal intra order-preserving functions may sound limiting in learning calibration functions, they still represent a useful class of functions. 

The next theorem relates diagonal intra order-preserving functions to increasing functions.
\begin{restatable}{theorem}{stFulOrderPreservingDiagonal}\label{st:ful_order_preserving_diagonal}
A continuous, intra order-preserving function $\fb:\R^n\to\R^n$ is diagonal, if and only if $\fb(\xb) = [\bar{f}(\xb_1), \dots, \bar{f}(\xb_n)]$ for some continuous and increasing function $\bar{f}:\R\to\R$.
\end{restatable}

Compared with general diagonal functions, diagonal intra order-preserving automatically implies that the same function $\bar{f}$ is shared across all dimensions. 
Thus, learning with diagonal intra order-preserving functions benefits from parameter-sharing across different dimensions, which could drastically decrease the number of parameters. 

Finally, below we show that functions in this sub-family are also order-invariant and inter order-preserving. 
%
%
Note that inter and intra order-preserving are orthogonal definitions. 
Inter order-preserving is also an important property for calibration functions, since this property guarantees that $\fb_i(\xb)$ increases with the original class logit $\xb_i$. The set diagram in \cref{fig:set_diagram} depicts the relationship among different intra order-preserving families. 

\begin{definition}
We say a function $\fb: \R^m \to \R^n$ is \emph{inter order-preserving} if, for any $\xb,\yb\in\R^m$ such that $\xb\geq \yb$, $\fb(\xb) \geq \fb(\yb)$, where $\geq$ denotes elementwise comparison.
\end{definition}
\begin{corollary}\label{st:diag_intra_is_invar_inter}
A diagonal, intra order-preserving function is order-invariant and inter order-preserving 
\end{corollary}

\begin{figure}[t]
\centering
\includegraphics[width=0.3\linewidth]{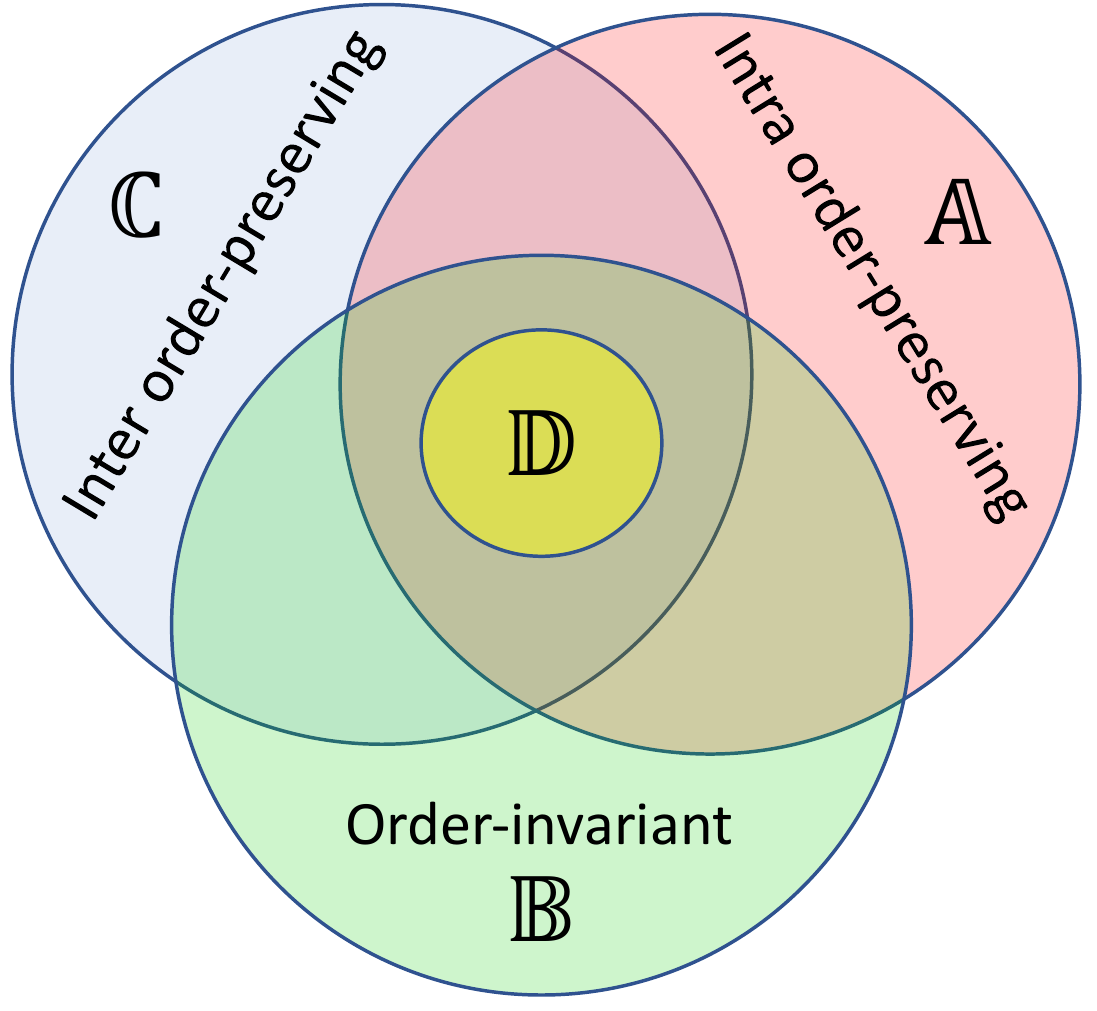}
\caption{\label{fig:set_diagram} Relationship between different function families. 
\cref{st:order_preserv} specifies the intra order-preserving functions $\Abb$.
\cref{st:intra_and_order_invariant} specifies the intra order-preserving and order-invariant functions $\Abb\cap\Bbb$.
\cref{st:ful_order_preserving_diagonal} specifies the diagonal intra order-preserving functions $\Dbb$.
By \cref{st:diag_intra_is_invar_inter}, these functions are also order-invariant and inter order-preserving i.e. $\Dbb \subseteq \Abb \cap \Bbb \cap \Cbb$.}
\end{figure}

\subsection{Practical Considerations} \label{sec:practical considerations}

%
\cref{st:order_preserv,st:intra_and_order_invariant} describe general representations of intra order-preserving functions through a function $\wb$ that satisfies certain non-negative constraints. 
Inspired by these theoretical results, we propose a neural network architecture, \cref{fig:presnet}, to represent exactly a family of intra order-preserving functions.

The main idea in \cref{fig:presnet} is to parameterize $\wb$ through a composition of smaller functions. For $i<n$, we set $\wb_i(\xb) = \sigma(\yb_i - \yb_{i+1}) \mb_i(\xb)$, where $\sigma:\R\to\R$ is a positive function such that $\sigma(a) = 0$ only when $a=0$, and $\mb_i$ is a strictly positive function. It is easy to verify that this parameterization of $\wb$ satisfies the requirements on $\wb$ in \cref{st:order_preserv}. However, we note that this class of functions cannot represent all possible $\wb$ stated in \cref{st:order_preserv}. 
In general, the speed $\wb_i(\xb)$ converges to $0$ can be a function of $\xb$, but in the proposed factorization above, the rate of convergence to zero is a function of only two elements $\yb_i$ and $\yb_{i+1}$. %
Fortunately, such a limitation does not 
substantially decrease the expressiveness of $\fb$ in practice, because the subspace where $\wb_i$ vanishes has zero measure in $\R^n$ (i.e. subspaces where there is at least one tie in $\xb\in\R^n$).

By \cref{st:order_preserv} and \cref{st:intra_and_order_invariant}, the proposed architecture in \cref{fig:presnet} ensures $\fb(\xb)$ is continuous in $\xb$ as long as $\sigma(\yb_i -\yb_{i+1})$ and $\mb_i(\xb)$ are continuous in $\xb$. In the appendix, we show that this is true when $\sigma$ and $\mb_i$ are continuous functions. 
Additionally, we prove that when $\sigma$ and $\mb$ are continuously differentiable, $\fb(\xb)$ is also directionally differentiable with respect to $\xb$. Note that the differentiability to the input is not a requirement to learn the parameters of $\mb$ with a first order optimization algorithm which only needs $\fb$ to be differentiable with respect to the parameters of $\mb$. The latter condition holds in general, since the only potential sources of non-differentiable $\fb$, $S(\xb)^{-1}$ and $\yb$ are constant with respect to the parameters of $\mb$. Thus, if $\mb$ is differentiable with respect to its parameters, $\fb$ is also differentiable with respect to the parameters of $\mb$.

\section{Implementation}\label{sec:implementation}

\def\thetab{\bm\theta}
Given a calibration dataset $\DD_c = \{(\zb^{i}, y^{i})\}_{i=1}^N$ and a calibration function $\fb$ parameterized by some vector $\thetab$, we define the empirical calibration loss as 
$\frac{1}{N} \sum_{i=1}^N \ell(y^{i}, \fb(\xb^{i})) + \frac{\lambda}{2} ||\thetab||^2$, 
where $\xb^i = \gb(\zb^i)$, 
$\ell:\YY\times\R^n\to\R$ is a classification cost function, and $\lambda \geq 0$ is the regularization weight. Here we follow the calibration literature~\cite{thulasidasan2019mixup,guo2017calibration,kull2019beyond} and use the negative log likelihood (NLL) loss, i.e.,
$
    \ell(y, \fb(\xb)) = -\log(\mathbf{sm}_y(\fb(\xb)))
$,
where $\mathbf{sm}$ is the softmax operator and $\mathbf{sm}_y$ is its $y$th element. 
We use the NLL loss in all the experiments to study the benefit of learning $\fb$ with different structures. 
The study of other loss functions for calibration~\cite{seo2019learning,xing2020distancebased} is outside the scope of this paper.

To ensure $\fb$ is within the intra order-preserving family, we restrict $\fb$ to have the structure in \cref{st:order_preserv} and set $\wb_i(\xb) = \sigma(\yb_i - \yb_{i+1}) \mb(\xb)$, as described in \cref{sec:practical considerations}. 
We parameterize function $\mb$ by a generic multi-layer neural network and utilize the softplus activation $s^+(a) = \log(1+\exp(a))$ on the last layer when strict positivity is desired and represent $\sigma$ as $\sigma(a) = |a|$.
%
%
For example, when $\mb_i(\xb)$ is constant, our architecture recovers the temperature scaling scheme~\cite{guo2017calibration}. 

The order-invariant version in \cref{st:intra_and_order_invariant} can be constructed similarly. The only difference is that the neural network that parameterizes $\mb$ receives instead the sorted input. \cref{fig:presnet} illustrates the architecture of these models.

The diagonal intra order-preserving version in \cref{st:ful_order_preserving_diagonal} is formed by learning an increasing function shared across all logit dimensions. We use the official implementation of proposed architecture in~\cite{wehenkel2019unconstrained} that learns monotonic functions with unconstrained neural networks.

\begin{figure*}[t]
    \centering
    \includegraphics[width=1.0\linewidth]{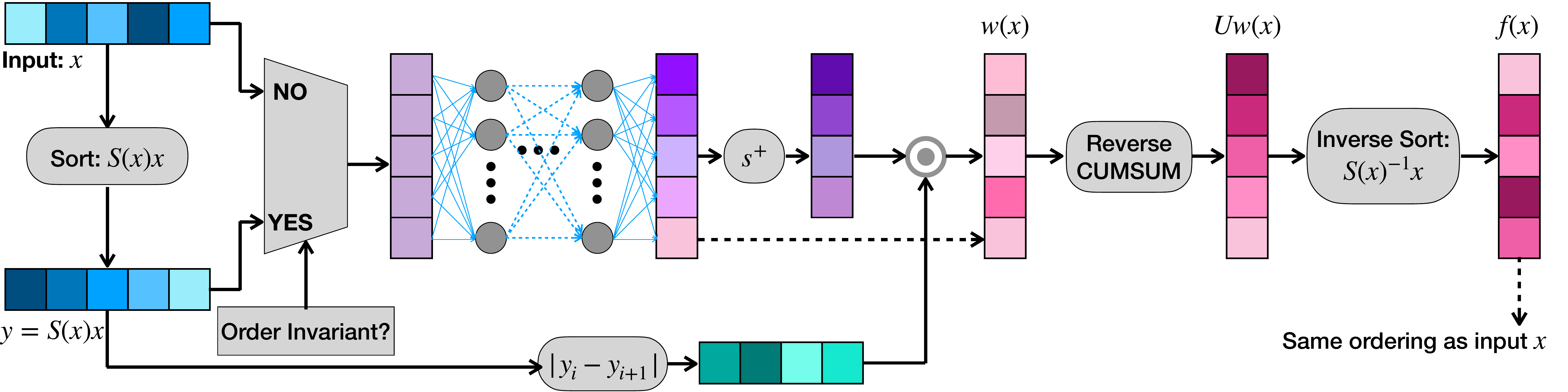}
    \caption{\label{fig:presnet} \footnotesize Flow graph of the intra order-preserving function. The vector $\xb\in\R^n$ is the input to the graph. Function $\mb$ is estimated using a generic multi-layer neural network with non-linear activation for the hidden layers. The input to the network is sorted for learning order-preserving functions. We employ softplus activation function $s^+$ to impose strict positivity constraints.}
    \vspace{-.6cm}
\end{figure*}

\vspace{-1mm}
\section{Related Work} 
\vspace{-1mm}

Many different post-hoc calibration methods have been studied in the literature~\cite{platt1999probabilistic, guo2017calibration,kull2019beyond,kull2017beyond,kull2017beta,kumar2019calibration}. Their main difference is in the parametric family of the calibration function.
In Platt scaling~\cite{platt1999probabilistic}, scale and shift parameters $a,b \in \R$ are used to transform the scalar logit output $x \in \R$ i.e. $f(x) = ax+b$ of a binary classifier. Temperature scaling~\cite{guo2017calibration} is a simple extension of Platt scaling for multi-class calibration in which only a single scalar temperature parameter is learned. Dirichlet calibration~\cite{kull2019beyond} allows learning within a richer linear functions family $f(\xb) = W\xb + \bb$, where $W \in \R^{n\times n}$ and $\bb \in \R^n$ but the learned calibration function may also change the decision boundary of the original model; Kull \etl~\cite{kull2019beyond} suggested regularizing the off-diagonal elements of $W$ to avoid overfitting. 
Similar to our work, the concurrent work in Zhang~\etl~\cite{zhang2020mix} also give special attention to order preserving transformations for calibration. However, their introduced functions are less expressive than the ones presented in this work.
Earlier works like isotonic regression~\cite{zadrozny2002transforming}, histogram binning~\cite{zadrozny2001obtaining}, and Bayesian binning~\cite{zadrozny2002transforming} are also post-hoc calibration methods.

In contrast to post-hoc calibration methods, several researches proposed to modify the  training process to learn a calibrated network in the first place. Data augmentation methods~\cite{thulasidasan2019mixup,yun2019cutmix} overcome overfitting by enriching the training data with new artificially generated pseudo data points and labels. Mixup~\cite{mixup} creates pseudo data points by computing the convex combination of randomly sampled pairs. Cutmix~\cite{yun2019cutmix} uses a more efficient combination algorithm specifically designed for image classification in which two images are combined by overlaying a randomly cropped part of the first image on the second image. In label smoothing~\cite{pereyra2017regularizing, muller2019does}, the training loss is augmented to penalize high confidence outputs. To discourage overconfident predictions, \cite{seo2019learning} modifies the original NNL loss by adding a cross-entropy loss term with respect to the uniform distribution. Similarly, \cite{kumar2018trainable} adds a calibration regularization to the NLL loss via kernel mean embedding.

Bayesian neural networks~\cite{gal2016dropout,maddox2019simple} derive the uncertainty of the prediction by making stochastic perturbations of the original model. Notably, \cite{gal2016dropout} uses dropout as approximate Bayesian inference. \cite{maddox2019simple} estimates the posterior distribution over the parameters and uses samples from this distribution for Bayesian model averaging. These methods are computationally inefficient since they typically feed each sample to the network multiple times.

\vspace{-1mm}

\vspace{-1mm}
\section{Experiments} \label{sec:experiments}
\vspace{-1mm}

\def\oinv{\textsc{OI}\xspace}
\def\dgfn{\textsc{Diag}\xspace}
\def\opre{\textsc{OP}\xspace}
\def\uncons{\textsc{Unconstrained}\xspace}
\def\msodir{\textsc{MS}\xspace}
\def\dirodir{\textsc{Dir}\xspace}

We evaluate the performance of intra order-preserving (\opre), order-invariant intra order-preserving (\oinv), and diagonal intra order-preserving (\dgfn) families in calibrating the output of various image classification deep networks and compare their results with the previous post-hoc calibration techniques. 

{\bf\noindent Datasets.} We use six different datasets: CIFAR-\{10,100\}~\cite{krizhevsky2009learning}, SVHN~\cite{netzer2011reading}, CARS~\cite{KrauseStarkDengFei-Fei_3DRR2013}, BIRDS~\cite{WelinderEtal2010}, and ImageNet~\cite{deng2009imagenet}. 
In these datasets, the number of classes vary from $10$ to $1000$. 
We evaluate the performance of different post-hoc calibration methods to calibrate ResNet~\cite{he2016deep}, Wide ResNet~\cite{Zagoruyko2016WRN}, DenseNet~\cite{huang2017densely}, and PNASNet5~\cite{liu2018progressive} networks.
We follow the experiment protocol in~\cite{kull2019beyond,kull2017beyond} and use cross validation on the calibration dataset to find the best hyperparameters and architectures for all the methods. Please refer to the Appendix for detailed description of the datasets, pre-trained networks, and hyperparameter tuning.

\begingroup
\setlength{\tabcolsep}{4pt}
\begin{table*}[t]
	\begin{center}
	    \caption{\em \label{tbl:ece_paper}\footnotesize{ECE (with $M=15$ bins) on various image classification datasets and models with different calibration methods. The subscript numbers represent the rank of the corresponding method on the given model/dataset. The accuracy of the uncalibrated model is shown in parentheses. The number in parentheses in \dirodir, \msodir, and \uncons methods show the change in accuracy for each method.}}
		\resizebox{0.99\linewidth}{!}{%
			\begin{tabular}{c c | c c c c | c c c|c}
				\toprule
				Dataset  & Model & Uncal. & TS & \dirodir & \msodir & \dgfn & \oinv & \opre & \uncons \\
				\hline
				
				CIFAR10 & ResNet~110 & $0.0475_8 (93.6\%)$ & $0.0113_5$ & $0.0109_4 (-0.1\%)$ & $0.0106_3 (-0.1\%)$ & $0.0067_2$ & ${\bf 0.0061}_1$ & $0.0119_6$ & $0.0170_7 (-0.4\%)$ \\
				CIFAR10 & Wide ResNet~32 & $0.0451_8 (93.9\%)$ & $0.0078_4$ & $0.0084_5 (+0.3\%)$ & $0.0073_2 (+0.3\%)$ & $0.0136_7$ & ${\bf 0.0064}_1$ & $0.0077_3$ & $0.0097_6 (-0.1\%)$ \\
				CIFAR10 & DenseNet~40 & $0.0550_8 (92.4\%)$ & $0.0095_2$ & $0.0110_4 (+0.1\%)$ & $0.0099_3 (+0.1\%)$ & ${\bf 0.0069}_1$ & $0.0116_5$ & $0.0128_7$ & $0.0125_6 (-0.5\%)$\\
				\hline
				SVHN & ResNet~152 (SD) & $0.0086_8 (98.1\%)$ & $0.0061_5$ & $0.0058_3 (+0.0\%)$ & $0.0060_4 (+0.0\%)$ & $0.0057_2$ & $00.0116_6$ & $0.0118_7$ & $\mathbf{0.0015}_1 (+0.0\%)$\\
				\hline
				CIFAR100 & ResNet~110 & $0.1848_8 (71.5\%)$ & $0.0238_2$ & $0.0282_5 (+0.2\%)$ & $0.0274_4 (+0.1\%)$ & $0.0507_7$ & ${\bf 0.0119}_1$ & $0.0253_3$ & $0.0346_6 (-4.4\%)$\\
				CIFAR100 & Wide ResNet~32 & $0.1878_8 (73.8\%)$ & $0.0147_2$ & $0.0189_5(+0.1\%)$ & $0.0258_6(+0.1\%)$ & $0.0172_3$ & ${\bf 0.0126}_1$ & $0.0173_4$ & $0.0421_7 (-6.1\%)$ \\
				CIFAR100 & DenseNet~40 & $0.2116_8 (70.0\%)$ & $0.0090_2$ & $0.0114_4 (+0.1\%)$ & $0.0220_6(+0.4\%)$ & ${\bf 0.0075}_1$ & $0.0098_3$ & $0.0154_5$ & $0.0990_7 (-12.9\%)$\\
				\hline
				CARS & ResNet~50 (pre) & $0.0239_7 (91.3\%)$ & $0.0144_3$ & $0.0243_8 (+0.2\%)$ & $0.0186_6(-0.3\%)$ & $0.0105_2$ & ${\bf 0.0103}_1$ & $0.0185_5$ & $0.0182_4 (-3.5\%)$\\
				CARS & ResNet~101 (pre) & $0.0218_7 (92.2\%)$ & $0.0165_5$ & $0.0225_8 (+0.0\%)$ & $0.0191_6(-0.8\%)$ & ${\bf 0.0102}_1$ & $0.0135_3$ & $0.0125_2$ & $0.0155_4 (-3.9\%)$\\
				CARS & ResNet~101 & $0.0421_8 (85.2\%)$ & $0.0301_4$ & $0.0245_3 (-0.3\%)$ & $0.0345_6(-1.1\%)$ & ${\bf 0.0206}_1$ & $0.0323_5$ & $0.0358_7$ & $0.0236_2 (-7.0\%)$ \\
				\hline
				BIRDS & ResNet~50 (NTS) & $0.0714_8 (87.4\%)$ & $0.0319_5$ & $0.0486_6 (-0.2\%)$ & $0.0585_7 (-1.1\%)$ & $0.0188_2$ & ${\bf 0.0172}_1$ & $0.0292_4$ & $0.0276_3 (-2.2\%)$\\
				\hline 
				ImageNet & ResNet~152 & $0.0654_7 (76.2\%)$ & $0.0208_4$ & $0.0452_5 (+0.1\%)$ & $0.0567_6 (+0.1\%)$ & ${\bf 0.0087}_1$ & $0.0109_2$ & $0.0167_3$ & $0.1297_8 (-33.4\%)$ \\
				ImageNet & DenseNet~161 & $0.0572_7 (77.1\%)$ & $0.0198_4$ & $0.0374_5 (+0.1\%)$ & $0.0443_6 (+0.4\%)$ & ${\bf 0.0103}_1$ & $0.0123_2$ & $0.0168_3$ & $0.1380_8 (-28.1\%)$\\
				ImageNet & PNASNet5~large & $0.0610_7 (83.1\%)$ & $0.0713_8$ & $0.0398_6 (+0.0\%)$ & $0.0217_4 (+0.3\%)$ & $0.0117_2$ & ${\bf 0.0084}_1$ & $0.0133_3$ & $0.0316_5 (-4.8\%)$\\
				\midrule
				\multicolumn{2}{c|}{Average Relative Error} & $1.00_8$ & $0.42_4$ & $0.49_5$ & $0.50_6$ & $\mathbf{0.27}_1$ & $0.33_2$ & $0.41_3$ & $0.66_7$ \\
				\bottomrule
			\end{tabular}
		}
	\end{center}
	\vspace{-.2cm}
\end{table*}
\endgroup

		
{\bf\noindent Baselines.} We compare the proposed function structures with temperature scaling (TS)~\cite{guo2017calibration}, Dirichlet calibration with off-diagonal regularization (\dirodir)~\cite{kull2019beyond}, and matrix scaling with off-diagonal regularization (\msodir)~\cite{kull2019beyond} as they are the current best performing post-hoc calibration methods. We also present the results of the original uncalibrated models (Uncal.) for comparison. To show the effect of intra order-preserving regularization, we also show the results of applying unconstrained multi-layer neural network without intra order-preserving constraint (\uncons). In cross-validation, we tune the architecture as well as regularization weight of \uncons and order-preserving functions.
As we are using the same logits as~\cite{kull2019beyond}, we report their results directly on CIFAR-10, CIFAR-100, and SVHN. However, since they do not present the results for CARS, BIRDS, and ImageNet datasets, we report the results of their official implementation\footnote{\href{https://github.com/dirichletcal/experiments_dnn/}{https://github.com/dirichletcal/experiments\_dnn/}} on these datasets.


\begin{figure*}[!t]
        \centering
        \includegraphics[width=0.9\linewidth]{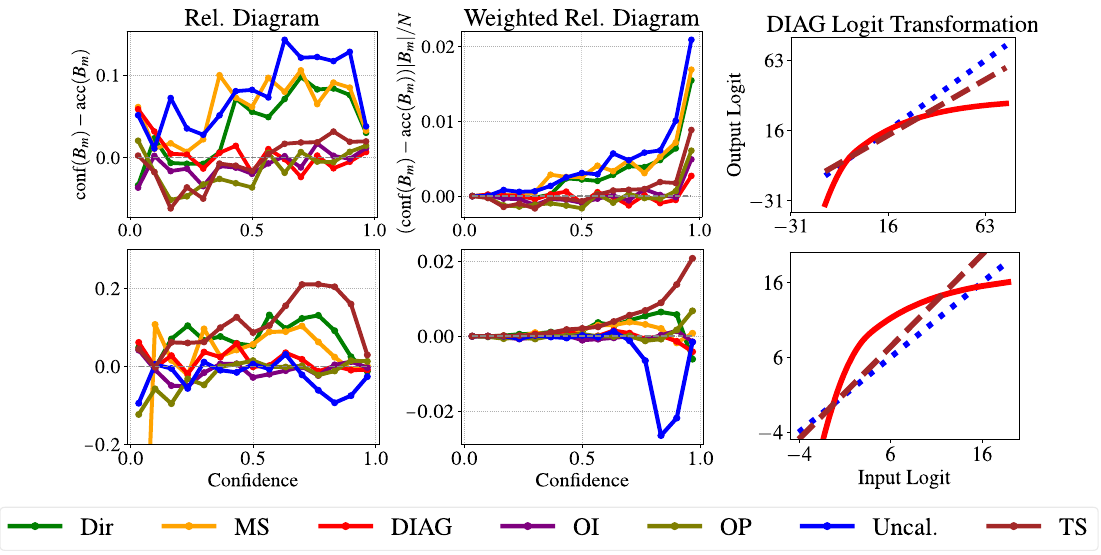}
	\caption{\label{fig:transformations} \footnotesize Performance evaluations of ResNet~152 (\textbf{Top Row}) and PNASNet5~large (\textbf{Bottom Row}) on ImageNet dataset. (\textbf{Left}) Reliability diagrams. As suggested by~\cite{maddox2019simple} we show the difference between the estimated confidence and accuracy over $M=15$ bins. The dashed grey lines represent the perfectly calibrated network at $y=0$. Points above~(below) the grey line show overconfident~(underconfident) predictions in a bin. (\textbf{Middle}) Weighted reliability diagrams where bin values are weighted by data frequency distribution. Since the uncalibrated network has different distances to the perfect calibration in different bins, scaling by a single temperature will lead to a mix of underconfident and overconfident regions. Our order-preserving functions, on the other hand, have more flexibility to reduce the calibration error. (\textbf{Right}) Transformation learned by \dgfn function compared to temperature scaling and uncalibrated model (identity map).}
\end{figure*}


{\bf\noindent Results.} \cref{tbl:ece_paper} summarizes the results of our calibration methods and other baselines in terms of ECE and presents the average relative error with respect to the uncalibrated model.
%
%
%
Overall, \dgfn has the lowest average relative error followed by \oinv among the models and datasets presented in \cref{tbl:ece_paper}. \oinv is the best-performing method in 7 out of 14 experiments including ResNet~110 and Wide~ResNet~32 models on CIFAR datasets as well as state-of-the-art PNASNet5~large model. \dgfn family's relative average error is half the \msodir and \dirodir methods and $15\%$ less compared to Temp. Scaling. 
Although \dirodir and \msodir were able to maintain the accuracy of the original models in most of the cases by imposing off diagonal regularization, order-preserving family could significantly outperform them regarding the ECE metric. Finally, we remark that learning an unconstrained multi-layer neural network does not exhibit a good calibration performance and drastically hurts the accuracy in some datasets as shown in the last column of \cref{tbl:ece_paper}.

\cref{fig:transformations} illustrates the reliability diagrams of models trained on ResNet~152 (top row) and PNASNet5~large (bottom row).  Weighted reliability diagrams are also presented to better indicate the differences regarding the ECE metric. Surprisingly, these diagrams show that the uncalibrated PNASNet5~large model is underconfident.  The differences between the mappings learned by \dgfn and temperature scaling on these models are illustrated on the right column. \dgfn is capable of learning complex increasing functions while temperature scaling only scales all the logits. Compared with \dirodir and \msodir which learn a linear transformation, all intra order-preserving methods can learn non-linear transformations on the logits while decoupling accuracy from calibration of the predictions.  

In addition to ECE, which considers the top prediction, we also measure the NLL, Marginal Calibration Error~\cite{kumar2019calibration}, Classwise-ECE, and Berier score. As it is shown in \cref{tbl:summary}, \dgfn and \oinv have the best overall performance in terms of average relative error in most cases, while \dirodir is the top performing method in Classwise-ECE.
Refer to the Appendix for discussions and the performance comparisons over all the datasets.

\begingroup
\setlength{\tabcolsep}{4pt}
\begin{table*}[h]
	\begin{center}
	    \caption{\em \label{tbl:summary}\footnotesize{Average relative error. Each entry shows the relative error compared to the uncalibrated model averaged over all the datasets. The subscripts represent the rank of the corresponding method on the given metric. See the Appendix for per dataset performance comparisons.}}
		\resizebox{0.8\linewidth}{!}{%
			\begin{tabular}{c | c c c c | c c c}
				\toprule
				Evaluation Metric & Uncal. & TS & \dirodir & \msodir & \dgfn & \oinv & \opre \\
				\hline
				ECE & $1.000_7$ & $0.420_4$ & $0.490_5$ & $0.500_6$ & $\mathbf{0.270}_1$ & $0.330_2$ & $0.410_3$ \\
				Debiased ECE~\cite{kumar2019calibration} & $1.000_7$ & $0.357_3$ & $0.430_6$ & $0.409_5$ & ${\bf 0.213}_1$ & $0.337_2$ & $0.406_4$\\
				NLL & $1.000_7$ & $0.766_4$ & $0.772_6$ & $0.768_5$ & $\mathbf{0.749}_1$ & $0.751_2$ & $0.765_3$\\
				Marginal Caliration Error~\cite{kumar2019calibration} & $1.000_7$ & $0.750_3$ & $0.735_2$ & $0.996_6$ & ${\bf 0.725}_1$ & $0.778_4$ & $0.898_5$ \\
				Classwise-ECE & $1.000_7$ & $0.752_6$ & $\mathbf{0.704}_1$ & $0.734_3$ & $0.729_2$ & $0.740_4$ & $0.743_5$ \\
				Brier & $1.000_7$ & $0.936_5$ & $0.930_3$ & $0.936_5$ & $\mathbf{0.924}_1$ & $0.929_2$ & $0.931_4$ \\
				\bottomrule
			\end{tabular}
		}
	\end{center}
	\vspace{-.2cm}
\end{table*}
\endgroup

\vspace{-0.3cm}
\section{Conclusion} \label{sec:conclusion}
\vspace{-0.2cm}
In this work, we introduce the family of intra order-preserving functions which retain the top-$k$ predictions of any deep network when used as the post-hoc calibration function.
We propose a new neural network architecture to represent these functions, and new regularization techniques based on order-invariant and diagonal structures. 
In short, calibrating neural network with this new family of functions generalizes many existing calibration techniques, with additional flexibility to express the post-hoc calibration function.
The experimental results show the importance of learning within the intra order-preserving family as well as support the effectiveness of the proposed regularization in calibrating multiple classifiers on various datasets.
%

We believe the applications of intra order-preserving family are not limited to network calibration. 
Other promising domains include, e.g., data compression, depth perception system calibration, and tone-mapping in images where tone-maps need to be monotonic. 
Exploring the applicability of  intra order-preserving functions in other applications is an interesting future direction.

\clearpage
\section*{Broader Impact}
Predicting calibrated confidence scores for multi-class deep networks is important for avoiding rare but costly mistakes. Trusting the network's output naively as confidence scores in system design could cause undesired consequences: a serious issue for applications where mistakes are costly, such as medical diagnosis, autonomous driving, suspicious events detection, or stock-market. As an example, in medical diagnosis, it is vital to estimate the chance of a patient being recovered by a certain operation given her/his condition. If the estimation is overconfident/underconfident this will put the life of the patient at risk. 
Confidence calibrated models would enable integration into downstream decision-making systems, allow machine learning interpretability, and help gain the user trust. 
While this work focuses primarily on some of the theoretical aspects of the neural network calibration, it also proposes novel techniques to potentially improve broader set of applications where preserving the rank of set of inputs is desired e.g. tone-mapping in images where tone-maps need to be monotonic, depth perception system calibration, and data compression.

We need to remark that our research shows that methods perform differently under various calibration metrics. Unfortunately, discrepancy between different calibration metrics is not well understood and fully explored in the literature. We believe more insights into these inconsistencies would be valuable to the field. We report the performance under different calibration metrics to highlight these differences for the future research. This also means that integrating the proposed work or any other calibration method into decision making systems requires application specific considerations. Other than that, since this work is mostly on the theoretical aspect of improving calibration, we do not foresee any direct negative impacts.

\begin{ack}
We would like to thank Antoine Wehenkel for providing helpful instructions for his unconstrained monotonic networks.
This research is supported in part by the Australia Research Council Centre of Excellence for Robotics Vision
(CE140100016).
\end{ack}

\bibliography{ms}
\bibliographystyle{plain}

\clearpage
\onecolumn
\appendix

\title{Intra Order-Preserving Functions for \\ Calibration of Multi-Class Neural Networks \\ Supplementary Material}

\maketitle

\section{Missing Proofs}

\subsection{Proof of \cref{st:order_preserv}, Intra Order-preserving Functions}
\stOrderPreserv*
\begin{proof}[Proof of \cref{st:order_preserv}]

($\rightarrow$) 
For a continuous intra order-preserving function $\fb(\xb)$, let $\wb(\xb) = U^{-1} S(\xb) \fb(\xb)$. 
First we show $\wb$ is continuous.
Because $\fb$ is intra order-preserving, it holds that $S(\xb) = S(\fb(\xb))$.
Let $\hat{\fb}(\xb) \coloneqq S(\fb(\xb)) \fb(\xb)$ be the sorted version of $\fb(\xb)$. 
The above implies  $\wb(\xb)=  U^{-1} \hat{\fb}(\xb)$. 
By \cref{st:sorting_continuous}, we know  $\hat{\fb}$ is continuous and therefore $\wb$ is also continuous.

\begin{lemma}\label{st:sorting_continuous}
Let $\fb:\R^n\to\R^n$ be a continuous intra order-preserving function.  $S(\fb(\xb)) \fb(\xb)$ is a continuous function.
\end{lemma}

Next, we show that $\wb$ satisfies the properties listed in \cref{st:order_preserv}.
As $\wb(\xb) = U^{-1} \hat{\fb}(\xb)$, we can equivalently write $\wb$ as
\begin{equation*}
    \wb_i(\xb) = \begin{cases}
    \hat{\fb}_{i}(\xb) - \hat{\fb}_{i+1}(\xb) &1 \le i < n \\
    \hat{\fb}_{n}(\xb) &i = n.
    \end{cases}
\end{equation*}
Since $\hat{\fb}$ is the sorted version of $\fb$, it holds that $\wb_i(\xb) \geq 0$ for $1 \le i < n$. Also, by the definition of the order-preserving function, $\wb_i(\xb)$ can be zero if and only if $\yb_i = \yb_{i+1}$, where $\yb = S(\xb)\xb$. These two arguments prove the necessary condition.  \\ \\
($\leftarrow$) For a given $\wb(\xb)$ satisfying the condition in the theorem statement, let $\vb(\xb) = U\wb(\xb)$. Equivalently, we can write $\vb_i(\xb) = \sum_{j=0}^{n-i} \wb_{n-j}(\xb)$ and $\vb_i(\xb) - \vb_{i+1}(\xb) = \wb_i(\xb)$,  $\forall i \in [n]$. By construction of $\wb$, one can conclude that $\vb(\xb)$ is a sorted vector where two consecutive elements $\vb_i(\xb)$ and $\vb_{i+1}(\xb)$ are equal if and only if $\yb_i=\yb_{i+1}$. 
Therefore, $\fb(\xb) = S(\xb)^{-1} \vb(\xb)$ has the same ranking as $\xb$. 
In other words, $\fb$ is an intra order-preserving function. 
The continuity of $\fb$ follows from the lemma below and the fact that $\vb$ is continuous when $\wb$ is continuous. \cref{st:sorting_inv_continuous}.

\begin{lemma}\label{st:sorting_inv_continuous}
Let $\vb: \R^n \to \R^n$ be a continuous function in which $\vb_i(\xb)$ and $\vb_{i+1}(\xb)$ are equal if and only if $\yb_i = \yb_{i+1}$, where $\yb = S(\xb) \xb$. Then $\fb(\xb) = S(\xb)^{-1} \vb(\xb)$ is a continuous function.
\end{lemma}

\end{proof}

\subsubsection{Deferred Proofs of Lemmas}

\begin{proof}[Proof of \cref{st:sorting_continuous}]
Let $\Pbb^n = \{P_1, \dots, P_K\}$ be the finite set of all possible $n\times n$ dimensional permutation matrices. For each $k\in[K]$, define the closed set $\Nbb_k = \{\xb : S(\xb) \xb = P_k \xb \}$. These sets are convex polyhedrons since each can be defined by a finite set of linear inequalities; in addition, they together form a covering set of $\R^n$. 
Note that $S(\xb) = P_k$ is constant in the interior $int(\Nbb_k)$, but $S(\xb)$ may change on the boundary $\partial(\Nbb_k)$ which corresponds to  points where a tie exists in elements of $\xb$ (for such a point $S(\xb) \neq P_k$).
Nonetheless, by definition of the set $\Nbb_k$, we have $S(\xb) \xb = P_k \xb$ for \emph{all} $\xb \in \Nbb_k$, which implies that $S(\xb)$ and $P_k$ can only have different elements for indices where elements of $\xb$ are equal.

To prove that $\hat{\fb}(\xb) \coloneqq S(\fb(\xb)) \fb(\xb)$ is continuous, we leverage the fact that $\hat{\fb}(\xb) = S(\xb) \fb(\xb)$ for intra order-preserving $\fb$. 
We will first show that $\hat{\fb}(\xb) = P_k \fb(\xb)$ for $\xb \in \Nbb_k$ and any $k\in[K]$, which implies $\hat{\fb}$ is continuous on $\Nbb_k$ when $\fb$ is continuous. 
To see this, consider an arbitrary $k\in[K]$. For $\xb \in int(\Nbb_k)$ in the interior, we have $S(\xb) = P_k$ and therefore $\hat{\fb}(\xb) = P_k \fb(\xb)$.
For $\xb \in \partial \Nbb_k$ on the boundary,  we have 
\begin{equation*}
    \hat{\fb}(\xb) = S(\xb) \fb(\xb) = P_k \fb(\xb).
\end{equation*}
The last equality holds because the difference between $S(\xb)$ and $P_k$ are only in the indices for which elements of $\xb$ are equal, and the order-preserving $\fb$ preserves exactly the same equalities. Thus, the differences between permutations $S(\xb)$ and $P_k$ do not reflect in the products $S(\xb) \fb(\xb)$ and $P_k \fb(\xb)$.

Next, we show that $\hat{\fb}(\xb) = P_k\fb(\xb) = P_{k'}\fb(\xb)$ for $\xb \in \partial \Nbb_k \cap \partial \Nbb_{k'}$. %
While $P_k \neq P_{k'}$, the intersection $\partial \Nbb_k \cap \partial \Nbb_{k'}$ contains exactly points $\xb$ such that 
the index differences in $P_k$ and $P_{k'}$ correspond to same value in $\xb$. Because $\fb$ is order-preserving, by an argument similar to the previous step, we have  $P_k\fb(\xb) = P_{k'}\fb(\xb)$ for $\xb \in \partial \Nbb_k \cap \partial \Nbb_{k'}$.

Together these two steps and the fact that $\{\Nbb_k\}$ is covering set on $\R^n$ show that $\hat{\fb}$ is a piece-wise continuous function on $\R^n$ when $\fb$ is continuous on $\R^n$.
\end{proof}

\begin{proof}[Proof of \cref{st:sorting_inv_continuous}]
In order to show the continuity of $\fb(\xb)$, we use a similar argument as in \cref{st:sorting_continuous} (see therein for notation definitions). 
For any $k\in[K]$, it is also trivial to show that $\fb$ is continuous over the open set $int(\Nbb_k)$ since $\fb(\xb) = P_k^{-1} \vb(\xb)$. We use the same argument as \cref{st:sorting_continuous} to show it is also a continuous for any point $\xb \in \partial(\Nbb_k)$ 
\begin{equation*}
    \fb(\xb) = S(\xb)^{-1} \vb(\xb) = P_k^{-1} \vb(\xb).
\end{equation*}
The last equality holds because $P_k^{-1}$ and $S(\xb)^{-1}$ can only have different elements among elements of $\yb = S(\xb) \xb$ with equal values, and $\vb$ preserves exactly these equalities in $\yb$. 
Finally, the proof can be completed by piecing the results of different $\Nbb_k$ together.

\end{proof}

\subsection{Proof of \cref{st:intra_and_order_invariant}, Order-invariant Functions}

\stIntraAndOrderInvariant*

To prove \cref{st:intra_and_order_invariant}, we first study the properties of order invariant functions in \cref{sec:order invariant functions}. We will provide necessary and sufficient conditions to describe order invariant functions, like what we did in \cref{st:order_preserv} for intra order-preserving functions. 
Finally, we combine these insights and \cref{st:order_preserv} to prove \cref{st:intra_and_order_invariant} in \cref{sec:main proof of intra_and_order_invariant}.

\subsubsection{Properties of Order Invariant Functions} \label{sec:order invariant functions}

The goal of this section is to prove the below theorem, which characterizes the representation of order invariant functions using the concept of equality-preserving.
\begin{definition}
\label{def:equality_preserving}
We say a function $\fb: \R^n \to \R^n$ is \emph{equality-preserving}, if $\fb_i(\xb) = \fb_j(\xb)$ for all $\xb\in\R^n$ such that $\xb_i = \xb_j$ for some $i,j\in[n]$
\end{definition}
\begin{theorem}
\label{st:order_invariant}
A function $\fb: \R^n \to \R^n$ is order-invariant, if and only if $\fb(\xb) = S(\xb)^{-1}\bar{\fb}(S(\xb)\xb)$ for some function $\bar{\fb}:\R^n\to\R^n$ that is equality-preserving on the domain $\{\yb :\yb = S(\xb) \xb, \text{ for } \xb\in\R^n)$.
\end{theorem}

\cref{st:order_invariant} shows an order  invariant function can be expressed in terms of some equality-preserving  function. In fact, every order invariant function is equality-preserving.
\begin{proposition}
\label{st:order_invariant_props}
Any order-invariant function $\fb: \R^n \to \R^n$ is equality-preserving.
\end{proposition}
\begin{proof}
Let $P_{ij} \in \Pbb^n$ denote the permutation matrix that only swaps $i^{th}$ and $j^{th}$ elements of the input vector; i.e. $\yb = P_{ij} \xb \Rightarrow \yb_i = \xb_j, \yb_j = \xb_i, \yb_k = \xb_k$, $\forall \xb\in\R^n$, $i,j,k\in[n]$, and $k \neq i,j$. Thus, for an order-invariant function $\fb: \R^n \to \R^n$ and any $\xb\in\R^n$ such that $\xb_i = \xb_j$, we have 
\begin{align*}
 \fb(P_{ij}\xb) = P_{ij}\fb(\xb) \Rightarrow \fb_i(P_{ij}\xb) = \fb_j(\xb) \Rightarrow \fb_i(\xb) = \fb_j(\xb) \quad\text{($\because P_{ij}\xb = \xb$ for $\xb$ such that $\xb_i = \xb_j$).}  
\end{align*}
\end{proof}

We are almost ready to prove \cref{st:order_invariant}. We just need one more technical lemma, whose proof is deferred to the end of this section.
\begin{lemma} \label{st:order_permutation}
For any $P\in\Pbb^n$ and an equality-preserving $\fb: \R^n \to \R^n$, $S(\xb)\fb(\xb)= S(P\xb)P\fb(\xb)$.
\end{lemma}

\begin{proof}[Proof of \cref{{st:order_invariant}}]
($\rightarrow$) For an order-invariant function $\fb: \R^n \to \R^n$, we have $\fb(P\xb) = P\fb(\xb)$ by \cref{def:order_invariant} for any $P\in\Pbb^n$. Take $P=S(\xb)$. We then have the equality $\fb(\xb) = S(\xb)^{-1}\fb(S(\xb)\xb)$. This is an admissible representation because, by \cref{st:order_invariant_props}, $\fb$ is equality-preserving. \\ \\
($\leftarrow$) Let $\fb(\xb) = S(\xb)^{-1}\bar{\fb}(S(\xb)\xb)$ for some equality-preserving function $\bar{\fb}$. First, because $\bar{\fb}$ is equality preserving and  $\fb$ is constructed through the sorting function $S$, we notice that $\fb(\xb)$ is equality-preserving. Next, we show $\fb$ is also order invariant:
\begin{align*}
    \fb(P\xb) &= S(P\xb)^{-1}\bar{\fb}(S(P\xb)P\xb) \\
    &= S(P\xb)^{-1} \bar{\fb}(S(\xb)\xb) &\text{($\because S(P\xb)P\xb = S(\xb)\xb$ by choosing $\fb(\xb) = \xb$ in \cref{st:order_permutation})} \\
    &= S(P\xb)^{-1} S(\xb) \fb(\xb) &\quad\text{($\because$ definition of $\fb(\xb)$}) \\
    &= S(P\xb)^{-1} S(P\xb) P\fb(\xb) &\since{\cref{st:order_permutation}}\\
    &= P\fb(\xb).
\end{align*}
\end{proof}

\subsubsection{Main Proof} \label{sec:main proof of intra_and_order_invariant}

\begin{proof}[Proof of \cref{st:intra_and_order_invariant}]
($\rightarrow$) From \cref{st:order_preserv} we can write $\fb(\xb) = S(\xb)^{-1}U\wb(\xb)$. On the other hand, from \cref{st:order_invariant} we can write $\fb(\xb) = S(\xb)^{-1} \bar{\fb}(\yb)$ for some equality-preserving function $\bar{\fb}$. Using both we can identify $\wb(\xb) = U^{-1}\bar{\fb}(\yb)$ which implies that $\wb$ is only a function of the sorted input $\yb$ and can be equivalently written as $\wb(\yb)$. \\ \\
($\leftarrow$) For $\wb$ with the properties in the theorem statement, the function $\fb(x) = S(x)^{-1}U\wb(\yb)$ satisfies the conditions of \cref{st:order_preserv}; therefore $\fb$ is intra order-preserving. 
To show $\fb$ is also order-invariant, we write $\fb(\xb) = S(x)^{-1} \bar{\fb}(\yb)$ where $\bar{\fb}(\yb) = U\wb(\yb)$. Because $\bar{\fb}_i(\yb) = \sum_{j=0}^{n-i} \wb_{n-j}(\xb)$, we can derive with the definition of $\wb$ that
\begin{equation*}
    \yb_i = \yb_{i+1} \Rightarrow \wb_i(\yb) = 0 \Rightarrow \bar{\fb}_i(\xb) = \bar{\fb}_{i+1}(\xb).
\end{equation*}
That is, $\bar{\fb}(\yb)$ is equality-preserving on the domain of sorted inputs. Thus, $\fb$ is also order-invariant.
\end{proof}

\subsubsection{Deferred Proof of Lemmas}
\begin{proof}[Proof of \cref{st:order_permutation}]
To prove the statement, we first notice a fact that $S(\xb)= S(P\xb)P$, for any $P\in\Pbb^n$ and $\xb \in \Xbb \coloneqq \{\xb\in\R^n :  \xb_i\neq \xb_j, \forall i,j\in[n], i\neq j\}$. 
Therefore, for $\xb\in \Xbb$, we have $S(\xb) \fb(\xb)= S(P\xb)P\fb(\xb)$.

Otherwise, consider some $\xb\in\R^n \setminus \Xbb$. Without loss of generality\footnote{This choice is only for convenience of writing the indices.}, we may consider $n>2$ and $\xb$ such that $\xb_1=\xb_2>\xb_k$ for all $k> 2$; because $\fb$ is equality-preserving, we have $\fb_1(\xb) = \fb_2(\xb)$. 

To prove the desired equality, we will introduce some extra notations. We use subscript $_{i:j}$ to extract contiguous parts of a vector, e.g. $\xb_{2:n} = [\xb_2, \dots, \xb_n]$ and $\fb_{2:n}(\xb) = [\fb_2(\xb), \dots, \fb_n\xb)]$ (by our construction of $\xb$, $\xb_{2:n}$ is a vector where each element is unique.)
In addition, without loss of generality, suppose $P\in\Pbb^n$ shifts index $1$ to some index $i\in[n]$; we define $\bar{P}\in\{0,1\}^{n-1\times n-1}$ by removing the $1$st column and the $i$th row of $P$ (which is also a permutation matrix). 
Using this notion, we can partition $S(\bar{P}\xb_{2:n}) \in \{0,1\}^{n-1\times n-1}$ as 
\begin{align*}
    S(\bar{P}\xb_{2:n}) =
\begin{bmatrix}
    B_1 & B_2\\
    B_3 & B_4
        \end{bmatrix}        
\end{align*}
where $B_1\in \R^{1\times i-1}$, $B_2\in \R^{1\times n-i}$, $B_3\in \R^{n-2\times i-1}$, and  $B_4\in \R^{n-2\times n-i}$. This would imply that $S(P\xb) \in \{0,1\}^{n\times n}$ can be written as one of followings 
\begin{align} \label{eq:sorting examples}
    \begin{bmatrix}
    & e_i^\t  & \\
    B_1 & 0 & B_2\\
    B_3 & 0 & B_4
    \end{bmatrix}    
    \quad\text{or}\quad
    \begin{bmatrix}
    B_1 & 0 & B_2\\
    & e_i^\t  & \\
    B_3 & 0 & B_4
    \end{bmatrix}
\end{align}
where $e_i$ is the $i$th canonical basis.

To prove the statement, let $\yb = P \fb(\xb)$. By the definition of $\bar{P}$, we can also write $\yb$ as 
\begin{align} \label{eq:another way to write y}
    \yb = 
    \begin{bmatrix}
    \yb_{1:i-1}\\
    \yb_i\\
    \yb_{i+1:n}
    \end{bmatrix} = 
    \begin{bmatrix}
        (\bar{P}\fb_{2:n}(\xb))_{1:i-1}\\
        \fb_1(\xb) \\
        (\bar{P}\fb_{2:n}(\xb))_{i:n-1}
    \end{bmatrix}
\end{align}
Let us consider the first case in \eqref{eq:sorting examples}. We have
\begin{align*}
    S(P\xb)P\fb(\xb) 
    =\begin{bmatrix}
        \yb_i \\ 
        B_1 y_{1:i-1} +  B_2 y_{i+1:n}\\
        B_3 y_{1:i-1} +  B_4 y_{i+1:n}
    \end{bmatrix}
    =\begin{bmatrix}
        \yb_i \\ 
        S(\bar{P}\xb_{2:n}) \bar{P} \fb_{2:n}(\xb)
    \end{bmatrix}
    =
      \begin{bmatrix}
    \fb_1(\xb) \\
    S(\xb_{2:n}) \fb_{2:n}(\xb)
    \end{bmatrix}
    = S(\xb)\fb(\xb)
\end{align*}
where the second equality follows from \eqref{eq:another way to write y}, the third from the fact we proved at the beginning for the set $\Xbb$, and the last equality is due to the assumption $\xb_1=\xb_2>\xb_k$ and the equality-preserving property that $\fb_1(\xb) = \fb_2(\xb)$.
For the second case in \eqref{eq:sorting examples}, based on the same reasoning above, we can show 
\begin{align*}
    S(P\xb)P\fb(\xb) 
    =
      \begin{bmatrix}
     (S(\xb_{2:n}) \fb_{2:n}(\xb))_1\\
    \fb_1(\xb) \\
    (S(\xb_{2:n}) \fb_{2:n}(\xb))_{2:n-1}
    \end{bmatrix}, 
\end{align*}
Because $\xb_1=\xb_2$, we have $(S(\xb_{2:n}) \fb_{2:n}(\xb))_1 = \fb_1(\xb) = \fb_2(\xb)$. Thus, $S(P\xb) P\fb(\xb)=S(\xb)\xb$.
\end{proof}

\subsection{Proof of \cref{st:ful_order_preserving_diagonal}, Diagonal Functions}\label{app:diagonal}

\stFulOrderPreservingDiagonal*

We first prove some properties of \emph{diagonal} intra order-preserving functions, which will be used to prove \cref{st:ful_order_preserving_diagonal}.
\begin{proposition}
\label{st:order_preserving_props}
Any intra order-preserving function $\fb: \R^n \to \R^n$ is equality-preserving.
\end{proposition}
\begin{proof}
This can be seen directly from the definition of intra order-preserving functions. 
\end{proof}

\begin{corollary}\label{st:diagonal_equality_preserving}
The following statements are equivalent
\begin{enumerate}
\item A function $\fb: \R^n \to \R^n$ is diagonal and equality-preserving.
\item $\fb(\xb) = [\bar{f}(\xb_1), \dots, \bar{f}(\xb_n)]$ for some $\bar{f}:\R\to\R$.
\item A function $\fb: \R^n \to \R^n$ is diagonal and order-invariant.
\end{enumerate}
\end{corollary}
\begin{proof}
($1\rightarrow 2$) Let $\fb(\xb) = [f_1(\xb_1), \dots, f_n(\xb_n)]$ be a diagonal and equality-preserving function. One can conclude that $\fb_1(x) = \dots = \fb_n(x)$ for all $x \in \R$.\\ \\
($2\rightarrow 3$) Let $\ub= P\xb$ for some permutation matrix $P\in\Pbb^n$. Then $\fb(P\xb) = [\bar{f}(\ub_1), \dots, \bar{f}(\ub_n)] = P [\bar{f}(\xb_1), \dots, \bar{f}(\xb_n)] = P \fb(\xb)$. \\ \\
($3\rightarrow 1$) True by \cref{st:order_invariant_props}.
\end{proof}

\begin{proof}[Proof of \cref{st:ful_order_preserving_diagonal}]
$(\rightarrow)$ By \cref{st:order_preserving_props}, an intra order-preserving function $\fb$ is also equality-preserving. Therefore, by \cref{st:diagonal_equality_preserving} it can be represented in the form $\fb(\xb) = [\bar{f}(\xb_1), \dots, \bar{f}(\xb_n)]$ for some $\bar{f}:\R\to\R$. 
Furthermore, because $\fb(\xb)$ is intra order-preserving, for any $\xb \in \R^n$ with $\xb_1 > \xb_2$, it satisfies $\fb_1(\xb_1) > \fb_2(\xb_2)$; that is, $\bar{f}(\xb_1) > \bar{f}(\xb_2)$. Therefore, $\bar{f}$ is an increasing function. Continuity is inherited naturally.
\\ \\
$(\leftarrow)$ Because $\fb_i(\xb) = \bar{f}(\xb_i)$ and $\bar{f}$ is an increasing function, it follows that $\fb$ is intra order-preserving 
\begin{equation*}
\begin{aligned}
    \xb_i = \xb_j \Rightarrow \fb_i(\xb) = \fb_j(\xb) \qquad \text{ and} \qquad
    \xb_i > \xb_j \Rightarrow \fb_i(\xb) > \fb_j(\xb).
\end{aligned}
\end{equation*}
\end{proof}

Finally, we prove that diagonal intra order-preserving functions are also order-invariant. This fact was mentioned in the paper without a proof.
\begin{corollary}
\label{st:diagonal_order_preserving}
A diagonal intra order-preserving function is also order-invariant.
\end{corollary}
\begin{proof}
Intra order-preserving functions are equality-preserving by \cref{st:order_preserving_props}. By \cref{st:diagonal_equality_preserving} an diagonal equality-preserving function is order-invariant.
\end{proof}
\section{Continuity and Differentiability of the Proposed Architecture}
In this section, we discuss properties of the function $\fb(\xb) = S(\xb)^{-1} U D(\yb) \mb(\xb)$. In order to learn the parameters of $\mb$ with a first order optimization algorithm, it is important for $\fb$ to be differentiable with respect to the parameters of $\mb$. This condition holds in general, since the only potential sources of non-differentiable $\fb$, $S(\xb)^{-1}$ and $\yb$ are constant with respect to the parameters of $\mb$. Thus, if $\mb$ is differentiable with respect to its parameters, $\fb$ is also differentiable with respect to the parameters of $\mb$. 

Next, we discuss continuity and differentiability of $\fb(\xb)$ with respect the \emph{input} $\xb$. These properties are important when the input to function $f$ is first processed by a trainable function $\gb$ (i.e. the final output is computed as $\fb\circ\gb(\xb)$). This is not the case in post-hoc calibration considered in the paper, since the classifier $\gb$ here is not being trained in the calibration phase.

We show below that when $\wb(\xb) = D(\yb) \mb(\xb)$ satisfies the requirements in \cref{st:order_preserv}, the function $\fb(\xb) = S(\xb)^{-1} U D(\yb) \mb(\xb)$ is a continuous intra order-preserving function. 

\begin{corollary}\label{st:w_dm_conds}
Let $\sigma: \R \to \R$ be a continuous function where $\sigma(0) = 0$ and strictly positive on $\R\setminus\{0\}$, and let $\mb$ be a continuous function where $\mb_i(\xb) > 0$ for $i<n$,
and arbitrary for $\mb_d(\yb)$. Let $D(\yb)$ denote a diagonal matrix with entries $D_{ii} = \sigma(\yb_i - \yb_{i+1})$ for $i<n$ and $D_{nn} = 1$. 
Then $\wb(\xb) = D(\yb) \mb(\xb)$ is a continuous function and satisfies the following conditions 
\begin{itemize}
    \item $\wb_i(\xb) = 0$, for $i<n$ and $\yb_i = \yb_{i+1}$
    \item $\wb_i(\xb) > 0$, for $i<n$ and $\yb_i > \yb_{i+1}$
    \item $\wb_n(\xb)$ is arbitrary,
\end{itemize}
where $\yb = S(\xb) \xb$ is the sorted version of $\xb$.
\end{corollary}
\begin{proof}
First, because $\yb = S(\xb) \xb$ is a continuous function (by \cref{st:sorting_continuous} with $\fb(\xb) = \xb$), $\wb(\xb)=D(\yb) \mb(\xb)$ is also a continuous function. 
Second, because $\norm{\xb}<\infty$, we have $\mb(\xb)<\infty$ due to continuity. Therefore, it follows that $\wb_i(\xb) = \sigma(\yb_i - \yb_{i+1}) \mb_i(\xb)$ satisfies all the listed conditions.
\end{proof}

To understand the differentiability of $\fb$, we first see that $\fb$ may not be differentiable at a point where there is a tie among some elements of the input vector.
\begin{corollary}
For $\wb$ in \cref{st:w_dm_conds}, there exists differentiable functions $\mb$ and $\sigma$ such that $\fb(\xb) = S(\xb)^{-1} U \wb(\xb)$ is not differentiable globally on $\R^n$.
\end{corollary}
\begin{proof}
For the counter example, let $\mb:\R^3\to\R^3$ be a constant function $\mb(\xb) =[1,1,1]^\top$, and $\sigma(a) = a^2$. It is easy to verify that they both satisfy the conditions in \cref{st:w_dm_conds} and are differentiable. We show that the partial derivative $\frac{\partial \fb_1(\xb)}{\partial \xb_3}$ does not exists at $\xb = [2,1,1]^\t$.
With few simple steps one could see $\fb_1(\xb + \alpha \eb_3)$ for $\alpha \in (-\infty, 1]$ is
\begin{equation}
    \fb_1(\xb + \alpha \eb_3) = 
    \begin{cases} 
      \sigma(1) + \sigma(-\alpha) + 1 & \alpha \leq 0 \\
      \sigma(1-\alpha) + \sigma(\alpha) + 1 & 0 < \alpha \leq 1
    \end{cases}
\end{equation}
Though this function is continuous, 
the left and right derivatives are not equal at $\alpha = 0$ so the function is not differentiable at $\xb = [2,1,1]^\t$. 
\end{proof}
The above example shows that $\fb$ may not be differentiable for tied inputs. On the other hand, it is straightforward to see function $\fb$ is differentiable at points where there is no tie. 
More precisely, for the points with tie in the input vector, we show the function $\fb$ is B-differentiable, which is a weaker condition than the usual (Frech\'{e}t) differentiability.
\begin{definition}\citeRef{facchinei2007finite}
A function $\fb:\R^n \to \R^m$  is said to be \emph{B(ouligand)-differentiable} at a point $\xb \in \R^n$, if $\fb$ is Lipschitz continuous in the neighborhood of $\xb$ 
 and directionally differentiable at $\xb$.
\end{definition}

\begin{proposition}\label{st:differentiable_order_preserving}
For $\fb:\R^n \to \R^n$ in \cref{st:order_preserv}, let $\wb(\xb)$ be as defined in \cref{st:w_dm_conds}. 
If $\sigma$ and $\mb$ are continuously differentiable, then $\fb$ is B-differentiable on $\R^n$.
\end{proposition}
\begin{proof}
Let $\Pbb^n = \{P_1, \dots, P_K\}$ be the finite set of all possible $n\times n$ dimensional permutation matrices. For each $k\in[K]$, define the closed set $\Nbb_k = \{\xb : S(\xb) \xb = P_k \xb \}$. These sets are convex polyhedrons since each can be defined by a finite set of linear inequalities; in addition, they together form a covering set of $\R^n$. 

If there is no tie in elements of vector $\xb$, then $\xb \in int(\Nbb_k)$ for some $k\in [K]$. Since the sorting function $S(\xb)$  has the constant value $P_k$ in a small enough neighborhood of $\xb$, the function $\fb$ is continuously differentiable (and therefore B-differentiable) at $\xb$. 
%

Next we show that, for any point $\xb \in \R^n$ with some tied elements, the directional derivative of $\fb$ along an arbitrary direction $\db\in\R^n$ exists. For such $\xb$ and $\db$, 
there exists a $k \in [K]$ and a small enough $\delta > 0$ such that $\xb, \xb+\epsilon \db \in \Nbb_k$ for all $0 \leq \epsilon \leq \delta$. Therefore, we have $\fb(\xb') = \hat{\fb}(\xb')$ for all $\xb' \in [\xb, \xb+\delta \db]$, where $\hat{\fb}_k(\xb) = P_k^{-1} U D(P_k \xb) \mb(\xb)$. Let $\hat{\fb}_k'(\xb;\db)$ denote the directional derivative of $\hat{\fb}_k$ at $\xb$ along $\db$. By the equality of $\hat{\fb}_k$ and $\fb$ in $[\xb, \xb+\delta \db]$, we conclude that the directional derivative $\fb'(\xb;\db)$ exists and is equal to $\hat{\fb}'(\xb;\db)$.

Finally, we note that $\fb$ is  Lipschitz continuous, since it is 
composed by pieces of Lipschitz continuous functions $\hat{\fb}_k$ for $k\in[K]$ (implied by the continuous differentiability assumption on $\sigma$ and $\mb$).
Thus, $\fb$ is B-differentiable.
\end{proof}



\section{Learning Increasing Functions}
We follow the implementation of~\citeRef{wehenkel2019unconstrained_ref} for learning increasing functions in the diagonal subfamily. The idea is to learn an increasing function $\bar{f}(x): \R \rightarrow \R$ using a neural network, which can be realized by learning a strictly positive function $\bar{f}'(x)$ and a bias $\bar{f}(0) \in \R$ and constructing the desired function $\bar{f}$ by the integral 
$
    \bar{f}(x) = \int_0^x \bar{f}'(t)dt + \bar{f}(0)
$.
In implementation, the derivative $\bar{f}'$ is modeled by a generic neural network and the positiveness is enforced by using a proper activation function in the last layer. 
In the forward computation, the integral is approximated numerically using Clenshaw-Curtis quadrature~\citeRef{clenshaw1960method_ref} and the backward pass is performed by Leibniz integral rule to reduce memory footprint. We the use official implementation of the algorithm provided by~\citeRef{wehenkel2019unconstrained_ref}.

\section{Datasets, Hyperparameters, and Architecture Selection}
The size of the calibration and the test datasets, as well as the number of classes for each dataset, are shown in~\cref{tbl:stats}. We note that the calibration sets sizes are the same as the previous methods~\citeRef{guo2017calibration2,kull2019beyond2}.
\begin{table}[!b]
\begin{center}
\caption{\label{tbl:stats}\small{Statistics of the Evaluation Datasets.}}
\resizebox{0.7\linewidth}{!}{%
\begin{tabular}{c|c c c c}
\toprule
Dataset  & \#classes & Calibration set size & Test dataset size \\
\hline
CIFAR-10 & 10 & 5000 & 10000 \\
SVHN & 10 & 6000 & 26032 \\
CIFAR-100 & 100 & 5000 & 10000 \\
CARS & 196 & 4020 & 4020  \\
BIRDS & 200 & 2897 & 2897 \\
ImageNet & 1000 & 25000 & 25000 \\
\bottomrule
\end{tabular}
}
\end{center}
\end{table}

We follow the experiment protocol in~\citeRef{kull2019beyond2} and use cross validation on the calibration set to find the best hyperparameters and architectures for all the methods. We found that \citeRef{kull2019beyond2} have improved their performance via averaging output predictions of models trained on different folds. We follow the same approach to have fair comparisons.
Our criteria for selecting the best architecture is the NLL value. We perform $3$ fold cross validation for ImageNet and $5$ folds for all the other datasets. We limit our architecture to fully connected networks and vary the number of hidden layers as well as the size of each layer. We allow networks with up to $3$ hidden layers in all the experiments. In CIFAR-10, SVHN, and CIFAR-100 with fewer classes, we test networks with $\{1,2,10,20,50,100,150\}$ units per layer and for the larger CARS, BIRDS, and ImageNet datasets, we allow a wider range of $\{2,10,20,50,100,150,500\}$ units per layer. We use the similar number of units for all the hidden layers to reduce the search space. 
We use ReLU activation for all middle hidden layers and Softplus on the last layer when strict positivity is desired. 
We utilize L-BFGS~\citeRef{liu1989limited} for small scale optimization problems when the computational resources allow (temperature scaling and diagonal intra order-preserving~(\dgfn) methods on CIFAR and SVHN datasets) and use Adam~\citeRef{kingma2014adam} optimizer for other experiments. \cref{tbl:hyperparams} summarize cross validation learned hyperparameter for each method.

Although the functions learned in \cref{tbl:hyperparams} are more complicated than linear transformations used in the baselines, they are not too complex to slow down computation as the calibration network size is negligible compared to the backbone network used in the experiments. In our experiments, all methods take less than $0.5$ milliseconds/sample in forward path and their differences are negligible. 

We use the pre-computed logits of these networks provided by~\citeRef{kull2019beyond2} for CIFAR, SVHN, and ImageNet with DenseNet and ResNet~\footnote{\href{https://github.com/markus93/NN_Calibration}{https://github.com/markus93/NN\_Calibration}}. 
In addition, we use the publicly available state-of-the-art models for PNASNet5-large and ResNet50~NTSNet~\citeRef{Yang2018Learning}~\footnote{\href{https://github.com/osmr/imgclsmob/blob/master/pytorch/README.md}{https://github.com/osmr/imgclsmob/blob/master/pytorch/README.md}} for ImageNet and BIRDS datasets, respectively. Furthermore, we trained different ResNet type models on CARS dataset using the standard pytorch training script. The ResNet models with (pre) are initialized with pre-trained ImageNet weights. We will release these models for future research.

The effect of weight regularization on different metrics for \msodir and \dirodir methods is illustrated in~\cref{fig:odir_wl}. This shows that simply regularizing the off diagonal elements of a linear layer has limited expressiveness to achieve good calibration especially in the case that number of classes is large.

\begin{table*}[!t]
	\begin{center}
	    \caption{\label{tbl:hyperparams}\footnotesize{Hyperparameters learned by cross validation. For \dgfn, \oinv, \opre, and \uncons we show the network architectures learned by cross validation. The number of units in each layer are represented by a sequence of numbers, e.g. $(10,20,30,40)$ represents a network with $10$ input units, $20$ and $30$ units in the first and second hidden layers, respectively, and $40$ output units. We perform multi-fold cross-validation and select the architecture with lowest NLL on validation set.}}
		\resizebox{1.0\linewidth}{!}{%
			\begin{tabular}{c c | c | c | c|c}
				\toprule
				Dataset  & Model  & \dgfn & \oinv & \opre & \uncons \\
				\hline
				CIFAR10 & ResNet~110 & (1,10,10,1)   & (10,150,150,10) & (10,2,2,10) & (10,500,10) \\
				CIFAR10 & Wide ResNet~32    & (1,2,2,1) & (10,10,10,10) & (10,2,2,10) & (10,150,150,10) \\
				CIFAR10 & DenseNet~40  & (1,2,2,1) & (10,50,50,100 & (10,2,2,10) & (10,150,150,10) \\
				\hline
				SVHN & ResNet~152 (SD) & (1,20,20,1) & (10,10,10,10) & (10,50,50,10) & (10,500,10) \\
				\hline
				CIFAR100 & ResNet~110 & (1,10,10,1) & (100,100,100,100) & (100,150,150,100) & (100,500,100) \\
				CIFAR100 & Wide ResNet~32   & (1,1,1) & (100,100,100,100) & (100,2,2,100) & (100,500,100)  \\
				CIFAR100 & DenseNet~40 & (1,1,1) & (100,10,10,100) & (100,2,2,100) & (100,500,500,100)\\
				\hline
				CARS & ResNet~50 (pre)  & (1,50,1) & (196,10,196) & (196,2,2,196) & (196,500,196)\\
				CARS & ResNet~101 (pre)  & (1,20,20,1) & (196,100,100,196) & (196,20,20,196) & (196,500,196)\\
				CARS & ResNet~101  & (1,50,1)  & (196,50,50,196) & (196,100,100,196) & (196,500,196)\\
				\hline
				BIRDS & ResNet~50 (NTS)  & (1,50,50,1) & (200,150,150,200) & (200,50,50,200) & (200,500,200)\\
				\hline 
				ImageNet & ResNet~152  & (1,10,10,1) & (1000,150,150,1000) & (1000,2,2,1000) & (1000,150,1000)\\
				ImageNet & DenseNet~161  & (1,10,1) & (1000,100,100,1000) & (1000,2,2,1000) & (1000,150,1000)\\
				ImageNet & PNASNet5~large  & (1,20,20,1) & (1000,50,50,1000) & (1000,100,100,1000) & (1000,100,1000)\\
				\bottomrule
			\end{tabular}
		}
	\end{center}
	\vspace{-.2cm}
\end{table*}

\begin{figure}[t]
\begin{center}
\includegraphics[width=0.3\linewidth]{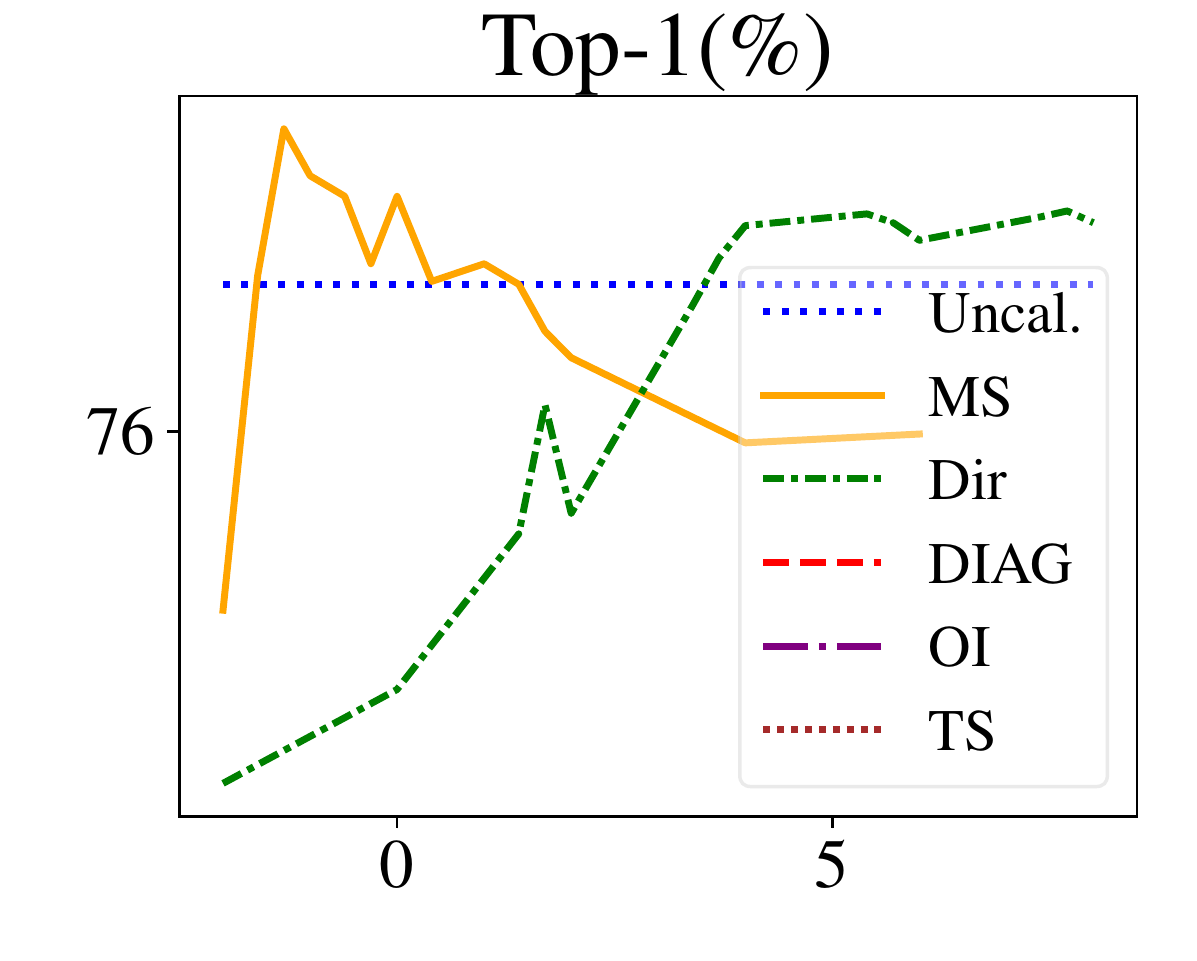}
\includegraphics[width=0.3\linewidth]{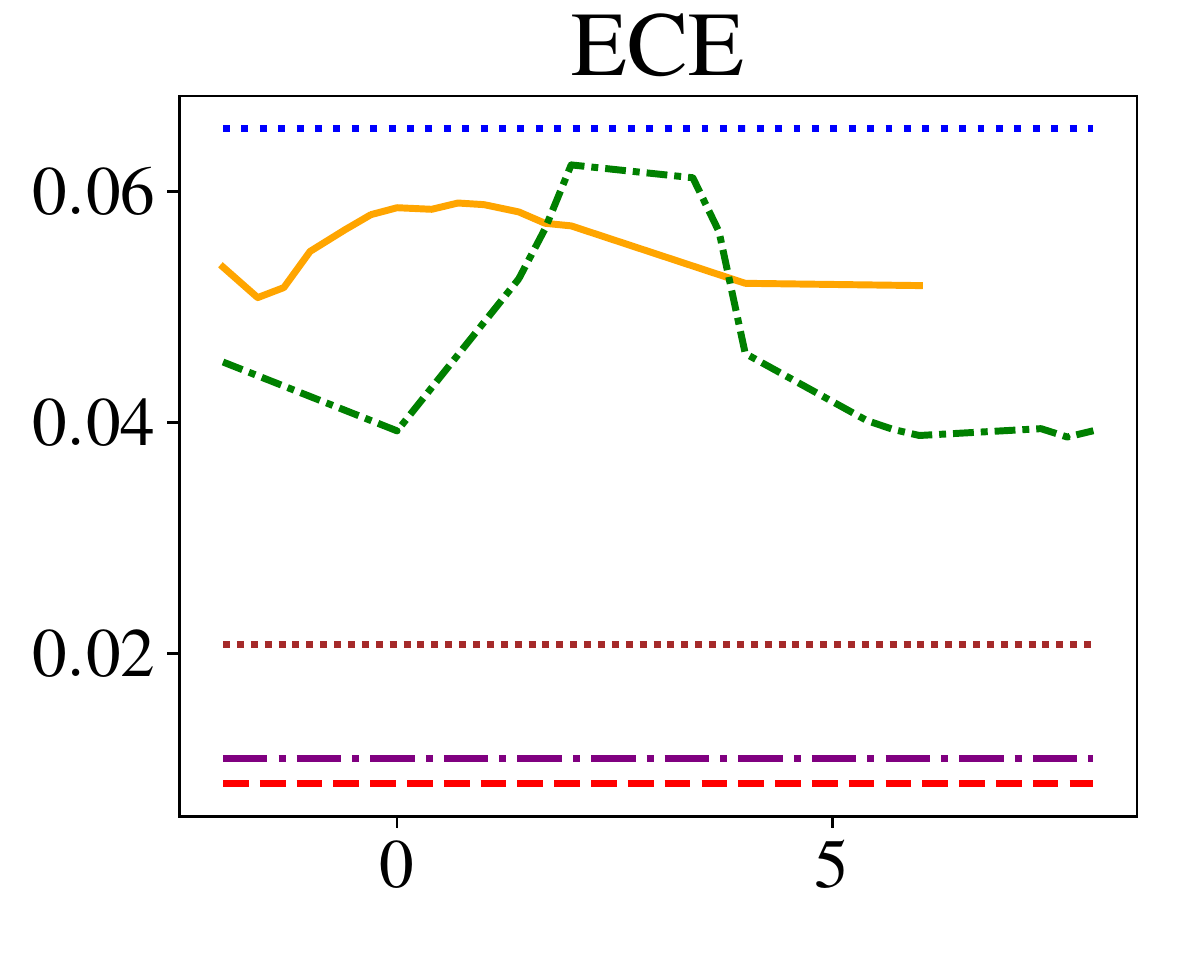}
\includegraphics[width=0.3\linewidth]{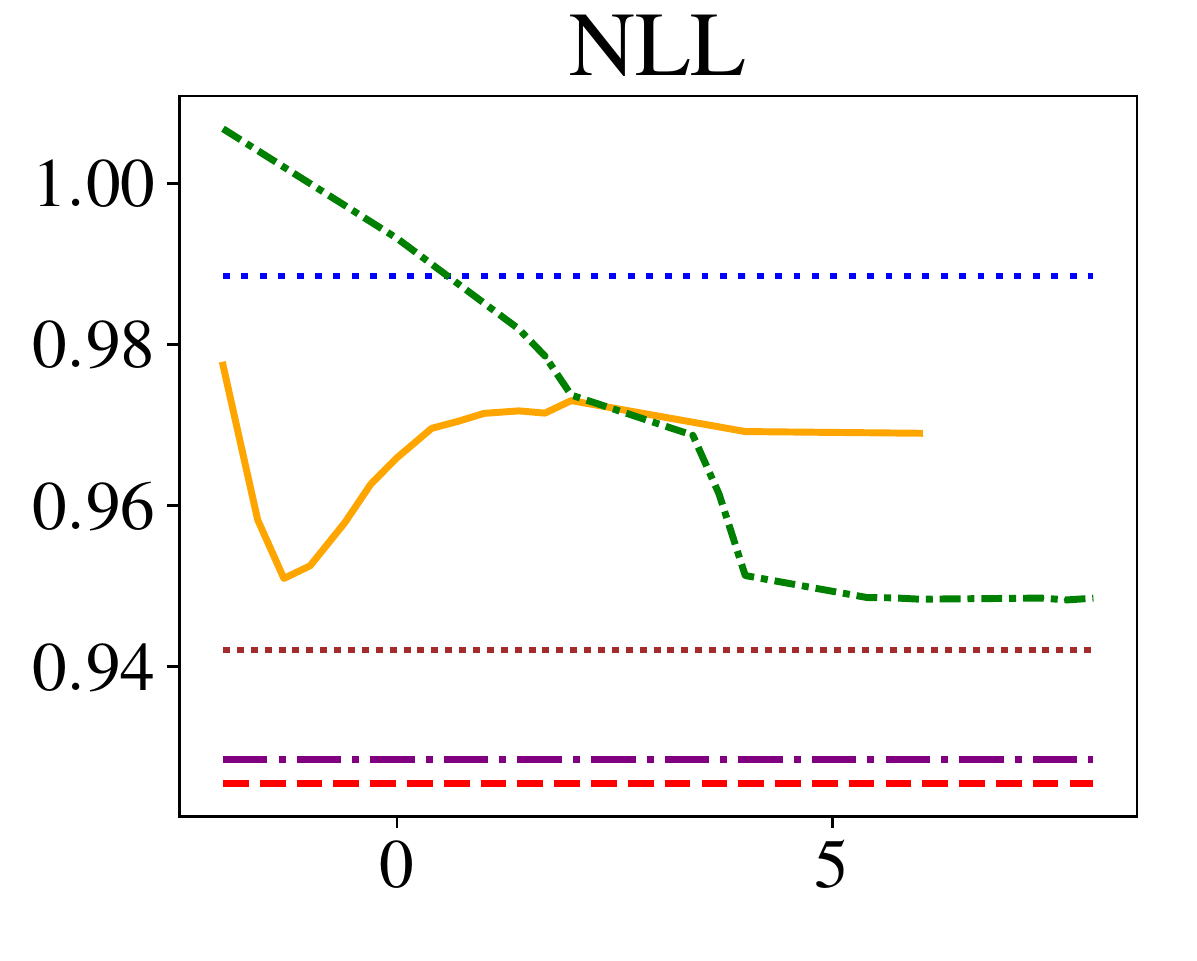}
\caption{\label{fig:odir_wl} Accuracy, ECE, and NLL plots in \msodir and \dirodir for ResNet~152 on ImageNet with different regularization weights. In the plots, x-axis shows the log scale regularization and y-axis shows the accuracy, ECE, and NLL of different methods, respectively. The value of the bias regularizer is found by cross validation and kept constant for visualization purpose. Changing the bias regularizer has little effect on the final shape of the plots. 
}
\end{center}
\end{figure}
\section{More Experiments and Discussions}
{\bf Reliability Diagrams.} In \cref{fig:transformations} of the paper, we show the reliablity diagrams and diagnoal functions leanred by TS and \dgfn in ResNet~152 and PNASNet5~large on ImageNet dataset. \cref{fig:reliability} and \ref{fig:reliability2} illustrate the reliability diagrams for different calibration algorithms in all the models. In general \dgfn method outperforms other methods in calibration in most of the regions. \opre and \oinv methods also achieve good calibration performance on this dataset and are slightly better than temperature scaling, while \msodir and \dirodir methods do not reduce the calibration error as much.

\begin{figure}
\centering
\begin{subfigure}[t]{\dimexpr0.30\textwidth+20pt\relax}
    \makebox[20pt]{\raisebox{40pt}{\rotatebox[origin=c]{90}{\tiny CIFAR10, ResNet 110}}}%
    \includegraphics[width=\dimexpr\linewidth-20pt\relax,trim=25 25 10 10, clip]
    {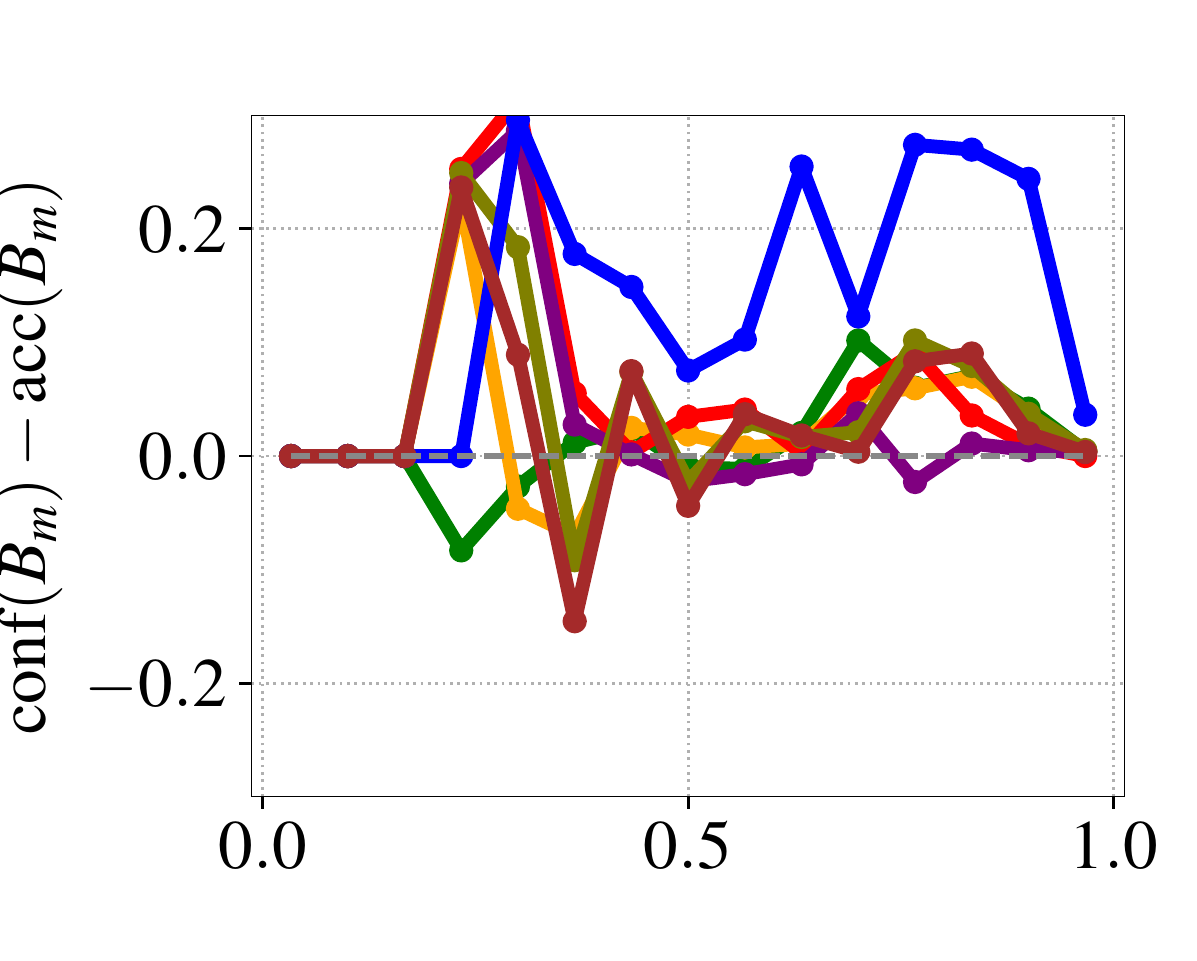}
    \makebox[20pt]{\raisebox{40pt}{\rotatebox[origin=c]{90}{\tiny CIFAR10, Wide ResNet 32}}}%
    \includegraphics[width=\dimexpr\linewidth-20pt\relax,trim=25 25 10 10, clip]
    {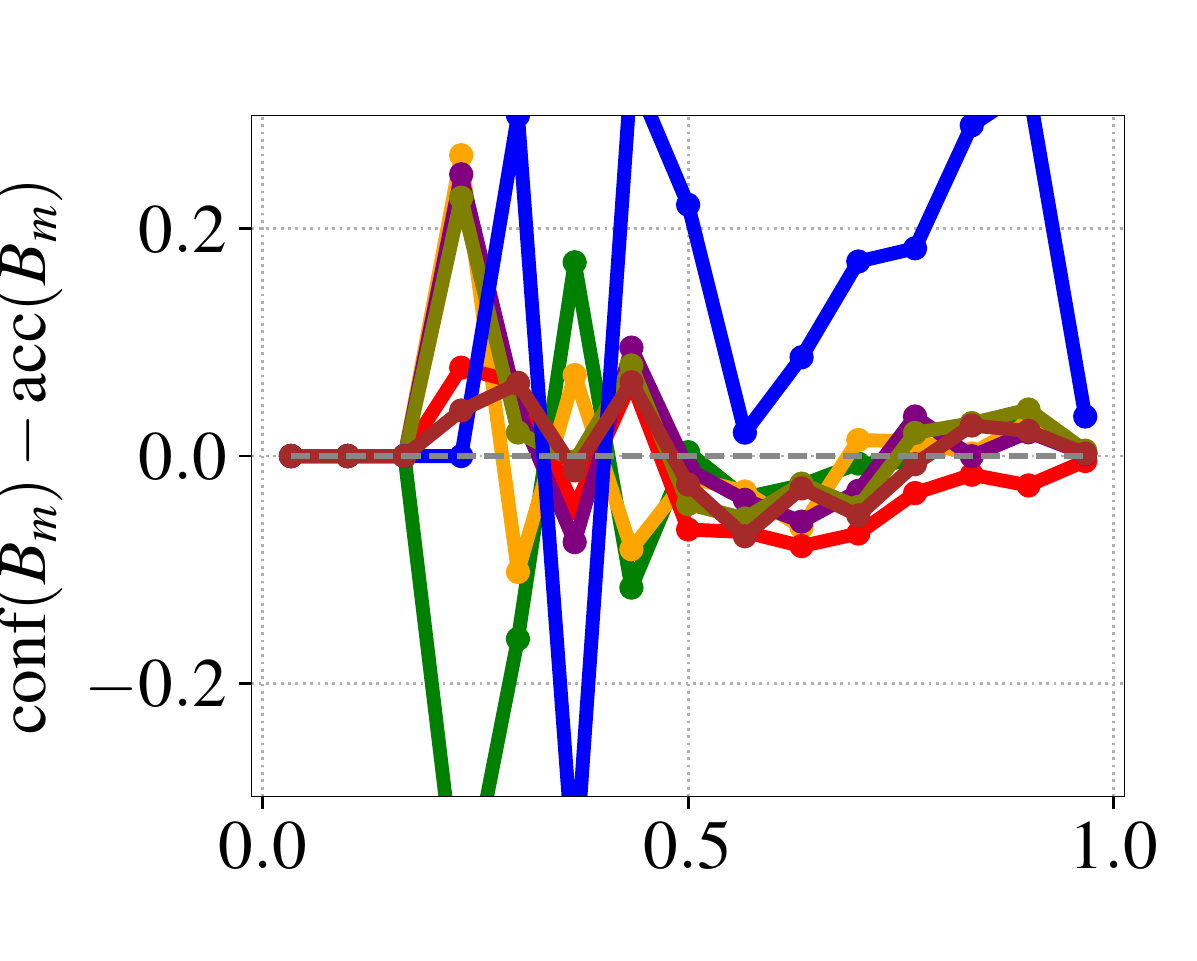}
    \makebox[20pt]{\raisebox{40pt}{\rotatebox[origin=c]{90}{\tiny CIFAR10, DenseNet 40}}}%
    \includegraphics[width=\dimexpr\linewidth-20pt\relax,trim=25 25 10 10, clip]
    {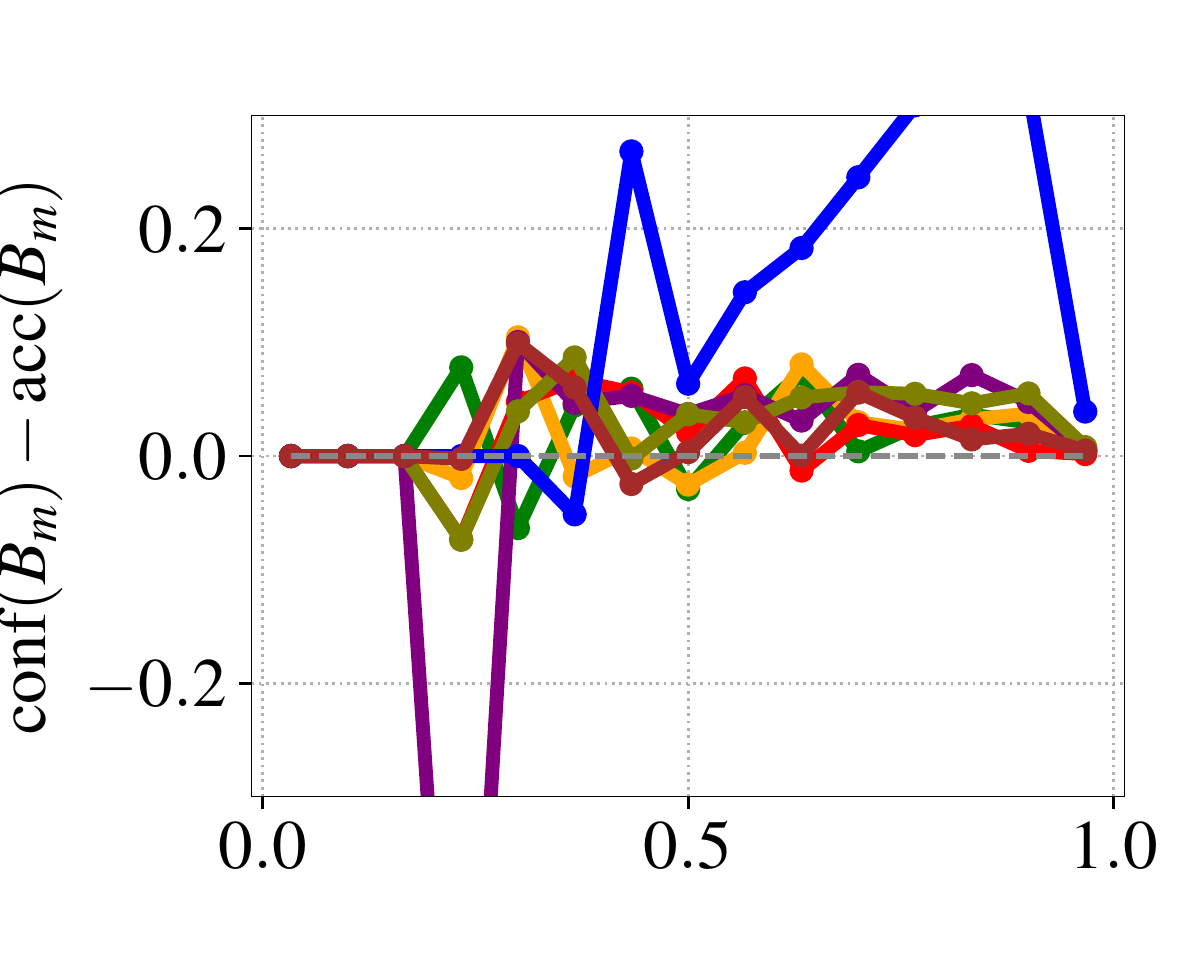}
    \makebox[20pt]{\raisebox{40pt}{\rotatebox[origin=c]{90}{\tiny CIFAR100, ResNet 110}}}%
    \includegraphics[width=\dimexpr\linewidth-20pt\relax,trim=25 25 10 10, clip]
    {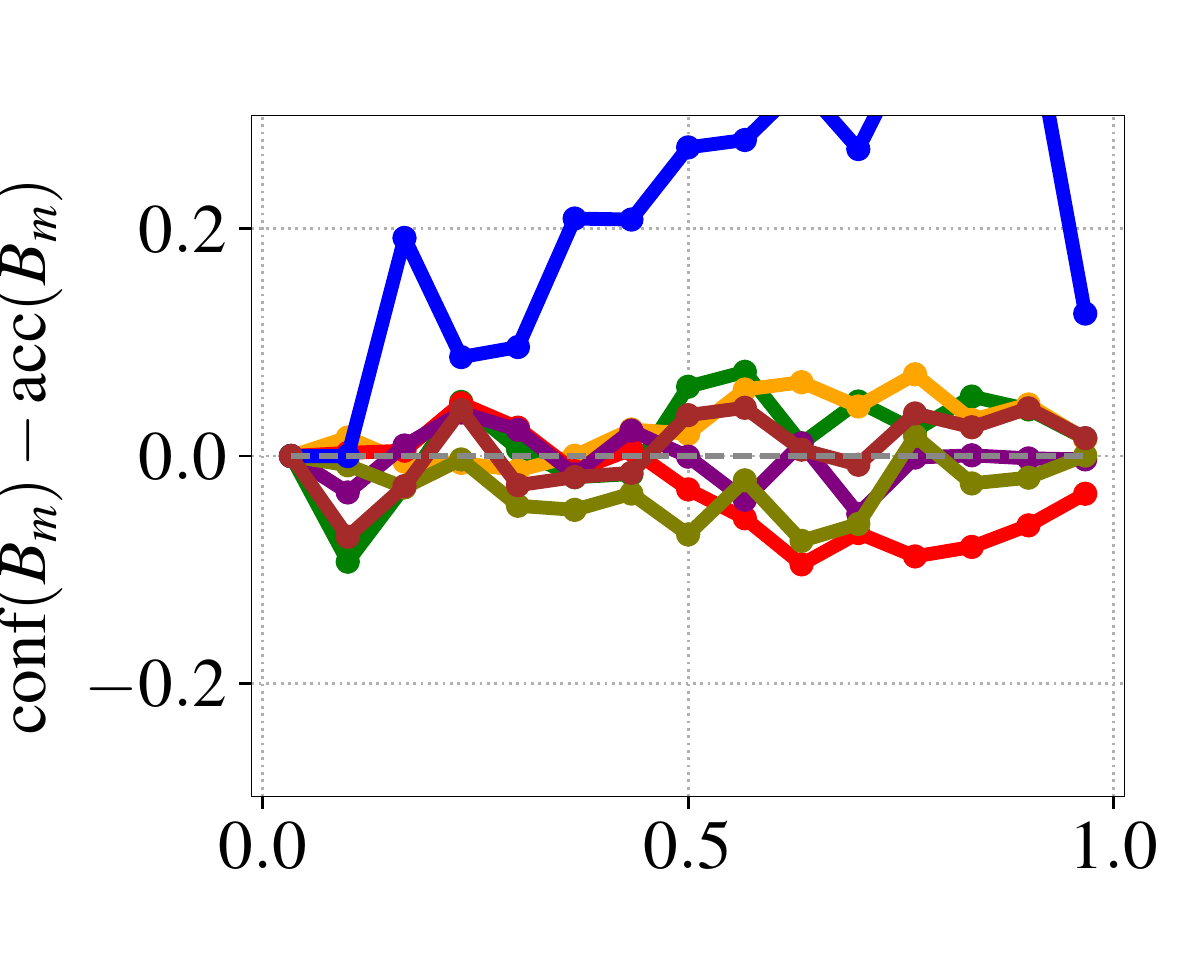}
    \makebox[20pt]{\raisebox{40pt}{\rotatebox[origin=c]{90}{\tiny CIFAR100, Wide ResNet 32}}}%
    \includegraphics[width=\dimexpr\linewidth-20pt\relax,trim=25 25 10 10, clip]
    {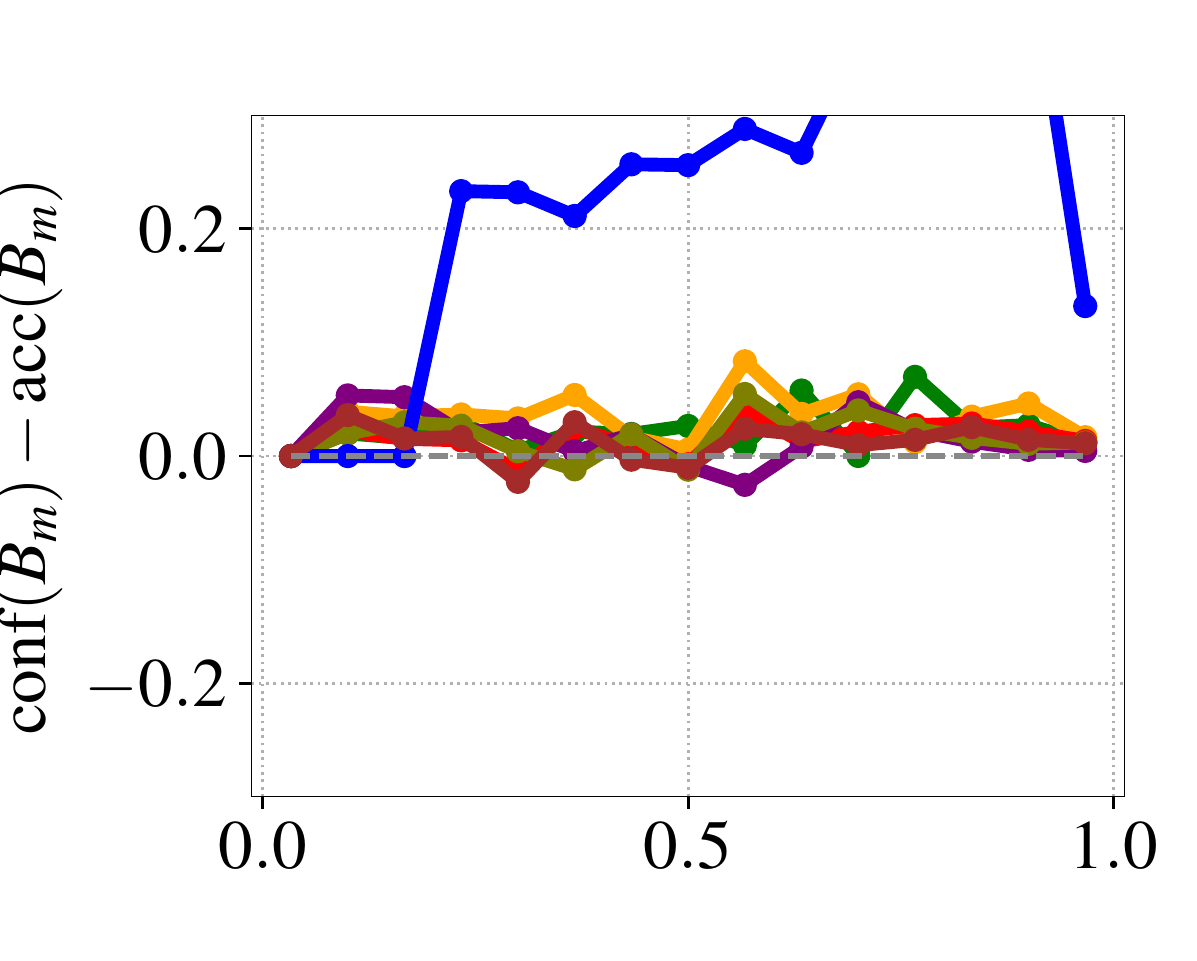}
    \makebox[20pt]{\raisebox{40pt}{\rotatebox[origin=c]{90}{\tiny CIFAR100, DenseNet 40}}}%
    \includegraphics[width=\dimexpr\linewidth-20pt\relax,trim=25 25 10 10, clip]
    {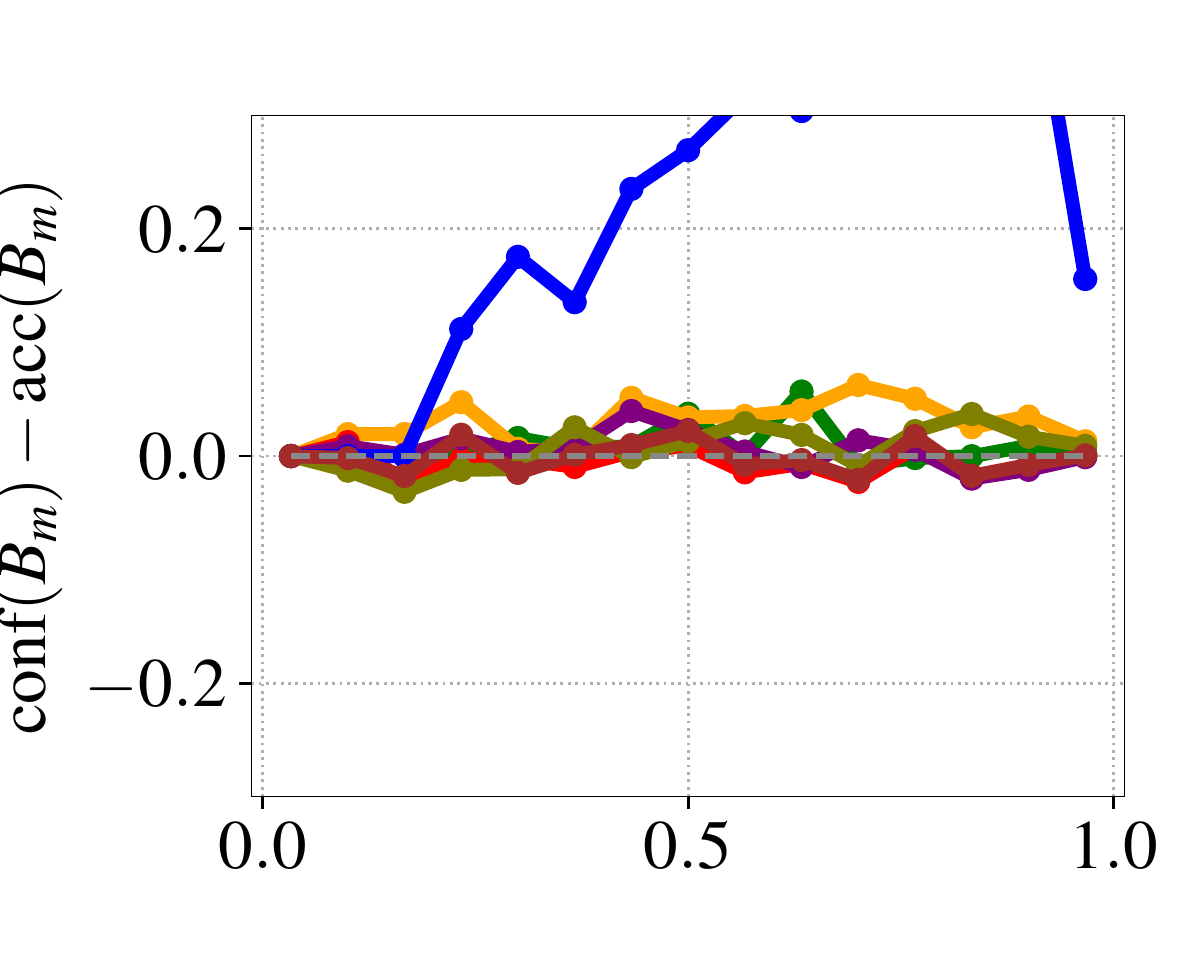}
    \caption{Reliability Diagram}
\end{subfigure}
\begin{subfigure}[t]{0.30\textwidth}
    \includegraphics[width=\textwidth,trim=25 25 10 10, clip]
    {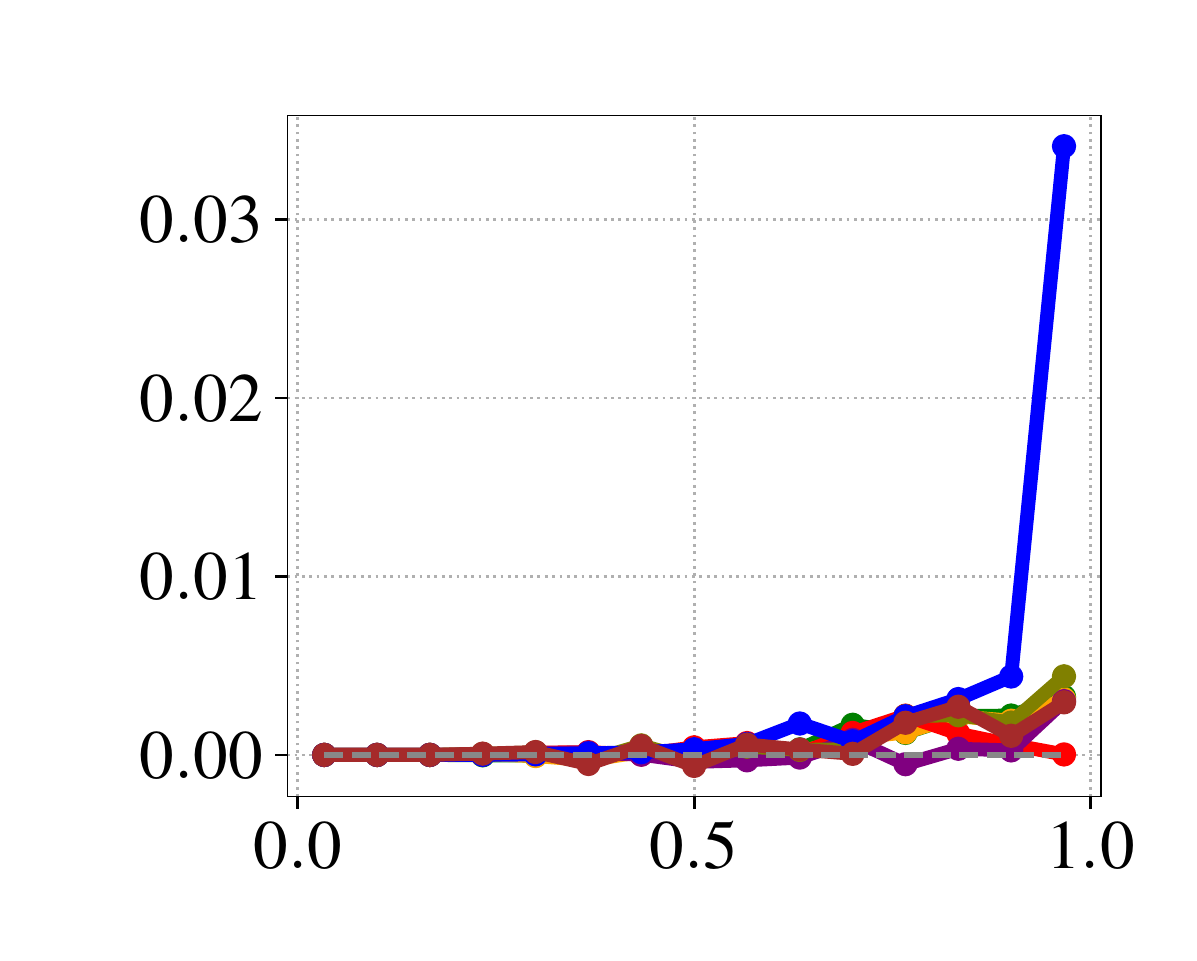}
    \includegraphics[width=\textwidth,trim=25 25 10 10, clip]
    {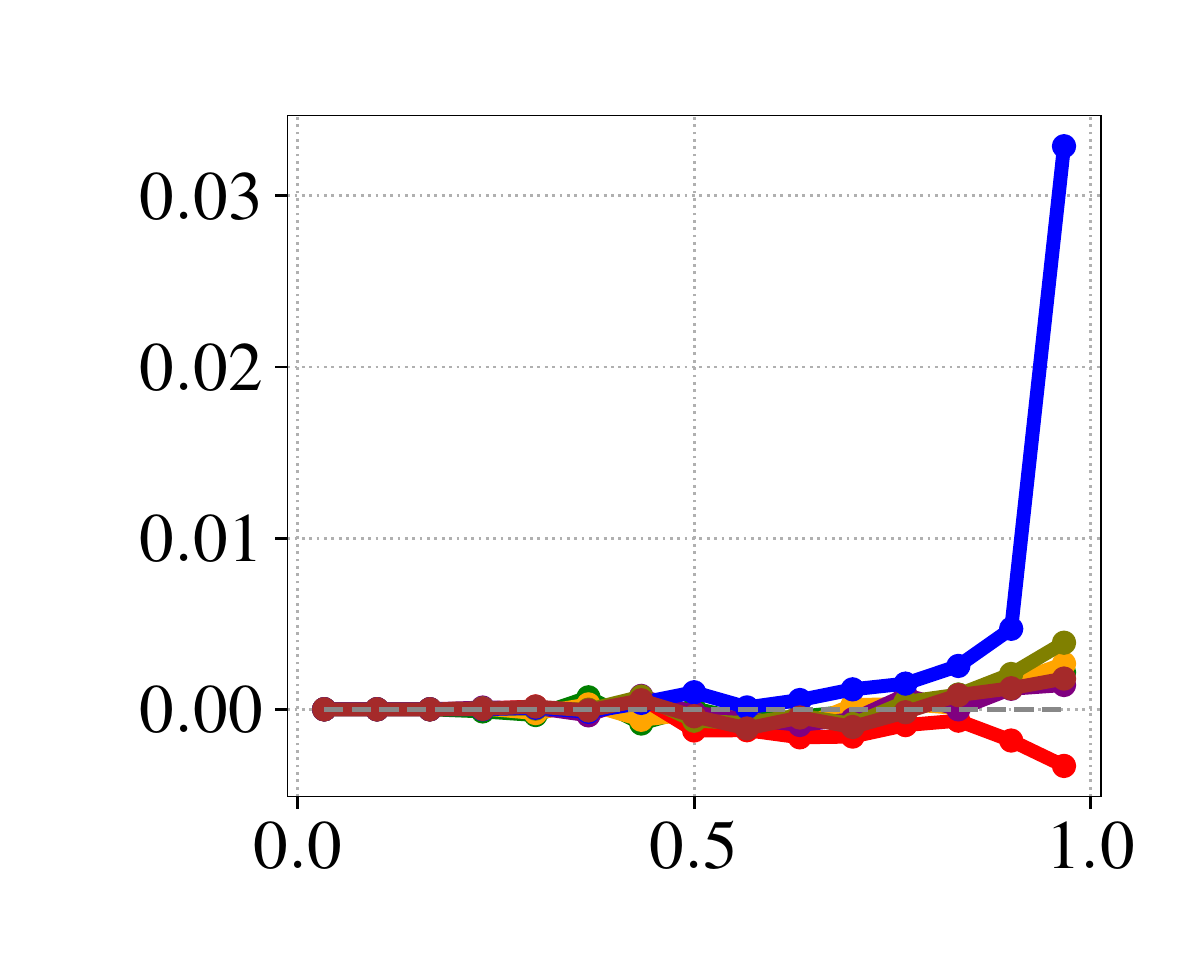}
    \includegraphics[width=\textwidth,trim=25 25 10 10, clip]
    {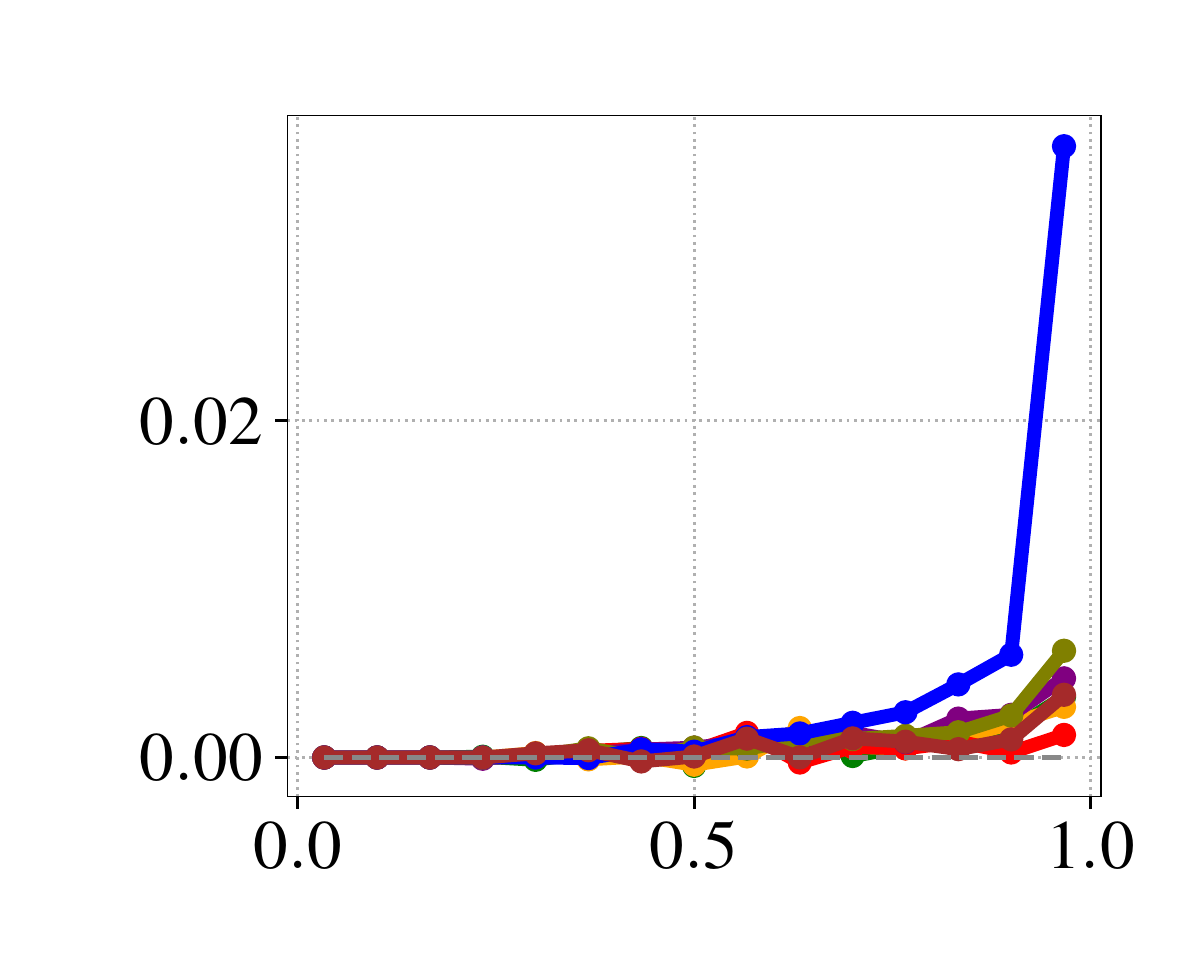}
    \includegraphics[width=\textwidth,trim=25 25 10 10, clip]
    {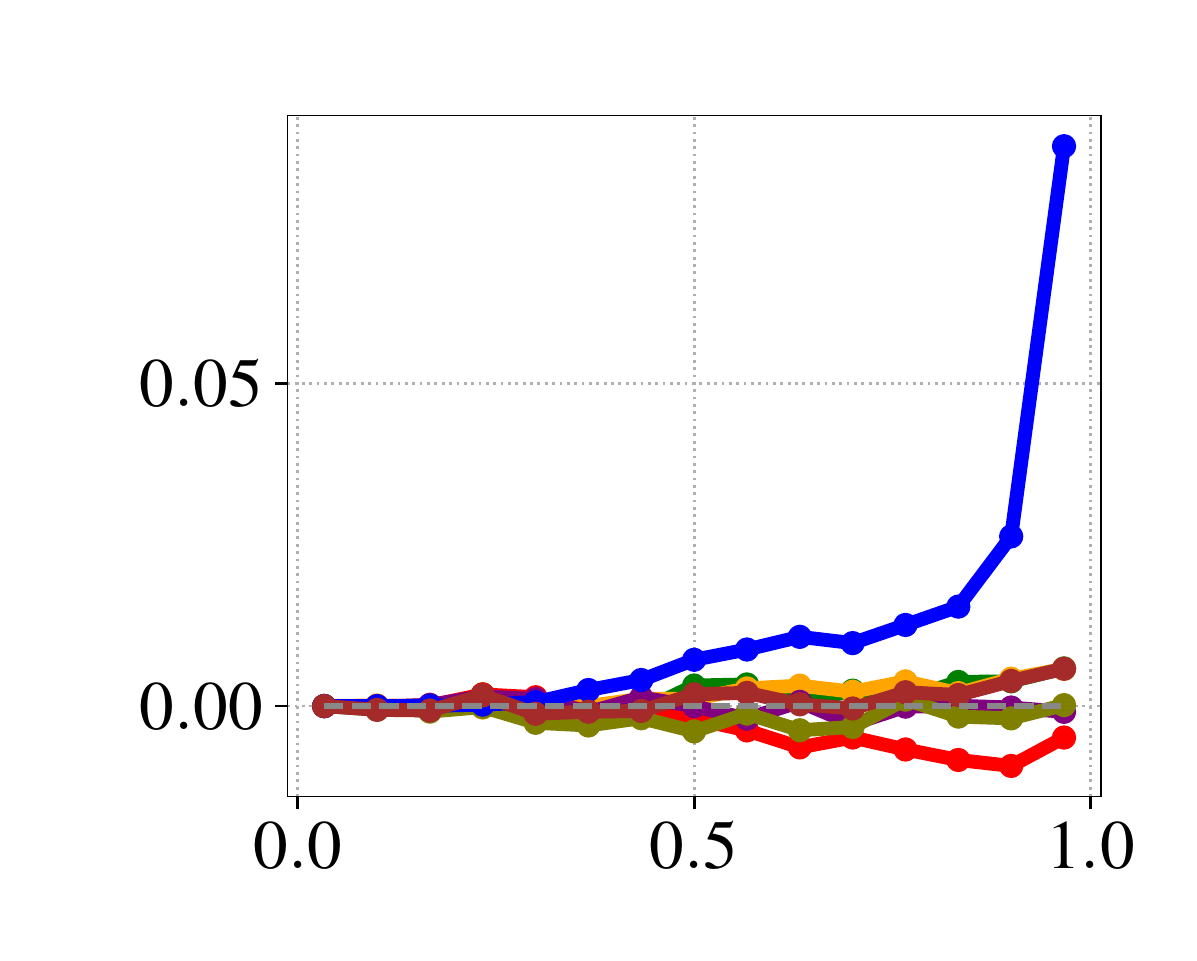}
    \includegraphics[width=\textwidth,trim=25 25 10 10, clip]
    {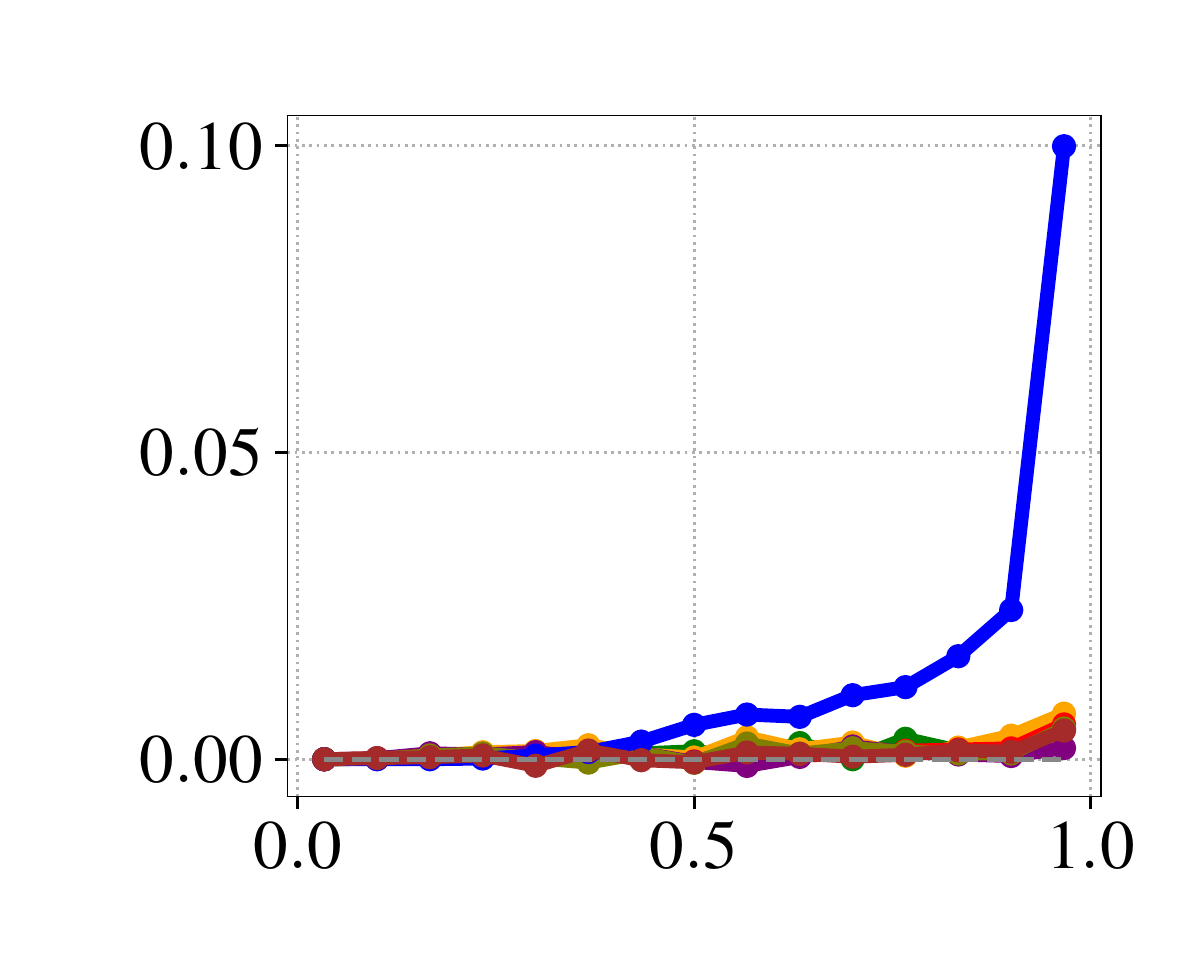}
    \includegraphics[width=\textwidth,trim=25 25 10 10, clip]
    {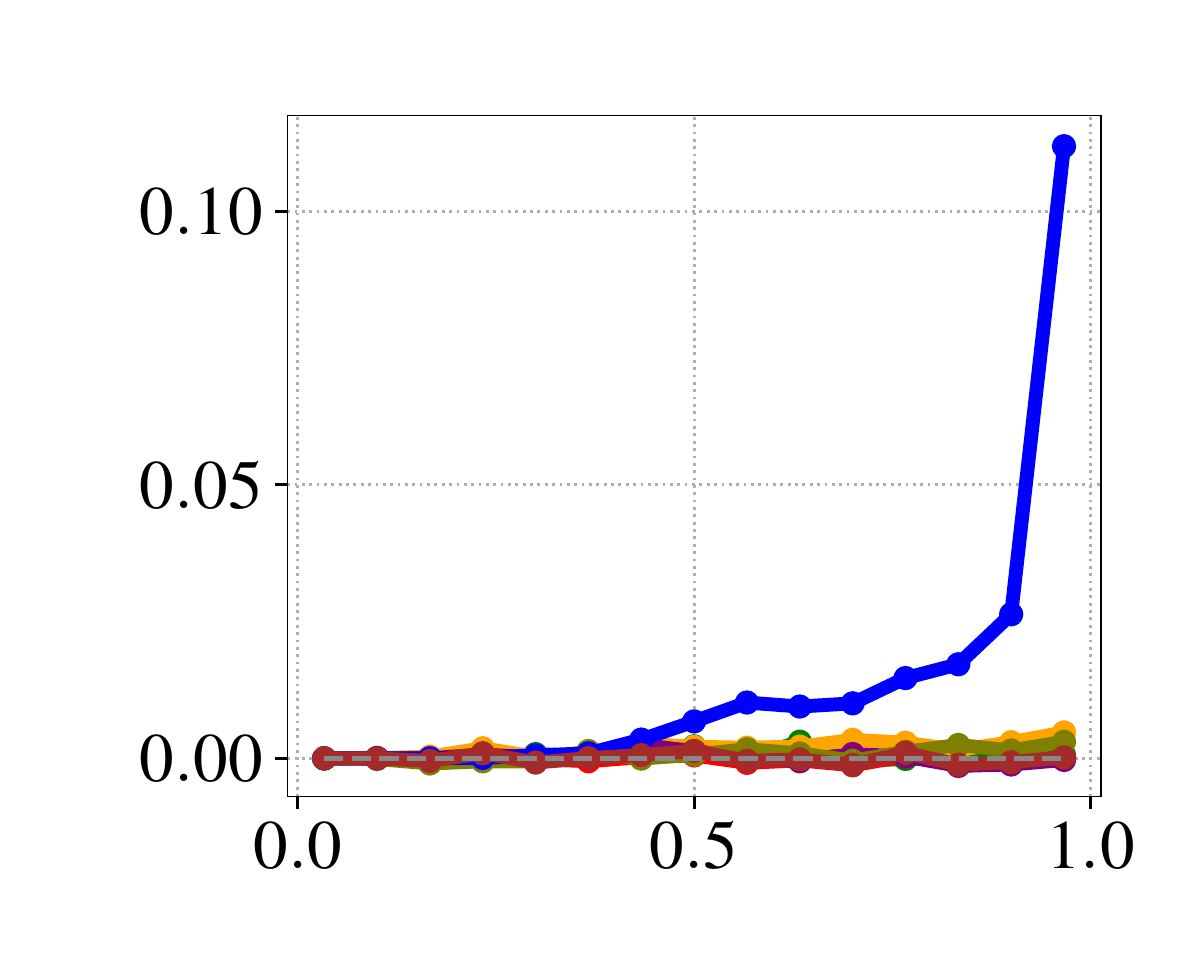}
    \caption{Weighted Reliability Diagram}
\end{subfigure}
\begin{subfigure}[t]{0.30\textwidth}
    \includegraphics[width=\textwidth,trim=25 25 10 10, clip]
    {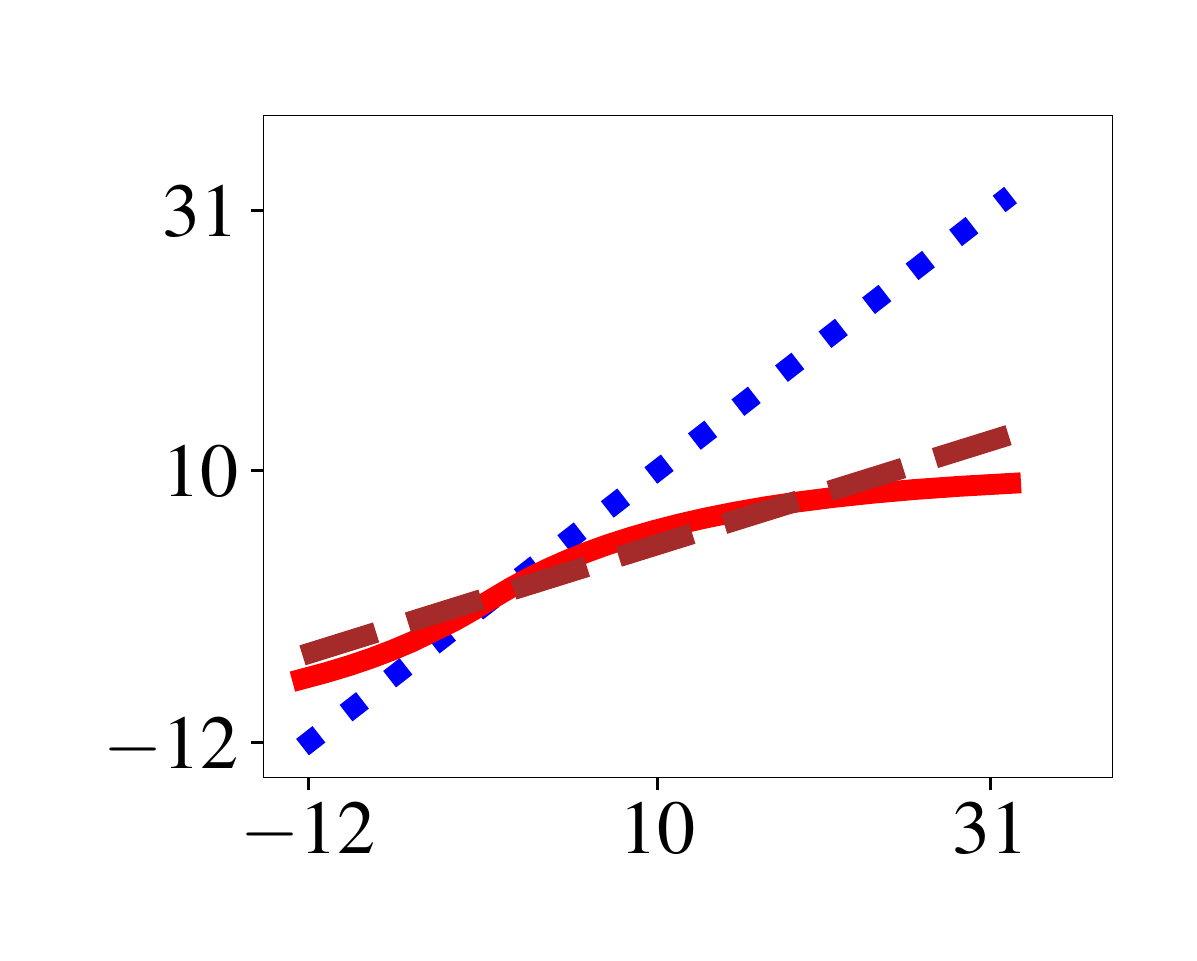}
    \includegraphics[width=\textwidth,trim=25 25 10 10, clip]
    {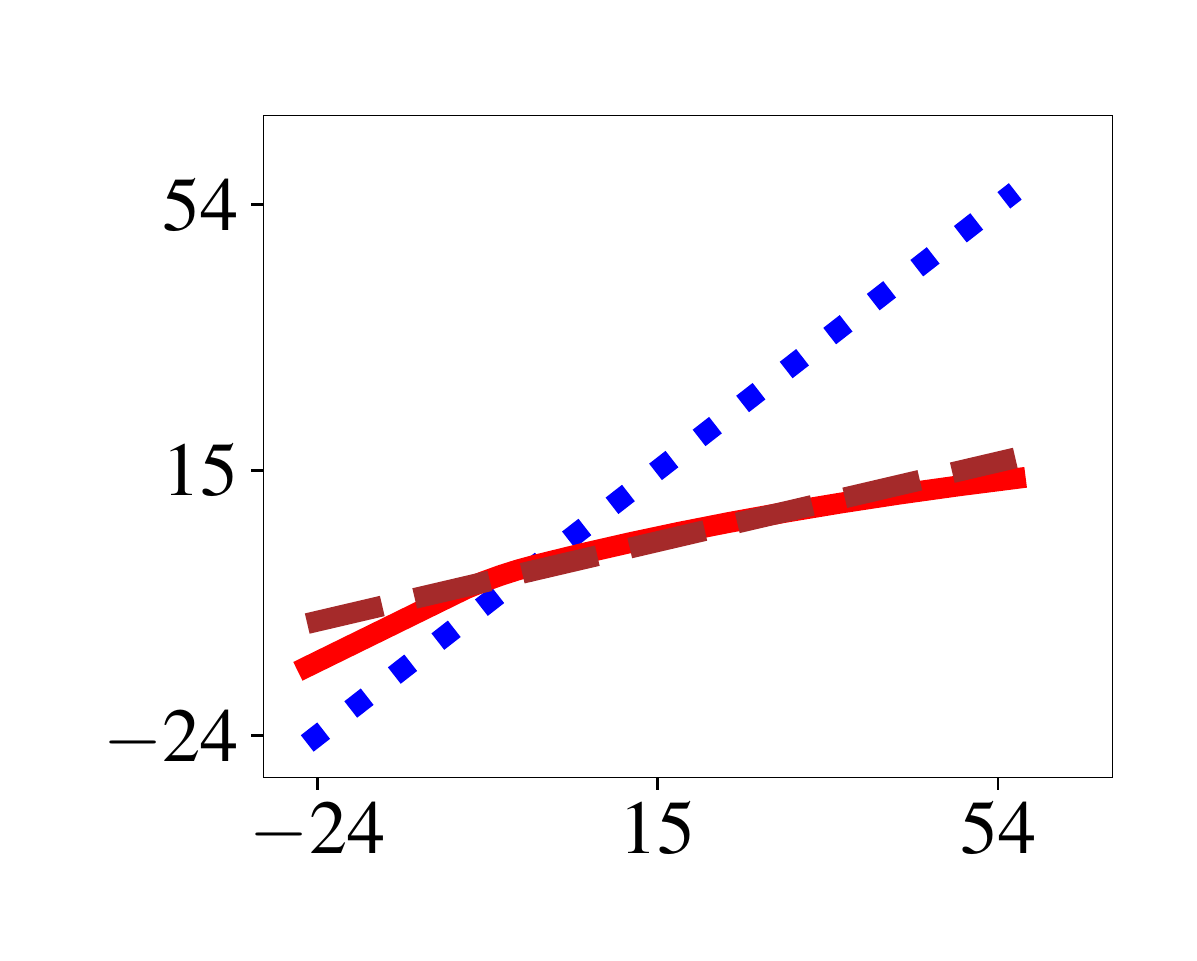}
    \includegraphics[width=\textwidth,trim=25 25 10 10, clip]
    {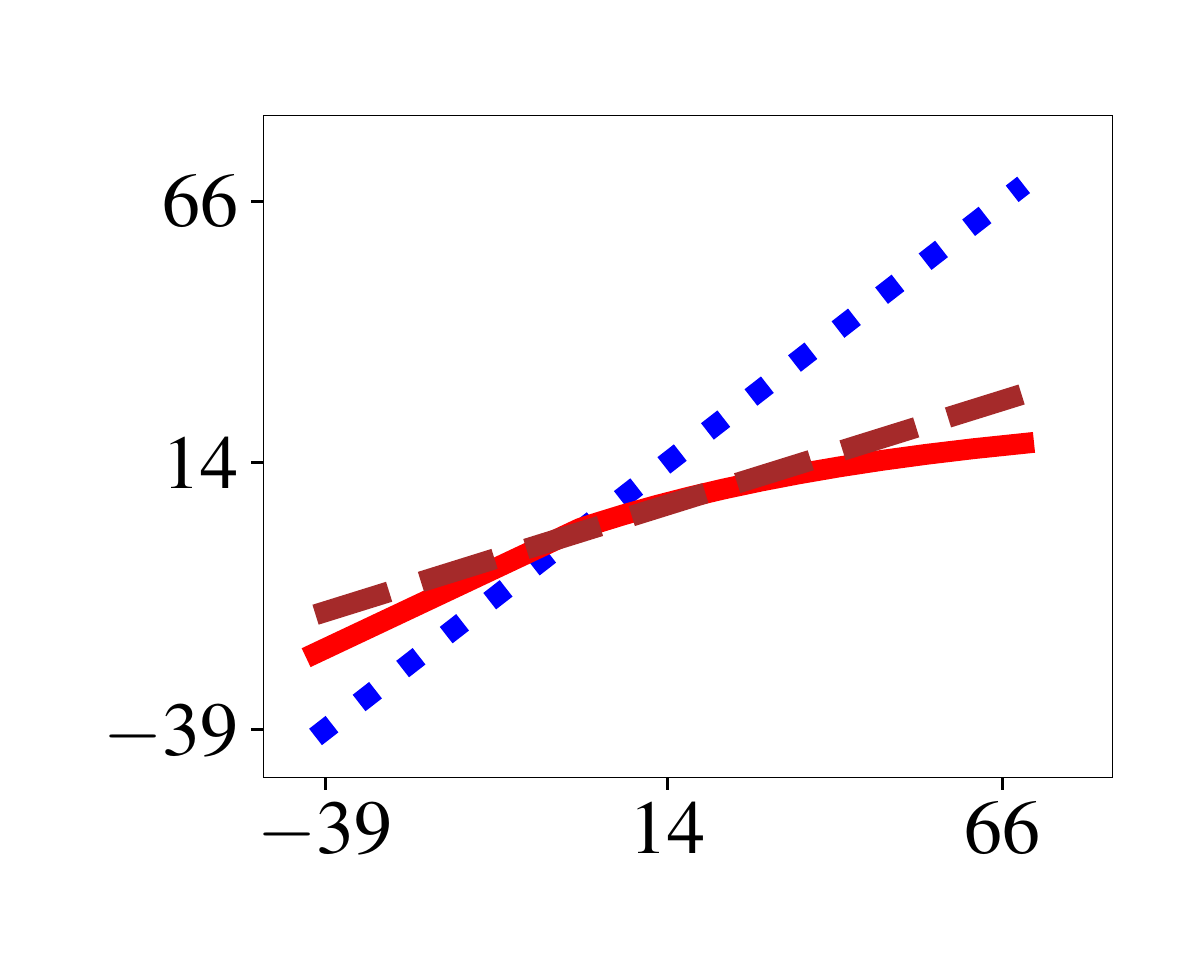}
    \includegraphics[width=\textwidth,trim=25 25 10 10, clip]
    {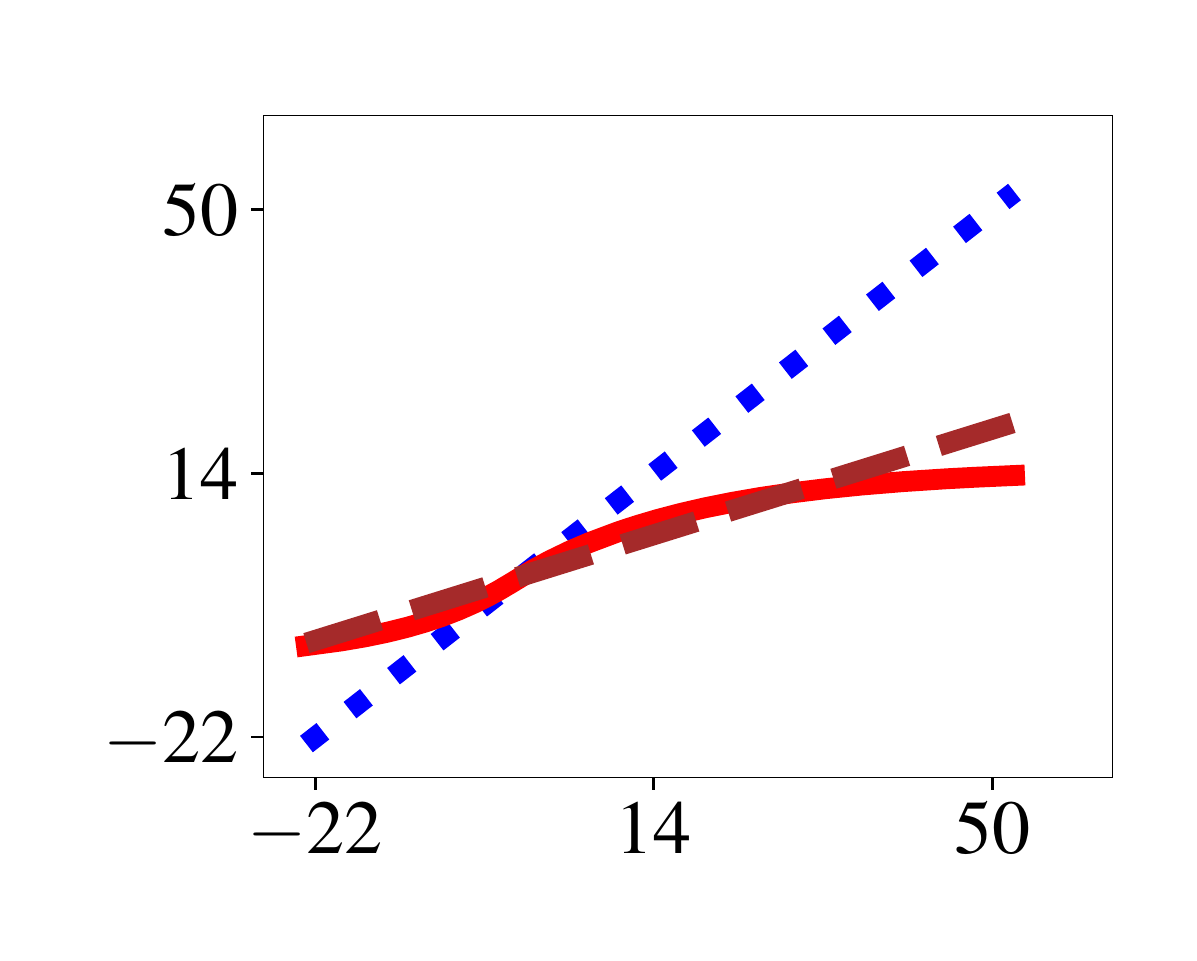}
    \includegraphics[width=\textwidth,trim=25 25 10 10, clip]
    {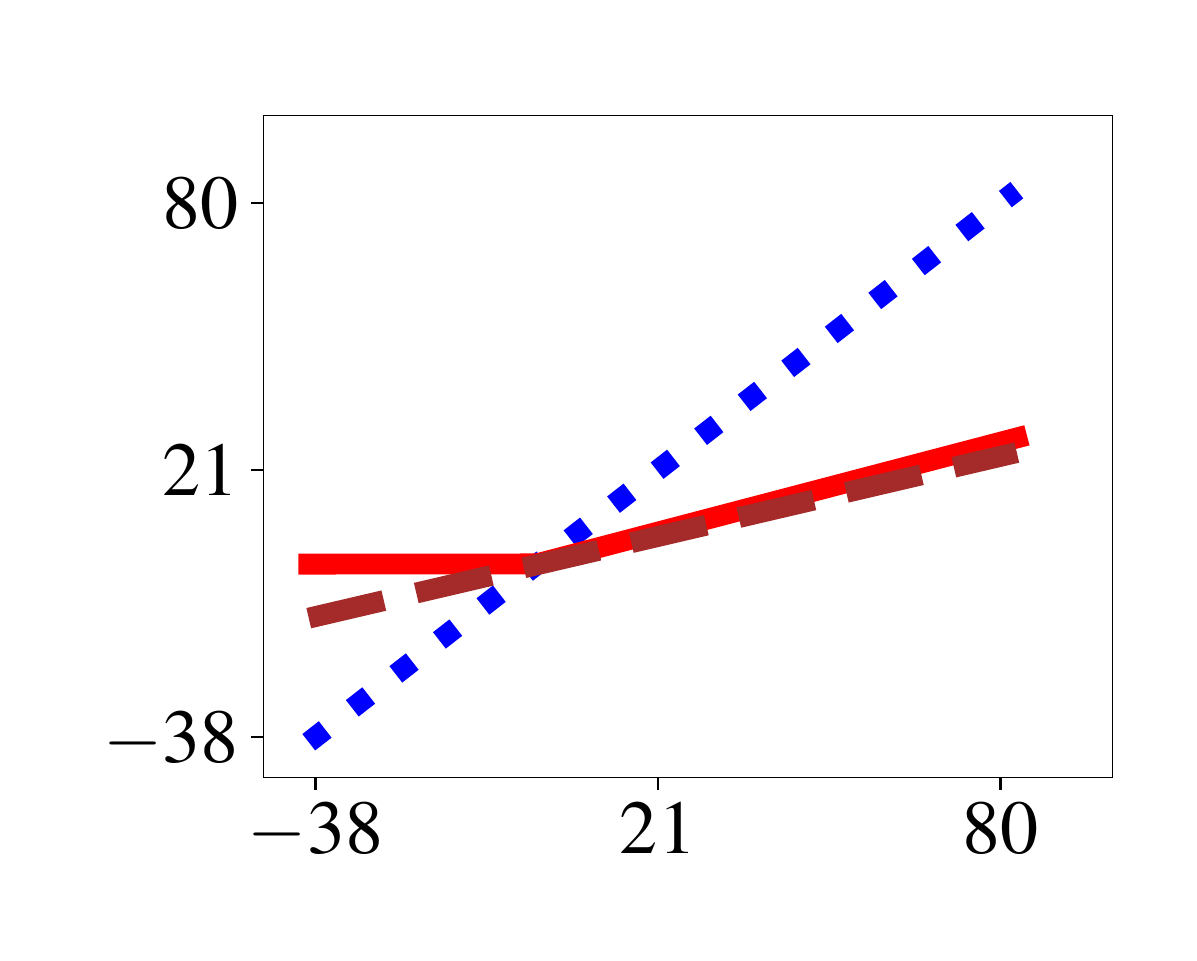}
    \includegraphics[width=\textwidth,trim=25 25 10 10, clip]
    {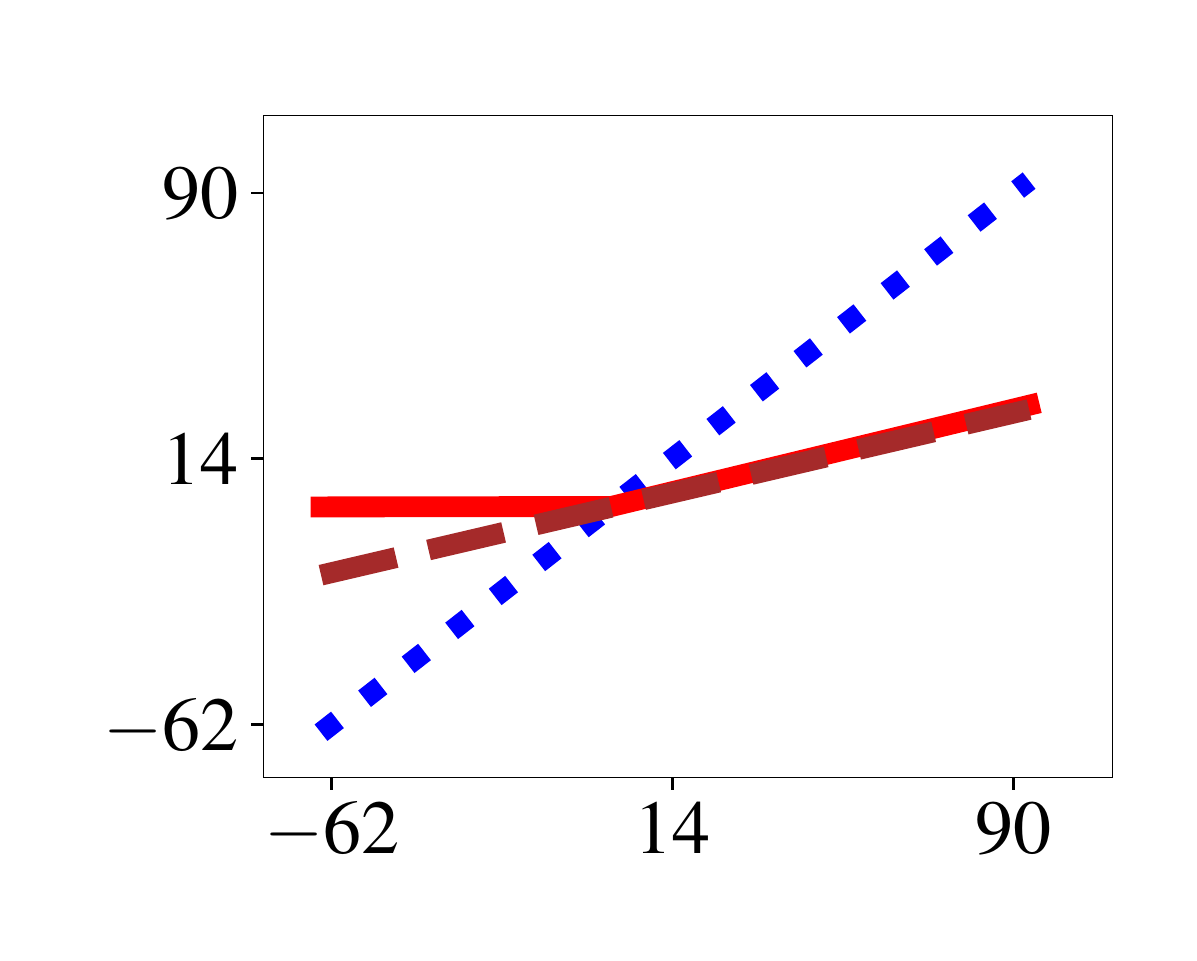}
    \caption{Diagonal Functions}
\end{subfigure}
\includegraphics[width=0.8\linewidth,trim=80 0 335 0, clip]{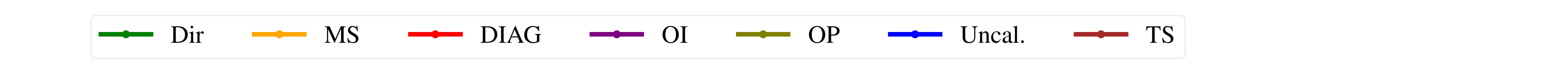}
\caption{\label{fig:reliability} Reliability diagrams and learned diagonal functions. See \cref{fig:transformations} for the explanation of each diagram and axis.}
\end{figure}
\begin{figure}
\centering
\begin{subfigure}[t]{\dimexpr0.30\textwidth+20pt\relax}
    \makebox[20pt]{\raisebox{40pt}{\rotatebox[origin=c]{90}{\tiny CARS, ResNet 50 (pre)}}}%
    \includegraphics[width=\dimexpr\linewidth-20pt\relax,trim=25 25 10 10, clip]
    {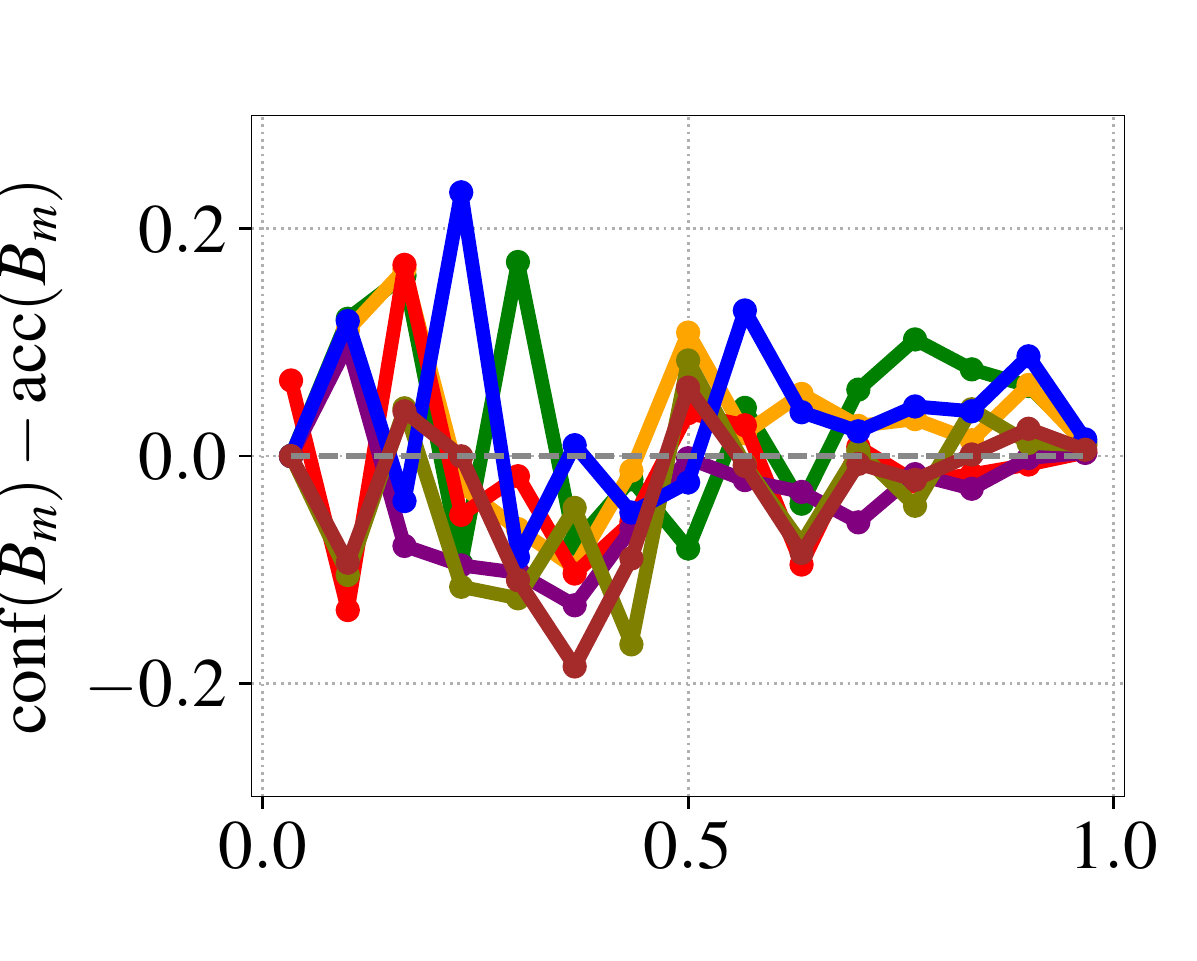}
    \makebox[20pt]{\raisebox{40pt}{\rotatebox[origin=c]{90}{\tiny CARS, ResNet 101 (pre)}}}%
    \includegraphics[width=\dimexpr\linewidth-20pt\relax,trim=25 25 10 10, clip]
    {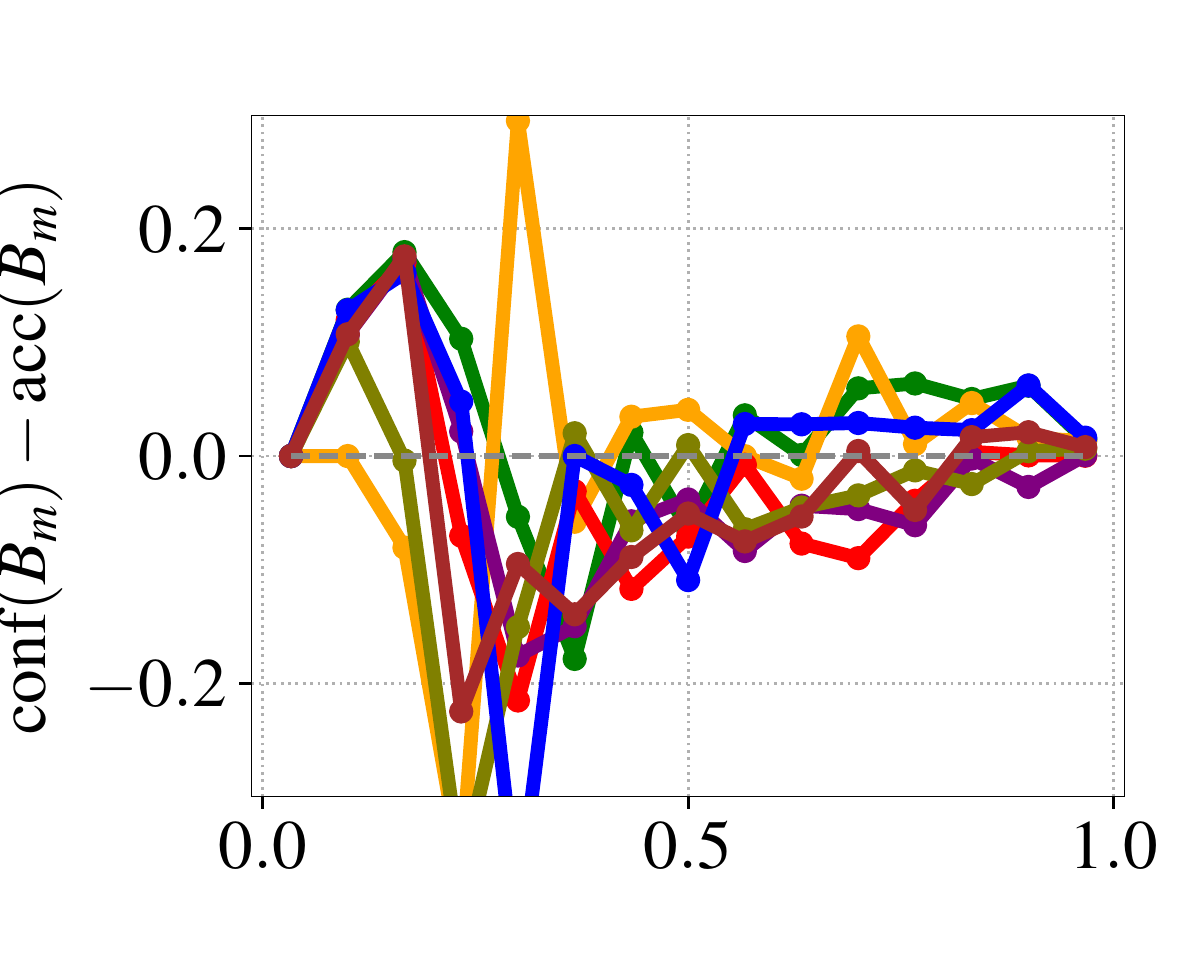}
    \makebox[20pt]{\raisebox{40pt}{\rotatebox[origin=c]{90}{\tiny CARS, ResNet 101}}}%
    \includegraphics[width=\dimexpr\linewidth-20pt\relax,trim=25 25 10 10, clip]
    {./figures/appendix/rel_resnet101_cars}
    \makebox[20pt]{\raisebox{40pt}{\rotatebox[origin=c]{90}{\tiny SVHN, ResNet 152 (SD)}}}%
    \includegraphics[width=\dimexpr\linewidth-20pt\relax,trim=25 25 10 10, clip]
    {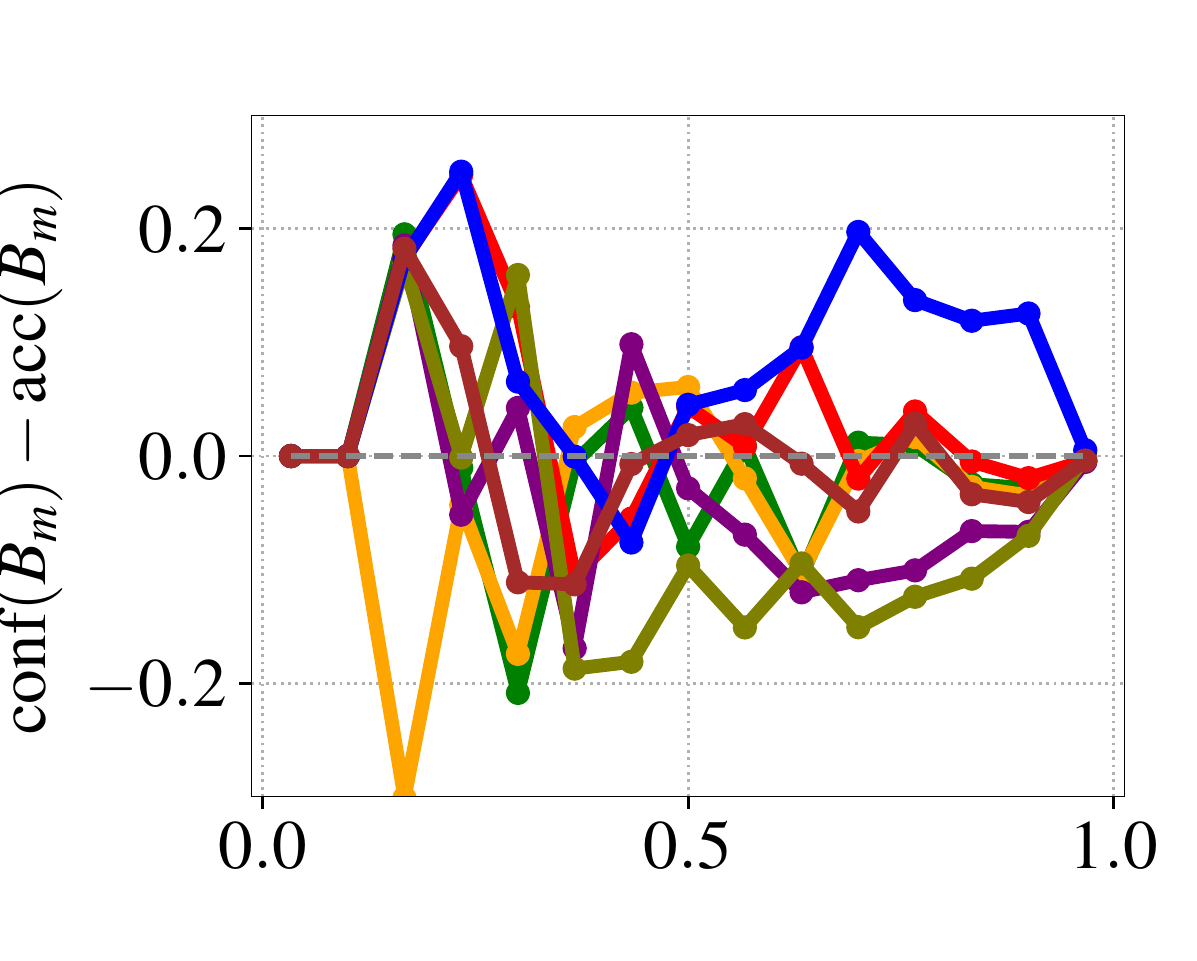}
    \makebox[20pt]{\raisebox{40pt}{\rotatebox[origin=c]{90}{\tiny BIRDS, ResNet 50 (NTS)}}}%
    \includegraphics[width=\dimexpr\linewidth-20pt\relax,trim=25 25 10 10, clip]
    {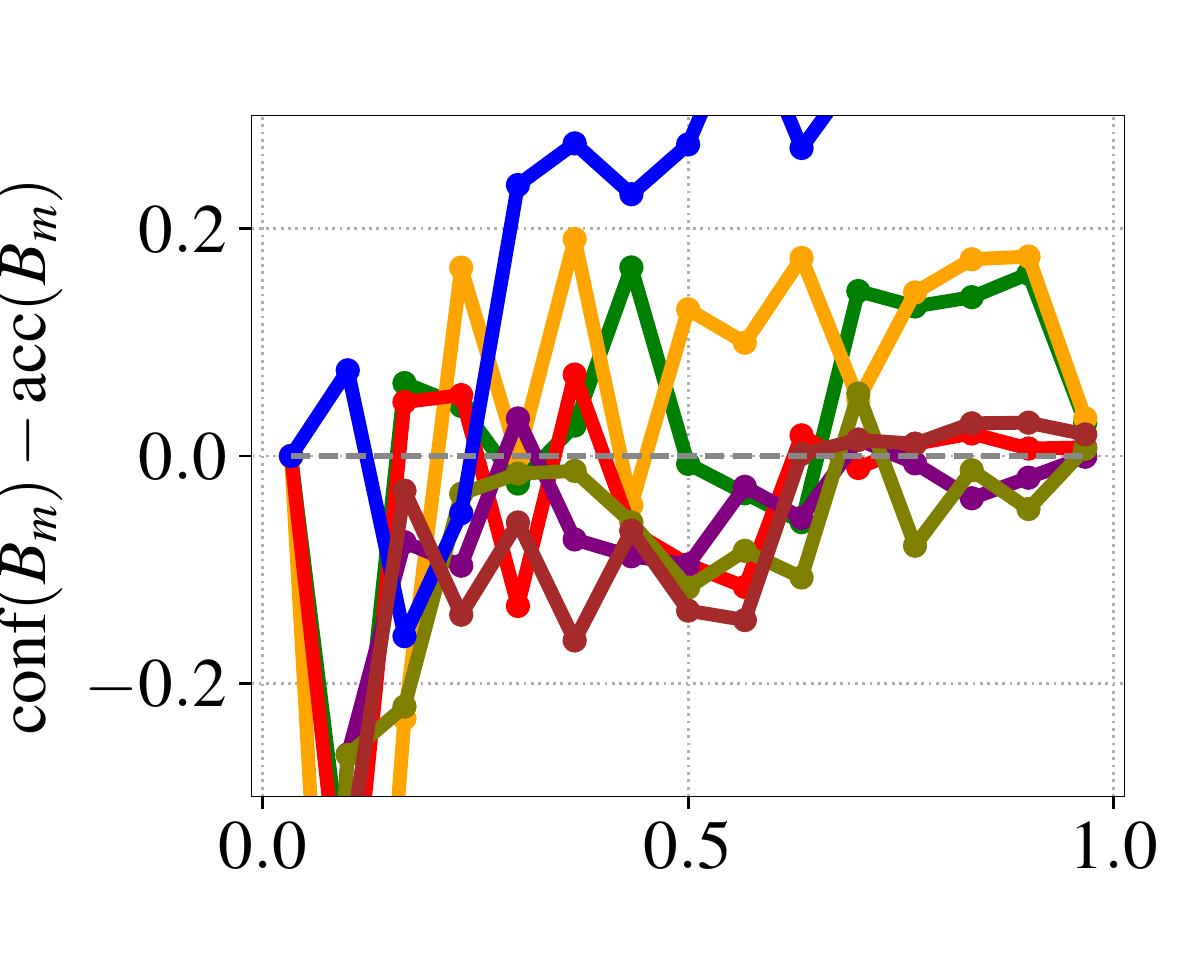}
    \makebox[20pt]{\raisebox{40pt}{\rotatebox[origin=c]{90}{\tiny ImageNet, DenseNet 161}}}%
    \includegraphics[width=\dimexpr\linewidth-20pt\relax,trim=25 25 10 10, clip]
    {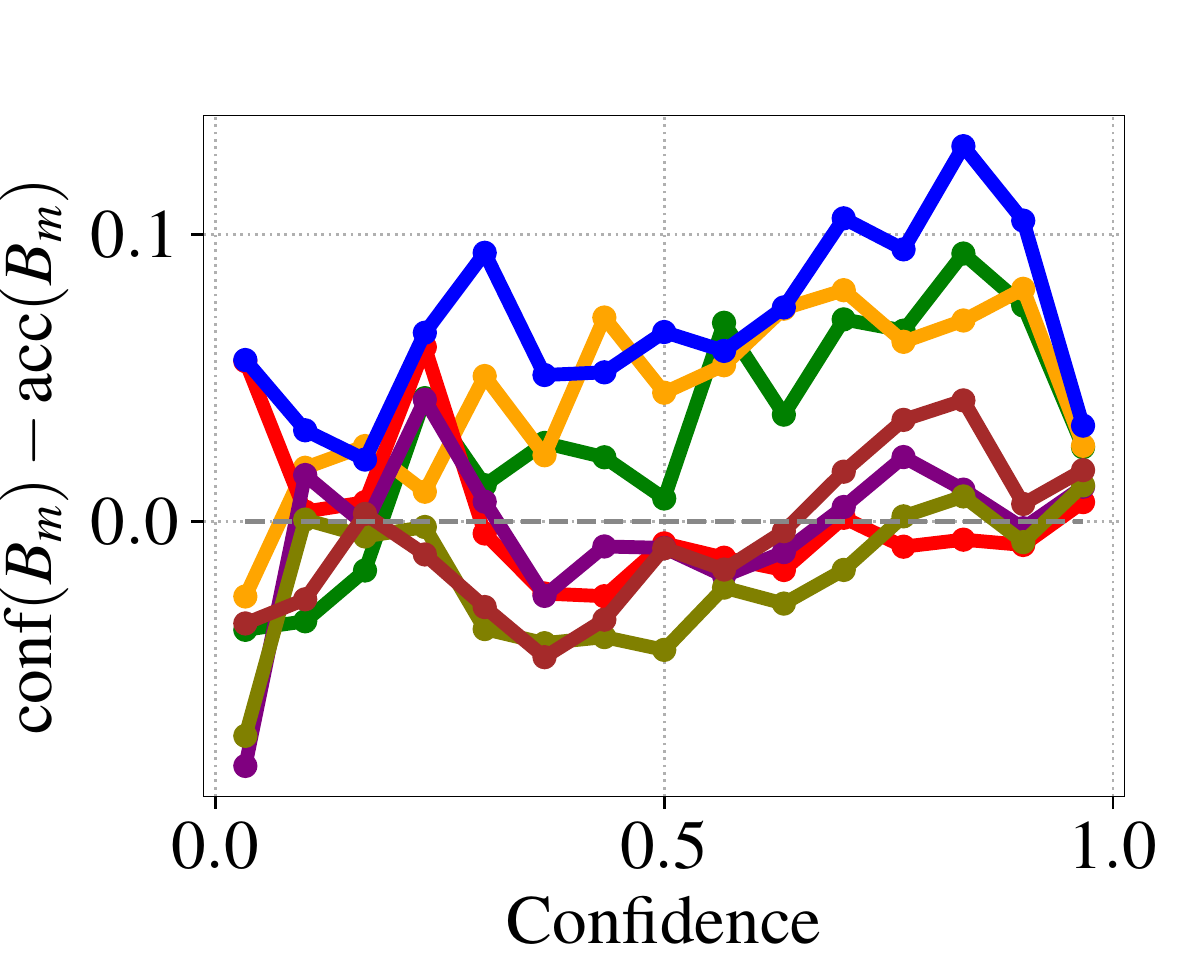}
    \caption{Reliability Diagram}
\end{subfigure}
\begin{subfigure}[t]{0.30\textwidth}
    \includegraphics[width=\textwidth,trim=25 25 10 10, clip]
    {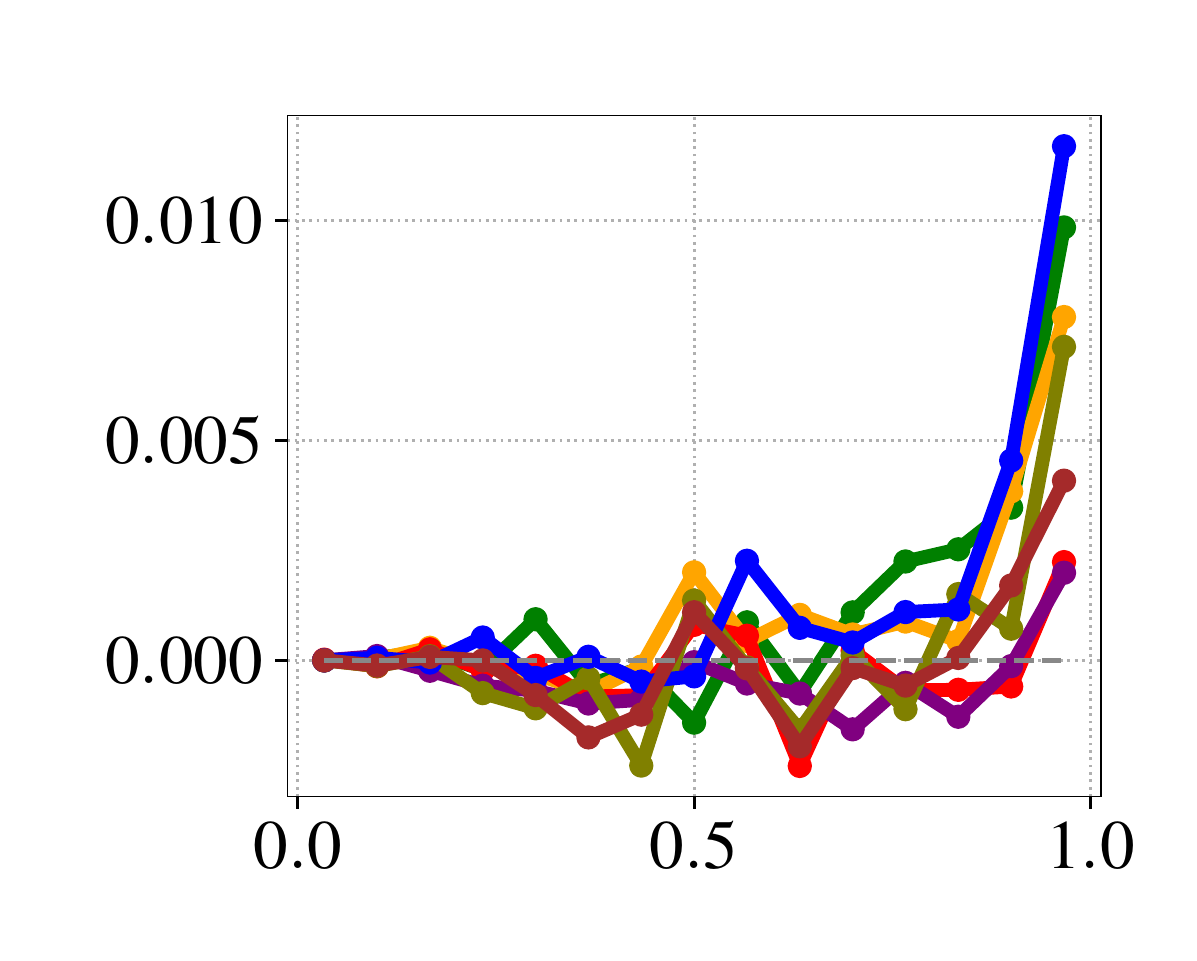}
    \includegraphics[width=\textwidth,trim=25 25 10 10, clip]
    {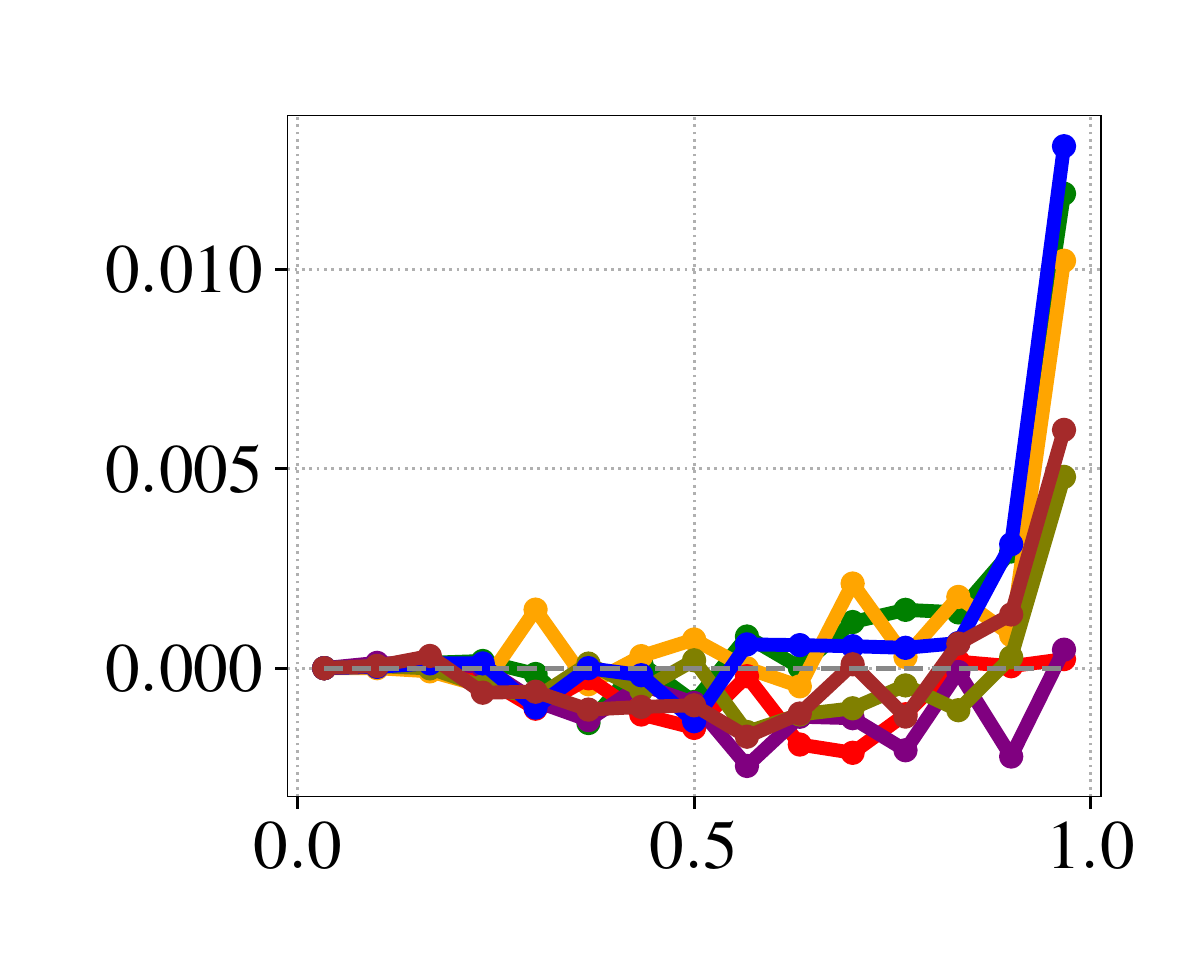}
    \includegraphics[width=\textwidth,trim=25 25 10 10, clip]
    {./figures/appendix/prop_resnet101_cars}
    \includegraphics[width=\textwidth,trim=25 25 10 10, clip]
    {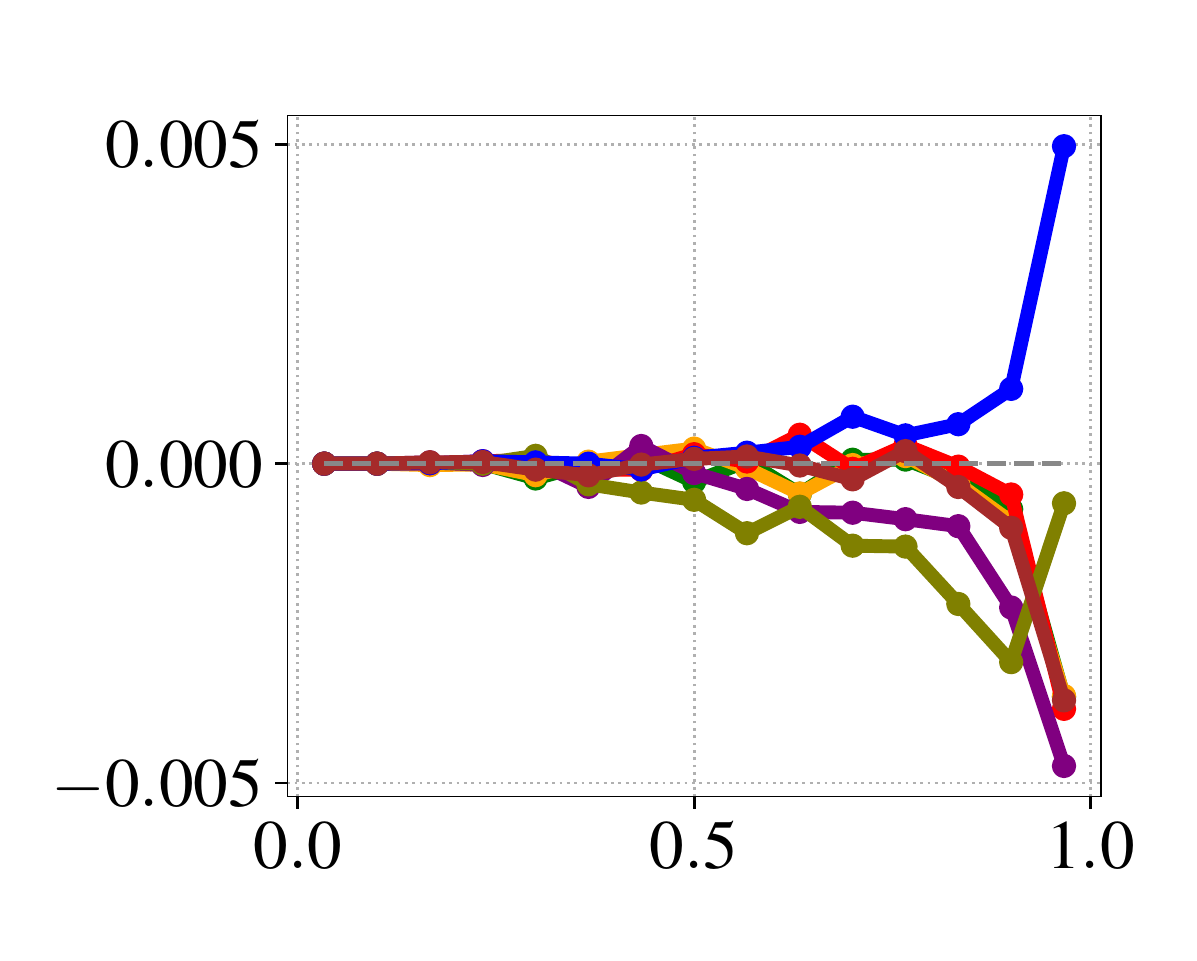}
    \includegraphics[width=\textwidth,trim=25 25 10 10, clip]
    {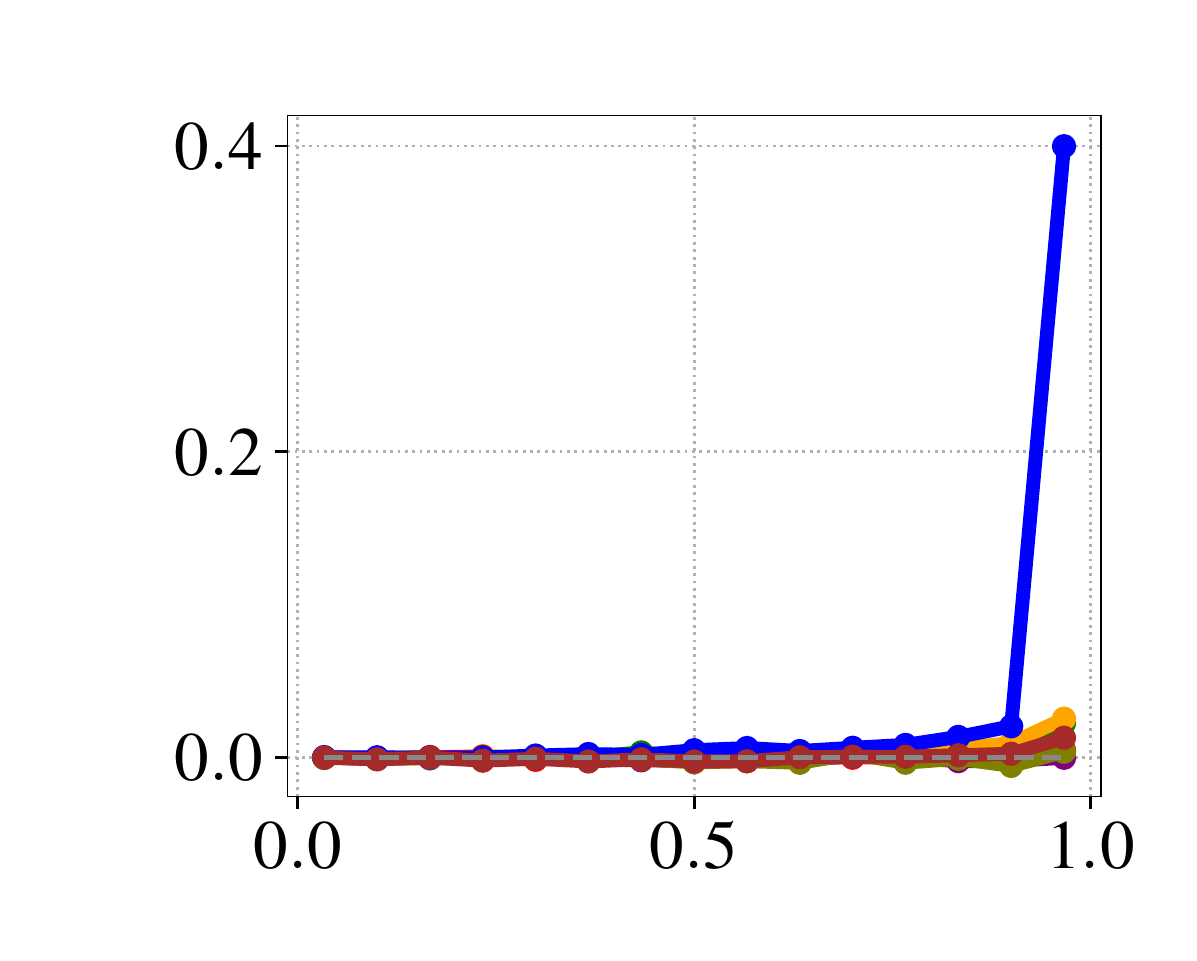}
    \includegraphics[width=\textwidth,trim=25 25 10 10, clip]
    {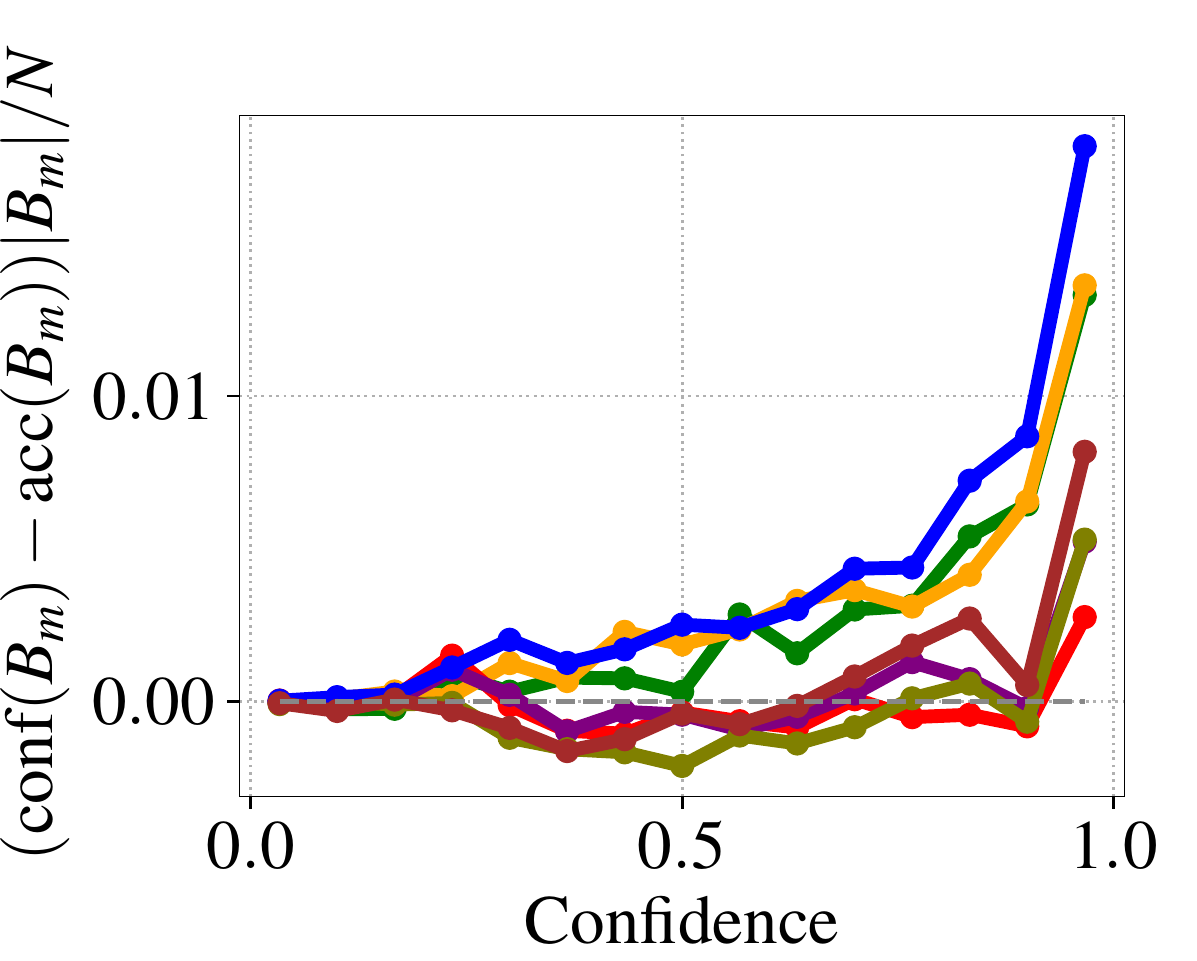}
    \caption{Weighted Reliability Diagram}
\end{subfigure}
\begin{subfigure}[t]{0.30\textwidth}
    \includegraphics[width=\textwidth,trim=25 25 10 10, clip]
    {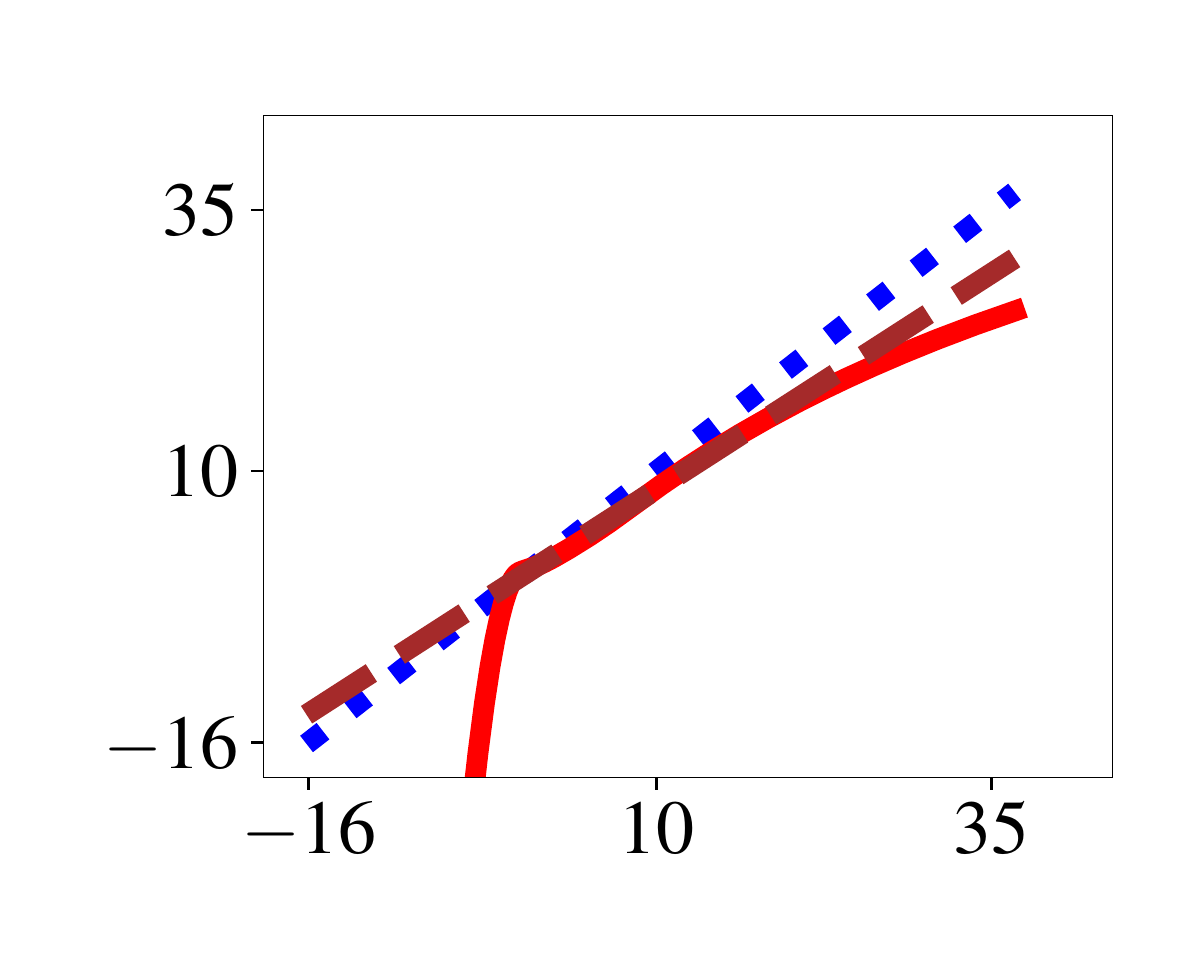}
    \includegraphics[width=\textwidth,trim=25 25 10 10, clip]
    {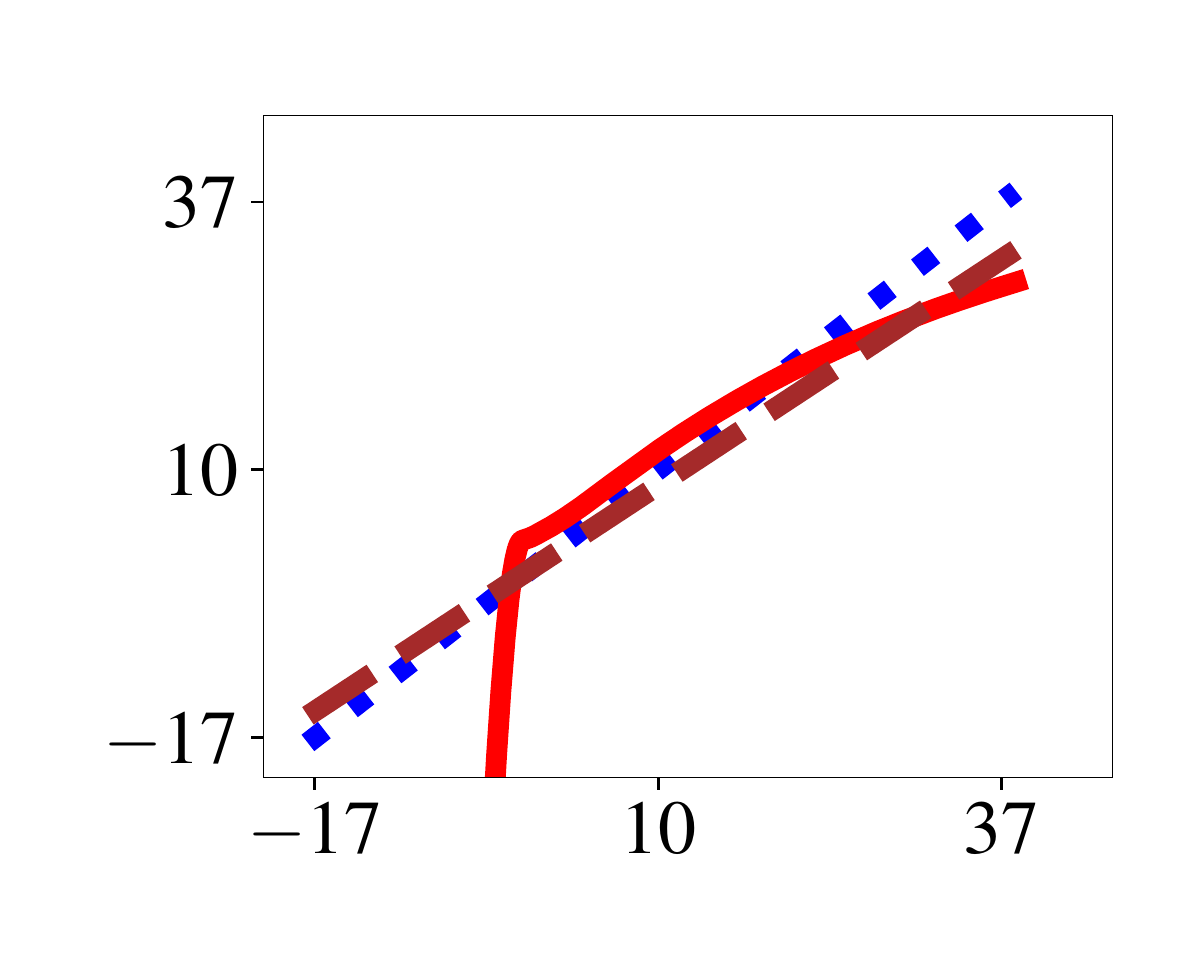}
    \includegraphics[width=\textwidth,trim=25 25 10 10, clip]
    {./figures/appendix/mono_resnet101_cars}
    \includegraphics[width=\textwidth,trim=25 25 10 10, clip]
    {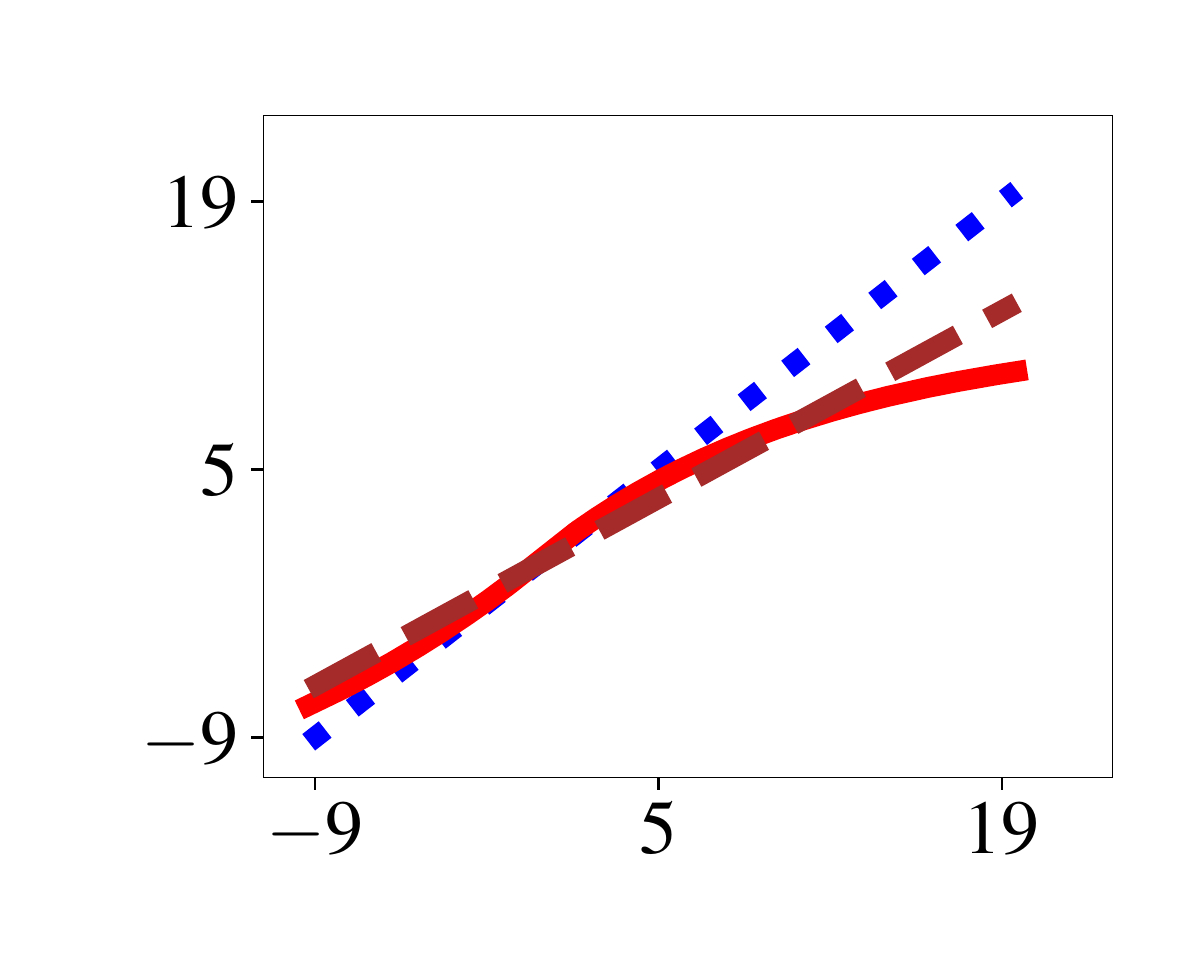}
    \includegraphics[width=\textwidth,trim=25 25 10 10, clip]
    {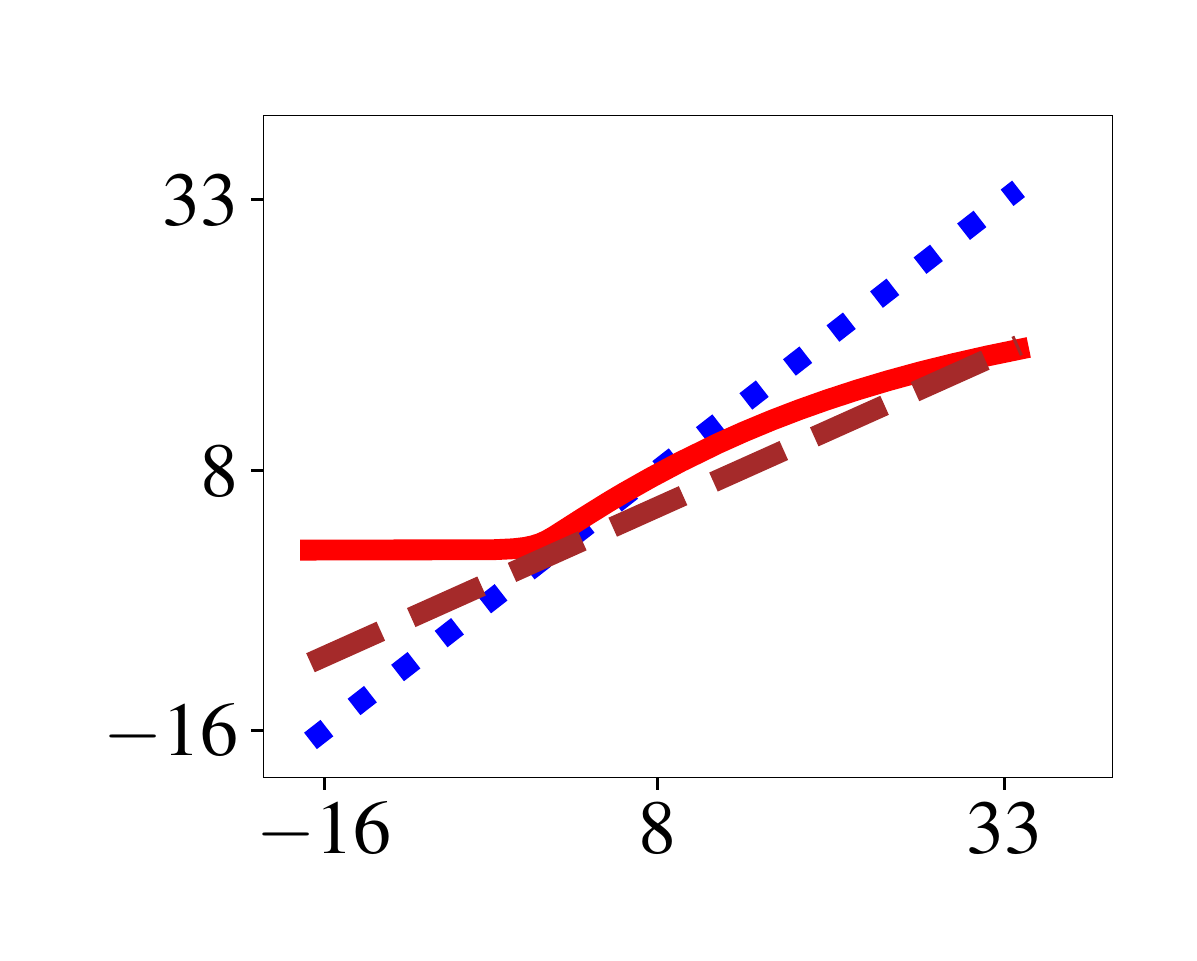}
    \includegraphics[width=\textwidth,trim=25 25 10 10, clip]
    {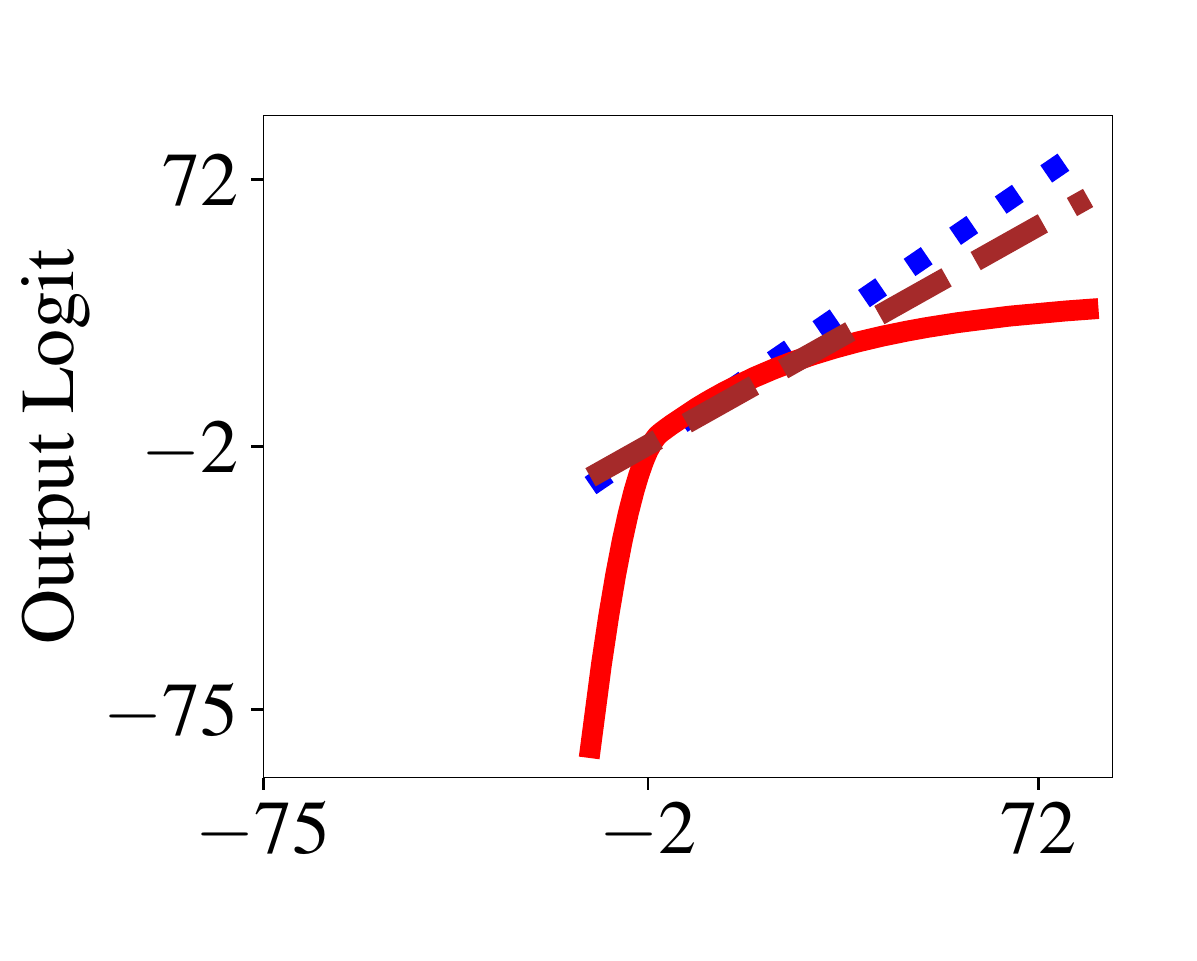}
    \caption{Diagonal Functions}
\end{subfigure}
\includegraphics[width=0.8\linewidth,trim=80 0 335 0, clip]{figures/prop_legend.pdf}
\caption{\label{fig:reliability2} Reliability diagrams and learned diagonal functions. See \cref{fig:transformations} for the explanation of each diagram and axis.}
\end{figure}

{\bf Calibration Set Size.} In this experiment, we gradually increase the calibration set size from $10\%$ to $100\%$ of its original size to create smaller calibration subsets. Then, for each calibration subset, we train different post-hoc calibration methods and measure their accuracy, NLL, and ECE. The results are illustrated in \cref{fig:sample_size}. In overall, the performance of non intra order-preserving methods, i.e. \dirodir and \msodir, are more sensitive to the size of the calibration set while intra order-preserving methods maintain the accuracy and are more stable in terms of NLL and ECE.
\begin{figure}
\centering
\includegraphics[width=0.8\linewidth]{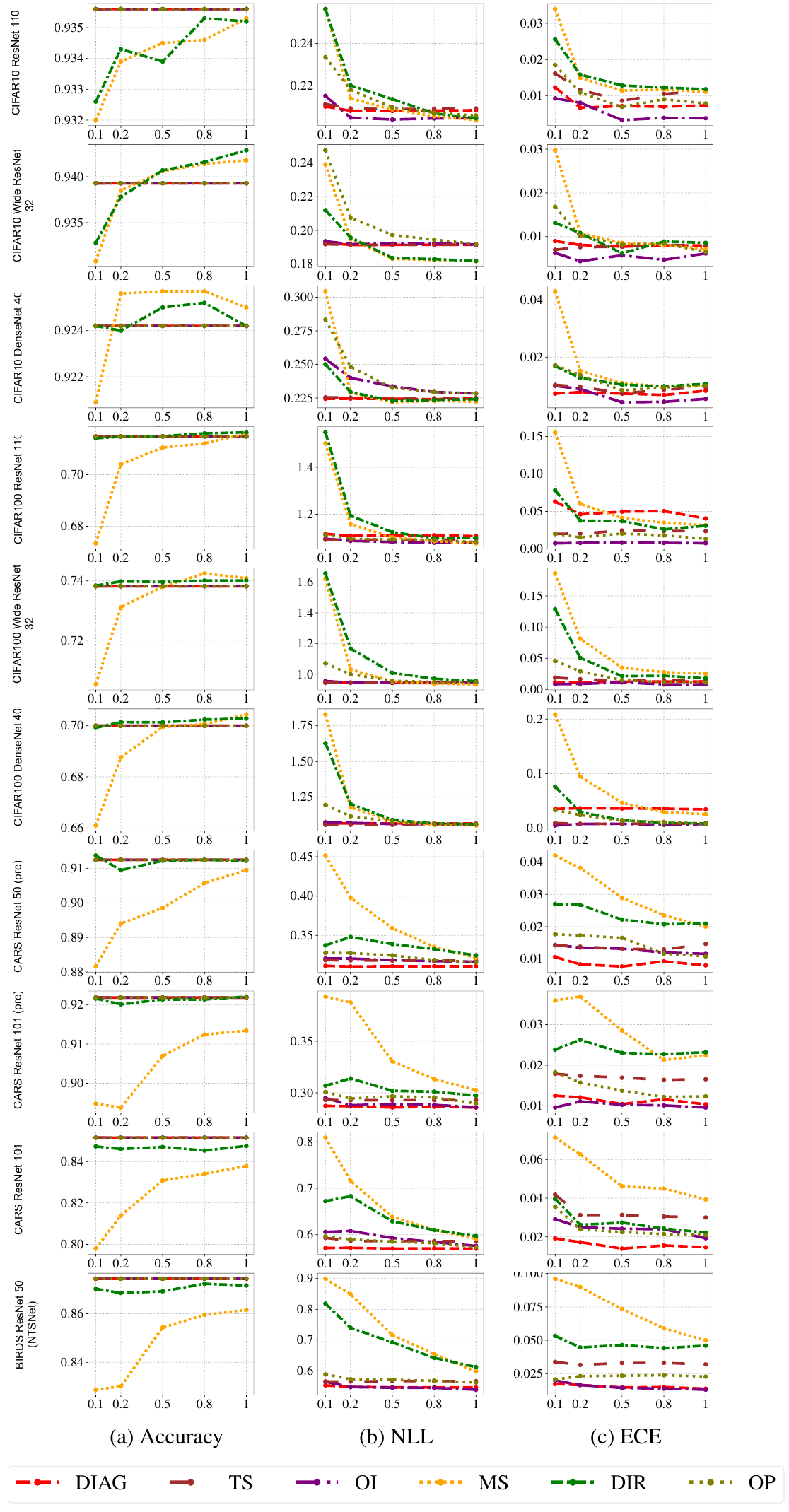}
\caption{\label{fig:sample_size} Accuracy, NLL, and ECE vs. calibration set size for CIFAR, CARS, BIRDS datasets. For each experiment, we use from $10\%$ to $100\%$ of the calibration set to train pos-hoc calibration functions and plot their accuracy, NLL, and ECE. Compared to \dirodir and \msodir, performance of the intra order-preserving methods (TS, \dgfn, \oinv, and \opre) degrades less with reducing the calibration set size.}
\end{figure}


{\bf Brier Score, NLL, and Classwise-ECE.} As shown in \cref{tbl:brier}, our \oinv is the best method in 5 out 14 models with respect to the Brier score. \msodir also wins in 4 models. However, it performs poorly on CARS and BIRDS datasets. Our \dgfn has the best average relative error. Overall, both \oinv and \dgfn perform well on this metric. The \dirodir is the third best method on this metric and is slightly worse than \oinv in average relative error.
\begin{table*}[!t]
\begin{center}
\caption{\em \label{tbl:brier}\small{Scores and rankings of different methods for Brier.}}
\resizebox{1.0\linewidth}{!}{%
\begin{tabular}{c c | c c c c | c c c}
\toprule
Dataset  & Model & Uncal. & TS & \dirodir & \msodir & \dgfn & \oinv & \opre \\
\hline
CIFAR10 & ResNet 110 & $0.01102_7$ & $0.00979_6$ & $0.00977_5$ & $0.00976_4$ & $0.00967_2$ & ${\bf 0.00963}_1$ & $0.00975_3$ \\
CIFAR10 & Wide ResNet 32 & $0.01047_7$ & $0.00924_4$ & ${\bf 0.00888}_1$ & $0.00889_2$ & $0.00926_6$ & $0.00921_3$ & $0.00925_5$ \\
CIFAR10 & DenseNet 40 & $0.01274_7$ & $0.01100_4$ & ${\bf 0.01097}_1$ & ${\bf 0.01097}_1$ & $0.01100_4$ & $0.01110_6$ & $0.01099_3$ \\
\hline
SVHN & ResNet 152 (SD) & $0.00297_6$ & ${\bf 0.00291}_1$ & $0.00293_3$ & $0.00298_7$ & $0.00292_2$ & $0.00293_3$ & $0.00296_5$ \\
\hline
CIFAR100 & ResNet 110 & $0.00453_7$ & $0.00392_4$ & $0.00391_3$ & $0.00391_3$ & $0.00393_5$ & ${\bf 0.00389}_1$ & $0.00390_2$ \\ 
CIFAR100 & Wide ResNet 32 & $0.00432_7$ & $0.00355_4$ & $0.00354_2$ & ${\bf 0.00351}_1$ & $0.00355_4$ & $0.00354_2$ & $0.00355_4$ \\
CIFAR100 & DenseNet 40 & $0.00491_7$ & $0.00401_3$ & ${\bf 0.00400}_1$ & ${\bf 0.00400}_1$ & $0.00401_3$ & $0.00401_3$ & $0.00402_6$ \\
\hline
CARS & ResNet~50 (pretrained) & $0.000667_5$ & $0.000666_4$ & $0.000663_2$ & $0.000679_7$ & ${\bf 0.000661}_1$ & $0.000664_3$ & $0.000674_6$ \\
CARS & ResNet~101 (pretrained) & $0.000626_6$ & $0.000625_5$ & $0.000623_3$ & $0.000655_7$ & $0.000622_2$ & ${\bf 0.000620}_1$ & $0.000623_3$ \\
CARS & ResNet~101 & $0.001131_6$ & $0.001129_5$ & $0.001123_3$ & $0.001154_7$ & ${\bf 0.001118}_1$ & $0.001123_3$ & $0.001119_2$ \\
\hline
BIRDS & ResNet~50 (NTSNet) & $0.001035_6$ & $0.000995_5$ & $0.000988_4$ & $0.001040_7$ & $0.000977_3$ & ${\bf 0.000972}_1$ & $0.000974_2$ \\
\hline 
ImageNet & ResNet~152 & $0.000338_7$ & $0.000332_4$ & $0.000333_6$ & $0.000332_4$ & ${\bf 0.000329}_1$ & $0.000330_2$ & $0.000331_3$ \\
ImageNet & DenseNet~161 & $0.000325_7$ & $0.000321_4$ & $0.000321_4$ & ${\bf 0.000318}_1$ & $0.000319_2$ & $0.000320_3$ & $0.000321_4$ \\
ImageNet & PNASNet5~large & $0.000255_6$ & $0.000261_7$ & $0.000252_5$ & $0.000247_3$ & $0.000245_2$ & ${\bf 0.000244}_1$ & $0.000248_4$  \\
\midrule
\multicolumn{2}{c|}{Average Relative Error} & $1.000_7$ & $0.936_5$ & $0.930_3$ & $0.936_5$ & ${\bf 0.924}_1$ & $0.929_2$ & $0.931_4$ \\
\bottomrule
\end{tabular}
}
\end{center}
\end{table*}

Results of different methods regarding the NLL metric are shown in \cref{tbl:nll}. \msodir is the best method when the number of classes is less than or equal to 100 on this metric. Its performance degrades as the number of classes grows. This is typically due to the excessive number of parameters introduced by this method.  Surprisingly, TS is the best method in SVHN with ResNet~152~(SD) model but its performance is very similar to the \dgfn. The reason is that this model has a very high accuracy and the original model is actually already well calibrated. So, the single parameter TS would be enough to improve the calibration slightly. Our \dgfn is the best method on datasets with larger number of classes and our \oinv is also comparable to it. Both these method have the best average ranking and \dgfn has the best relative error on NLL.
\begin{table*}[!t]
\begin{center}
\caption{\em \label{tbl:nll}\small{NLL.}}
\resizebox{1.0\linewidth}{!}{%
\begin{tabular}{c c | c c c c | c c c}
\toprule
Dataset  & Model & Uncal. & TS & \dirodir & \msodir & \dgfn & \oinv & \opre \\
\hline
CIFAR10 & ResNet 110 & $0.35827_7$ & $0.20926_5$ & $0.20511_3$ & ${\bf 0.20375}_1$ & $0.20674_4$ & $0.20488_2$ & $0.20954_6$ \\
CIFAR10 & Wide ResNet 32 & $0.38170_7$ & $0.19148_3$ & $0.18203_2$ & ${\bf 0.18165}_1$ & $0.19221_5$ & $0.19169_4$ & $0.19332_6$ \\
CIFAR10 & DenseNet 40 & $0.42821_7$ & $0.22509_3$ & $0.22371_2$ & ${\bf 0.22240}_1$ & $0.22551_4$ & $0.23097_6$ & $0.22798_5$ \\
\hline
SVHN & ResNet 152 (SD) & $0.08542_7$ & ${\bf 0.07861}_1$ & ${0.08038}_5$ & $0.08100_6$ & $0.07887_2$ & $0.07992_3$ & $0.08010_4$ \\
\hline
CIFAR100 & ResNet 110 & $1.69371_7$ & $1.09169_4$ & $1.09607_5$ & ${\bf 1.07370}_1$ & $1.10091_6$ & $1.07966_2$ & $1.08375_3$ \\
CIFAR100 & Wide ResNet 32 & $1.80215_7$ & $0.94453_3$ & $0.95288_6$ & ${\bf 0.93273}_1$ & $0.94928_4$ & $0.94312_2$ & $0.95001_5$ \\
CIFAR100 & DenseNet 40 & $2.01740_7$ & $1.05713_2$ & $1.05909_3$ & ${\bf 1.05084}_1$ & $1.05972_4$ & $1.06127_5$ & $1.07626_6$ \\
\hline
CARS & ResNet~50 (pretrained) & $0.32993_7$ & $0.31813_4$ & $0.32381_6$ & $0.31904_5$ & ${\bf 0.31234}_1$ & $0.31593_2$ & $0.31793_3$ \\
CARS & ResNet~101 (pretrained) & $0.30536_7$ & $0.29329_3$ & $0.29714_4$ & $0.29788_5$ & ${\bf 0.28573}_1$ & $0.28897_2$ & $0.30444_6$ \\
CARS & ResNet~101 & $0.61185_7$ & $0.58619_4$ & $0.59504_6$ & $0.58683_5$ & ${\bf 0.57385}_1$ & $0.57774_2$ & $0.58319_3$ \\
\hline
BIRDS & ResNet~50 (NTSNet) & $0.74676_7$ & $0.56569_4$ & $0.61239_5$ & $0.63055_6$ & $0.54915_2$ & ${\bf 0.54508}_1$ & $0.56288_3$ \\
\hline 
ImageNet & ResNet~152 & $0.98848_7$ & $0.94208_4$ & $0.95081_5$ & $0.95786_6$ & ${\bf 0.92553}_1$ & $0.92850_2$ & $0.93935_3$ \\
ImageNet & DenseNet~161 & $0.94395_7$ & $0.90928_5$ & $0.91214_6$ & $0.90578_3$ & ${\bf 0.88937}_1$ & $0.89552_2$ & $0.90632_4$ \\
ImageNet & PNASNet5~large & $0.80240_7$ & $0.75761_6$ & $0.73955_5$ & $0.71522_4$ & ${\bf 0.65550}_1$ & $0.65674_2$ & $0.69595_3$  \\
\midrule
\multicolumn{2}{c|}{Average Relative Error} & $1.000_7$ & $0.766_4$ & $0.772_6$ & $0.768_5$ & ${\bf 0.749}_1$ & $0.751_2$ & $0.765_3$ \\
\bottomrule
\end{tabular}
}
\end{center}
\end{table*}

Finally, \cref{tbl:cwece} compares different methods in Classwise-ECE. While there is no single winning method on Classwise-ECE when the number of classes is less than $200$, \dirodir is the best method on this metric in ImageNet and in overall. In the next section, we discuss a hidden bias in Classwise-ECE metric that might become problematic. It seems Classwise-ECE might promote uncertainty in the output regardless of the actual accuracy of the model. This suggests there might be more investigation required for this metric and a practitioner should be cautious about these numbers.

\begin{table*}[!t]
\begin{center}
\caption{\em \label{tbl:cwece}\small{Classwise ECE.}}
\resizebox{1.0\linewidth}{!}{%
\begin{tabular}{c c | c c c c | c c c}
\toprule
Dataset  & Model & Uncal. & TS & \dirodir & \msodir & \dgfn & \oinv & \opre \\
\hline
CIFAR10 & ResNet 110 & $0.09846_7$ & $0.04344_5$ & $0.03950_4$ & ${0.03615}_2$ & $0.03791_3$ & ${\bf 0.03454}_1$ & $0.04435_6$ \\
CIFAR10 & Wide ResNet 32 & $0.09530_7$ & $0.04775_4$ & $0.02947_2$ & ${\bf 0.02921}_1$ & $0.05462_6$ & $0.04747_3$ & $0.04918_5$ \\
CIFAR10 & DenseNet 40 & $0.11430_7$ & $0.03977_4$ & $0.03687_2$ & ${\bf 0.03678}_1$ & $0.03877_3$ & $0.04575_5$ & $0.05182_6$ \\
\hline
SVHN & ResNet 152 (SD) & $0.01940_4$ & $0.01849_2$ & $0.01988_5$ & $0.02088_6$ & ${\bf 0.01478}_1$ & ${0.01858}_3$ & $0.02128_7$ \\
\hline
CIFAR100 & ResNet 110 & $0.41644_7$ & $0.20095_3$ & ${\bf 0.18639}_1$ & $0.20270_5$ & $0.21966_6$ & $0.19977_2$ & $0.20237_4$ \\ 
CIFAR100 & Wide ResNet 32 & $0.42027_7$ & $0.18573_4$ & ${\bf 0.17951}_1$ & $0.17966_2$ & $0.18636_5$ & $0.19397_6$ & $0.18484_3$ \\
CIFAR100 & DenseNet 40 & $0.47026_7$ & $0.18664_3$ & ${0.18630}_2$ & $0.19112_5$ & ${\bf 0.18614}_1$ & $0.19866_6$ & $0.18752_4$\\
\hline
CARS & ResNet~50 (pretrained) & $0.17353_3$ &  $0.18513_7$ & $0.17094_2$  & $0.18312_6$ & ${\bf 0.16891}_1$ & $0.18217_5$ & $0.17567_4$ \\
CARS & ResNet~101 (pretrained) & $0.16503_4$ & $0.17186_6$ & $0.15914_2$ & $0.17405_7$ & $0.16692_5$ & $0.16434_3$ & ${\bf 0.15509}_1$ \\
CARS & ResNet~101 & $0.26300_2$ & $0.27234_6$ & $0.26333_3$ & $0.27447_7$ & $0.26594_5$ & $0.26488_4$ & ${\bf 0.25097}_1$ \\
\hline
BIRDS & ResNet~50 (NTSNet) & $0.24901_3$ &  $0.26369_7$ & ${\bf 0.22920}_1$ & $0.25639_6$ & $0.25073_5$ & $0.25069_4$ & $0.24031_2$ \\
\hline 
ImageNet & ResNet~152 & $0.31846_7$ & $0.30886_4$ & ${\bf 0.30061}_1$ & $0.30895_5$ & $0.31372_6$ & $0.30642_3$ & $0.30081_2$ \\
ImageNet & DenseNet~161 & $0.30992_7$ & $0.30309_5$ & ${\bf 0.29403}_1$ & $0.29807_2$ & $0.30659_6$ & $0.30248_4$ & $0.29959_3$ \\
ImageNet & PNASNet5~large & $0.31356_7$ & $0.25587_6$ & ${\bf 0.23797}_1$ & $0.24283_2$ & $0.25004_5$ & $0.24634_4$ & $0.24493_3$  \\
\midrule
\multicolumn{2}{c|}{Average Relative Error} & $1.000_7$ & $0.752_6$ & ${\bf 0.704}_1$ & $0.734_3$ & $0.729_2$ & $0.740_4$ & $0.743_5$ \\
\bottomrule
\end{tabular}
}
\end{center}
\end{table*}

\subsection{Is Classwise-ECE a Proper Scoring Rule Calibration Metric?}
\label{sec:badcwece}

It is known that ECE is not a proper scoring rule and thus there exist trivial solutions which yield optimal scores~\citeRef{ovadia2019can}. In this section, we show the same holds for Classwise-ECE metric. Classwise-ECE is \say{defined as the average gap across all classwise-reliability diagrams, weighted by the number of instances in each bin:
\begin{equation}
    \text{Classwise-ECE} = \frac{1}{k}\sum_{j=1}^k\sum_{i=1}^m \frac{B_{i,j}}{n} |y_j(B_{i,j})-\hat{p}_j(B_{i,j})|
\end{equation}
where $k$, $m$, $n$ are the numbers of classes, bins and instances, respectively, $|B_{i,j}|$ denotes the size of the bin, and $\hat{p}_j(B_{i,j})$ and $y_j(B_{i,j})$ denote the average prediction of class $j$ probability and the actual proportion of class $j$ in the bin $B_{i,j}$.}~\cite{kull2019beyond2}.

While the above definition of Classwise-ECE intuitively makes sense, we show that this metric fails to represent the quality of a predictor in a common degenerate case e.g. in a balanced dataset with $k$ classes one could achieve a perfect Classwise-ECE by scaling down the logits with a large enough positive scalar. A large enough temperature value increases the uncertainty of the model and brings all the class probabilities close to $1/k$ while maintaining the accuracy of the model. As the result, in all the classwise-reliability diagrams every data point falls into the bin that contains confidence values around $1/k$. Since the dataset is balanced, the actual proportion of class $j$ in that bin will also be $1/k$ so the model exhibits a perfect Classwise-ECE.

We remark that this problem does not happen with ECE, because ECE is computed with regard to the \emph{accuracy} of the bins. While all the data points still fall inside the bin that contains the confidence value $1/k$, the accuracy of this bin would be equal to the accuracy of the model. Thus, there would be mismatch between the confidence and the accuracy of the bin, which results to a high ECE.

To validate this insight, we scale down the uncalibrated logit values by a large scalar number and see how it affects Classwise-ECE in \cref{tbl:cwece_failure}.  It shows this simple hack drastically improves the Classwise-ECE value of the uncalibrated models and outperforms the methods in \cref{tbl:cwece} by large margin in most of the cases. Note that we can not achieve perfect Classwise-ECE because the datasets are not perfectly balanced. 

We are concerned that this issue with \emph{Classwise-ECE might bias future work to lean towards merely increasing the uncertainty of predictions without actually calibrating the model in a meaningful way}. To avoid this, Classwise-ECE metric should be always used with other proper scoring rule metrics (e.g., NLL or Brier) in evaluation. As we discuss in the next section, this issue would not happen when bins are dynamically chosen to ensure the number of data points in each bin remains equal.

\begin{table*}[!t]
\begin{center}
\caption{\em\label{tbl:cwece_failure}\small{Temperature scaling effect on Classwise-ECE. A large temperature value improves the Classwise-ECE in most of the cases. The subscript numbers represent the rank compared to the values in \cref{tbl:cwece}. We remark that the purpose of this experiment is not to improve the performance but rather highlight the need for studying Classwise-ECE metric in the future works.}}
\resizebox{0.6\linewidth}{!}{%
\begin{tabular}{c c | c c c c}
\toprule
Dataset  & Model & Uncal. & Uncal./1000 \\
\hline
CIFAR10 & ResNet 110 & $0.09846$ & $0.00021_1$ \\
CIFAR10 & Wide ResNet 32 & $0.09530$ & $0.00126_1$ \\
CIFAR10 & DenseNet 40 & $0.11430$ & $0.00143_1$ \\
\hline
SVHN & ResNet 152 (SD) & $0.01940$ & $0.33123_8$\\
\hline
CIFAR100 & ResNet 110 & $0.41644$ & $0.00080_1$ \\
CIFAR100 & Wide ResNet 32 & $0.42027$ & $0.00199_1$ \\
CIFAR100 & DenseNet 40 & $0.47026$ & $0.00282_1$ \\
\hline
CARS & ResNet~50 (pretrained) & $0.17353$ & $0.16048_1$ \\
CARS & ResNet~101 (pretrained) & $0.16503$ & $0.16108_3$ \\
CARS & ResNet~101 & $0.26300$ & $0.15067_1$ \\
\hline
BIRDS & ResNet~50 (NTSNet) & $0.24901$ & $0.05831_1$\\
\hline 
ImageNet & ResNet~152 & $0.31846$ & $0.11074_1$ \\
ImageNet & DenseNet~161 & $0.30992$ & $0.11074_1$ \\
ImageNet & PNASNet5~large & $0.31356$ & $0.10960_1$ \\
\bottomrule
\end{tabular}
}
\end{center}
\end{table*}

\subsection{Debiased ECE and a Fix to Classwise-ECE}
We believe that the issue mentioned above is due to the binning scheme used in estimating Classwise-ECE which allows all the data points fall into a single bin. Nixon~\etl~\citeRef{nixon2019measuring_ref} propose an adaptive binning scheme that guarantees the number of data points in each bin remains balanced; therefore, it does not exhibit the same issue as Classwise-ECE. 
%
%
In addition to the binning scheme, Kumar et al.~\citeRef{kumar2019calibration_ref} introduce {\em debiased ECE} and multiclass {\em marginal calibration error} metrics that are debiased versions similar to the ECE and Classwise-ECE metrics, respectively. The idea is to subtract an approximate correction term to reduce the biased estimate of the metrics. For the completeness, we present {\em debiased ECE} and multiclass {\em marginal calibration error} for all the methods in~\cref{tbl:debiasedECE} and ~\cref{tbl:marginalECE}, respectively. While the results in debaised ECE are similar to ECE, comparing the results in \cref{tbl:cwece} and \cref{tbl:marginalECE} shows \dgfn is performing better in terms of multiclass marginal calibration error and outperforms \dirodir in average relative error.

Overall, although the intra order-preserving models are the winning methods among most of the ever-increasing calibration metrics, one should carefully pick the calibration method and the metric depending on their application.

\begin{table*}[!t]
\begin{center}
\caption{
\em \label{tbl:debiasedECE}\small{Debiased ECE~\protect\citeRef{kumar2019calibration_ref}.}
}
\resizebox{1.0\linewidth}{!}{%
\begin{tabular}{c c | c c c c | c c c}
\toprule
Dataset  & Model & Uncal. & TS & \dirodir & \msodir & \dgfn & \oinv & \opre  \\
\hline
CIFAR-10     & ResNet 110          & $0.09070_7$ & $0.01924_4$ & $0.01927_5$ & $0.01716_3$ & $0.01573_2$ & ${\bf 0.00000}_1$ & $0.02282_6$ \\
CIFAR-10     & Wide ResNet 32      & $0.08661_7$ & $0.00809_2$ & $0.00943_3$ & $0.00958_4$ & $0.01717_6$ & ${\bf 0.00194}_1$ & $0.01073_5$ \\
CIFAR-10     & DenseNet 40         & $0.10340_7$ & $0.01195_2$ & $0.01228_3$ & $0.01266_4$ & ${\bf 0.01130}_1$ & $0.02410_6$ & $0.02309_5$ \\
\hline 
SVHN         & ResNet 152 (SD)     & $0.01922_5$ & $0.00892_2$ & $0.00979_4$ & $0.00939_3$ & ${\bf 0.00617}_1$ & $0.02269_6$ & $0.03429_7$ \\
\hline
CIFAR-100    & ResNet 110          & $0.22699_7$ & $0.02004_2$ & $0.02842_3$ & $0.03054_5$ & $0.05596_6$ & ${\bf 0.00626}_1$ & $0.02903_4$ \\
CIFAR-100    & Wide ResNet 32      & $0.24827_7$ & $0.01031_2$ & $0.01909_5$ & $0.03018_6$ & $0.01498_4$ & ${\bf 0.00545}_1$ & $0.01408_3$ \\
CIFAR-100    & DenseNet 40         & $0.26523_7$ & ${\bf 0.00000}_1$ & ${\bf 0.00000}_1$ & $0.02809_6$ & ${\bf 0.00000}_1$ & $0.00432_4$ & $0.01265_5$ \\
\hline
CARS         & ResNet 50 (pre)     & $0.02327_6$ & $0.00900_2$ & $0.02512_7$ & $0.01605_4$ & ${\bf 0.00000}_1$ & $0.01611_5$ & $0.01363_3$ \\
CARS         & ResNet 101 (pre)    & $0.02181_6$ & $0.01956_3$ & $0.02419_7$ & $0.01504_2$ & $0.01964_4$ & $0.02136_5$ & ${\bf 0.01271}_1$ \\
CARS         & ResNet 101          & $0.04280_5$ & $0.02654_3$ & ${\bf 0.01518}_1$ & $0.03728_4$ & $0.02542_2$ & $0.04766_6$ & $0.04821_7$ \\
\hline
BIRDS        & ResNet 50 (NTS)     & $0.47117_7$ & $0.04054_4$ & $0.05545_5$ & $0.07224_6$ & ${\bf 0.01518}_1$ & $0.01650_2$ & $0.03104_3$ \\
\hline
ImageNet     & ResNet 152          & $0.07745_7$ & $0.02157_4$ & $0.05247_5$ & $0.06099_6$ & ${\bf 0.00066}_1$ & $0.00941_2$ & $0.01804_3$ \\
ImageNet     & DenseNet 161        & $0.06598_7$ & $0.02008_4$ & $0.04542_5$ & $0.04888_6$ & ${\bf 0.00998}_1$ & $0.01158_2$ & $0.01924_3$ \\
ImageNet     & PNASNet5 large      & $0.06820_6$ & $0.09620_7$ & $0.05728_5$ & $0.03580_4$ & $0.01273_3$ & ${\bf 0.00713}_1$ & $0.01272_2$ \\
\midrule
\multicolumn{2}{c|}{Avgerage Relative Error} & $1.000_7$    & $0.357_3$    & $0.430_6$    & $0.409_5$    & ${\bf 0.213}_1$    & $0.337_2$    & $0.406_4$   \\
\bottomrule
\end{tabular}
}
\end{center}
\end{table*}
\begin{table*}[!t]
\begin{center}
\caption{\em \label{tbl:marginalECE}\small{Marginal Calibration Error~\protect\citeRef{kumar2019calibration_ref}.}}
\resizebox{1.0\linewidth}{!}{%
\begin{tabular}{c c | c c c c | c c c}
\toprule
Dataset  & Model & Uncal. & TS & \dirodir & \msodir & \dgfn & \oinv & \opre  \\
\hline
CIFAR-10      & ResNet 110          & $0.00859_7$                    & $0.00305_2$                & $0.00371_6$                 & $0.00363_5$                & $0.00346_4$                  & ${\bf 0.00218}_1$       & $0.00336_3$  \\
CIFAR-10      & Wide ResNet 32      & $0.01516_7$                    & $0.01408_3$                & ${\bf 0.00410}_1$        & $0.00432_2$                & $0.01442_6$                  & $0.01416_5$                & $0.01410_4$                \\
CIFAR-10      & DenseNet 40         & $0.01132_7$                    & $0.00602_4$                & ${\bf 0.00417}_1$        & $0.00583_2$                & $0.00601_3$                  & $0.00729_6$                & $0.00686_5$                \\
\hline
SVHN          & ResNet 152 (SD)     & $0.00227_2$                    & $0.00245_3$                & $0.00426_5$                 & $0.00541_6$                & ${\bf 0.00178}_1$         & $0.00387_4$                & $0.00691_7$                \\
\hline
CIFAR-100     & ResNet 110          & $0.00315_7$                    & ${\bf 0.00129}_1$       & $0.00185_5$                 & $0.00233_6$                & $0.00144_3$                  & $0.00141_2$                & $0.00151_4$                \\
CIFAR-100     & Wide ResNet 32      & $0.00356_7$                    & $0.00266_4$                & $0.00222_2$                 & ${\bf 0.00199}_1$       & $0.00270_6$                  & $0.00268_5$                & $0.00257_3$                \\
CIFAR-100     & DenseNet 40         & $0.00417_7$                    & $0.00266_6$                & ${\bf 0.00222}_1$        & $0.00263_5$                & $0.00261_4$                  & $0.00259_3$                & $0.00234_2$                \\
\hline
CARS          & ResNet 50 (pre)     & $0.00063_6$                    & $0.00058_2$                & ${\bf 0.00035}_1$        & $0.00090_7$                & $0.00060_5$                  & $0.00059_3$                & $0.00059_3$                \\
CARS          & ResNet 101 (pre)    & $0.00043_3$                    & $0.00044_4$                & $0.00044_4$                 & $0.00092_7$                & $0.00041_2$                  & ${\bf 0.00034}_1$       & $0.00046_6$                \\
CARS          & ResNet 101          & ${\bf 0.00114}_1$                    & ${\bf 0.00114}_1$       & $0.00173_6$                 & $0.00230_7$                & ${\bf 0.00114}_1$         & $0.00118_5$                & $0.00117_4$                \\
\hline
BIRDS         & ResNet 50 (NTS)     & $0.00934_7$                    & $0.00139_4$                & $0.00148_6$                 & $0.00141_5$                & $0.00138_3$                  & $0.00132_2$                & ${\bf 0.00130}_1$       \\
\hline
ImageNet      & ResNet 152          & $0.00040_6$                    & $0.00038_2$                & ${\bf 0.00034}_1$        & $0.00042_7$                & $0.00038_2$                  & $0.00038_2$                & $0.00038_2$                \\
ImageNet      & DenseNet 161        & $0.00041_7$                    & $0.00039_3$                & ${\bf 0.00035}_1$        & $0.00038_2$                & $0.00039_3$                  & $0.00039_3$                & $0.00039_3$                \\
ImageNet      & PNASNet5 large      & $0.00039_7$                    & $0.00032_6$                & ${\bf 0.00025}_1$        & $0.00028_2$                & $0.00028_2$                  & $0.00029_4$                & $0.00030_5$                \\
\midrule
\multicolumn{2}{c|}{Average Relative Error} & $1.000_7$                      & $0.750_3$                  & $0.735_2$                   & $0.996_6$                  & ${\bf 0.725}_1$           & $0.778_4$                  & $0.898_5$  \\               
\bottomrule
\end{tabular}
}
\end{center}
\end{table*}

\clearpage

\bibliographystyleRef{plain}
\bibliographyRef{appendix}

\end{document}